\documentclass{article}
\usepackage{iclr2027_conference,times}
\usepackage[utf8]{inputenc}
\usepackage{amsmath,amssymb,amsthm}
\usepackage{booktabs}
\usepackage{longtable}
\usepackage{array}
\usepackage{multirow}
\usepackage{graphicx}
\usepackage{tikz}
\usetikzlibrary{arrows.meta,positioning,calc,fit,backgrounds}
\usepackage{url}

\newtheorem{theorem}{Theorem}
\newtheorem{proposition}{Proposition}
\newtheorem{corollary}{Corollary}
\newtheorem{definition}{Definition}
\newtheorem{lemma}{Lemma}
\newtheorem{remark}{Remark}
\newcommand{\Fcal}{\mathcal{F}}
\newcommand{\ind}[1]{\mathbb{I}\{#1\}} 
\newcommand{\csub}{c_{\mathrm{sub}}}

\title{Nociception as a Control Primitive:\\
Afferent Channels and Nociceptive Memory for Agents Deployed in One Body}

\author{Wolfgang Maass\\
Saarland University and\\
German Research Center for Artificial Intelligence (DFKI)}

\newcommand{\fehlendeabbildung}[1]{%
  \fbox{\parbox{0.95\textwidth}{\centering\vspace{6pt}\textbf{FIGURE FILE MISSING: \texttt{#1}}\\
  \footnotesize this build does not contain the figure; the file is not in the
  repository checkout\vspace{6pt}}}}

\iclrfinalcopy
\begin{document}

\maketitle
\lhead{Preprint}

\begin{abstract}
An agent deployed in a single body cannot learn how fast that body wears, because every trial that would reveal its wear resistance wears the body it would protect. We study this \emph{epoch-one} setting, in which the parameters of a fixed-weight policy are set before the body is drawn and never updated in life. The agent carries a load-gated nociceptive channel and a memory that retains what was felt. We prove that felt cost moves the allocation to the best-\emph{paid} work not yet felt rather than the gentlest, that an agent without retention never sees the felt-cost constraint bind, and that the channel pays only where the threat is individually unpredictable, cheap to avoid and expensive to ignore. We measure per body, setting the agent with channel and memory against the same individual without them, where neither carries a schedule learned across lives. On $2{,}000$ simulated floor-layer knees, with wear anchored to published loss rates, feeling, retaining and substituting extends the working life from age $55.2$ to $59.6$ and raises career output from $33.7$ to $36.1$. $69.3\%$ of bodies gain and \textbf{none lose}. A body that feels but retains nothing past the day gains one of the $+4.4$ years, and retention carries the rest. A population-trained agent gains $+0.65$ years from the same channel at $-0.54$ output. The difference is what a species prior already supplies, and a single body has none. The two are related by an identity, the ablation mean reporting $(1-\chi)$ of the per-body value with $\chi$ the share a blind schedule already captures, so we report both. Where the regime map predicts value, a care robot sextuples its certified service life and a field-anchored fleet writes off $0.15$ of its machines instead of $0.55$. Where it predicts none, a rover gains little over blind caution, so the map holds in both directions.
\end{abstract}

\section{Introduction}
\label{sec:intro}
Gradient-based optimization needs its objective sampled many times \citep{schulman2017ppo}, whereas an individual has one life. The events that dominate a lifetime objective occur less than once within it \citep{lipton2016fear}, and the trials that would identify its own wear resistance are the ones it cannot afford, because each wears the body it would protect, an irreversibility that safe exploration assumes known or learnable \citep{garcia2015safe,turchetta2016safe,grinsztajn2021reversibility,turner2020conservative,khetarpal2022continual}. Individuals are nevertheless competent, and this paper asks what an agent must carry into one life for that to be possible and measures what the answer is worth to the body that carries it.

Two timescales have to be kept apart. Optimization belongs to the slow loop, phylogeny in biology and population training in our experiments, and adaptation to the fast loop, one deployment life. A policy trained over a distribution of bodies and environments \citep{tobin2017domain,duan2016rl2,rakelly2019pearl} installs a prior over bodies and never observes the draw that is realized. In the fast loop, adaptation cannot be optimization, because with fewer than one informative event per life the hidden wear resistance is not recoverable from safe experience (Prop.~\ref{prop:estimation}), so the premise of Bayes-adaptive control, a latent observed often enough to estimate \citep{duff2002,ghavamzadeh2015brl,kumar2021rma}, fails by construction. What remains is memorization, conditioning on what has been felt rather than on an estimate. We call this regime \textbf{epoch one}: one deployment life, parameters fixed before the embodiment is known, no in-life updates, the setting of single-life RL \citep{chen2022singlelife} and of in-context adaptation with fixed weights \citep{laskin2022ad,lee2023dpt,bauer2023ada}.

We treat nociception as a control primitive, as homeostatic RL and artificial pain do \citep{keramati2014homeostatic,man2019homeostasis,kuehn2017pain}, but read differently. It is a coarse readout of tissue state that is \emph{load-gated}. The signal rises with the activity that produces it, so it identifies its own cause and affords \emph{substitution} within the repertoire a trade or a mission permits, rather than reduced effort. Because what is felt can be retained, as a conditioned aversion is \citep{bouton2004extinction,hasenbring2010fear}, one event can redirect the allocation for the rest of the life. What is under test is therefore a loop closed through the body. Pain moves the allocation, the allocation sets the load, the load sets both how fast the tissue wears and how strong the next signal is, and a signal that fades with the work that caused it leads the body to resume that work. The four links, \textbf{pain, memory, substitution, and a longer working life}, are traversed many times in one life, and memory is what keeps the loop from cycling. Pain without retention moves a body to the best-\emph{paid} work it can still do and returns it there whenever the signal dips, the same optimism that drives exploration for rewards \citep{auer2002ucb,brafman2002rmax} acting on a cost it has not yet felt. What is retained moves it to the gentlest and holds it there (Lemma~\ref{lem:escape}, and \S\ref{sec:individual} counts the returns). The loop is conditional at every link and operates only where the repertoire admits a gentler allocation that still clears the output floor (Thm.~\ref{thm:conjunction}, Prop.~\ref{prop:regime}).

The clinical counterpart is Charcot arthropathy, in which an insensate patient reloads an injured joint where one with intact sensation protects it for decades, on the same species prior \citep{nagasako2003cip,rogers2011charcot,woolf2010pain}. We measure per agent, setting the apparatus against the same individual without it, where neither arm carries a schedule learned across lives. On a floor layer's knee, working at full effort reaches age $55.2$ and $33.7$ in career output. Sparing from the first day reaches age $61.5$ and output $34.4$. Feeling, retaining and substituting reaches age $59.6$ and output $36.1$ (Table~\ref{tab:leiter}). Over $2{,}000$ paired bodies $69.3\%$ gain years and \textbf{none lose any}. The body that feels but retains nothing past the day gains one year, and retention carries the rest. A care robot, retired when its certification lapses, serves $323$ rather than $51$ months when it reads the signal, and there retention subtracts $34$ of them (Prop.~\ref{app:prior}).

What must a fixed-weight agent carry through one life in a body that wears, and what is that worth to the body? Three results answer this. \emph{Content of the state.} Under an optimistic prior a felt-cost constraint sends the agent to the best-paid untried option, and without retention the constraint never binds, since the signal exists only while the harmful work is done (Lemma~\ref{lem:escape}, Cor.~\ref{cor:loadgate}). Retention is therefore the precondition of sensing, and the state stores a threat, kept after the activity is set down. \emph{Regime.} The channel pays only where the threat is individually unpredictable, cheap to avoid and expensive to ignore, and there the optimum is protective if and only if sensing, persistence and tail-sensitivity hold jointly (Prop.~\ref{prop:regime}, Thm.~\ref{thm:conjunction}). All three conditions can be checked before a world is run, and a world built after the prediction confirmed it. \emph{Measurement.} The value lies on the bodies at risk, so a population mean reports $q$ times what an at-risk body experiences, $q = 0.693$ here (Rem.~\ref{thm:insurance}), and an ablation of a population-trained agent reports $(1-\chi)$ of the per-body value, $\chi$ the share a blind schedule already captures (Prop.~\ref{prop:zerlegung}). $\chi$ is a property of the world, since in a minimal MDP with no body and no learner the channel is worth the same to every individual while the ablation contrast moves by a factor of twelve, so we measure per body, against the same body without in-life state. The paper thus names a setting, shows that in it retention is a precondition of sensing rather than a refinement, gives conditions for when a body model pays that are decidable before the world exists, and supplies a measurement identity that applies beyond its subject.

\section{Related work}
\label{sec:related}
Four bodies of work touch the setting, and each rests on an assumption that epoch one removes. \emph{Adaptation within a life.} Episodic control stores returns of visited states \citep{blundell2016mfec,pritzel2017nec}, meta-RL and rapid motor adaptation infer a latent task or embodiment from repeated episodes or recent history \citep{duan2016rl2,rakelly2019pearl,kumar2021rma}, single-life RL adapts by transfer from the same task distribution \citep{chen2022singlelife}, and in-context adaptation with fixed weights adapts through state that grows within the episode \citep{laskin2022ad,lee2023dpt,bauer2023ada}. All of them assume that the latent which matters can be estimated from the experience the life affords, or that similar lives exist to transfer from. Here the latent is a wear resistance that safe experience cannot identify (Prop.~\ref{prop:estimation}), so the within-life state cannot be an estimate. We keep the policy class of the last line, fixed weights and a growing state, and change what the state stores to a felt cost per activity, retained after the activity is set down.

\emph{Safety and constraints.} Risk-sensitive and constrained RL change the objective \citep{tamar2015cvar,chow2015cvar,altman1999constrained,achiam2017cpo,garcia2015safe}, safe exploration learns an unknown constraint by approaching it, pairing optimism on reward with pessimism on untried cost \citep{turchetta2016safe,wachi2020snomdp,efroni2020cmdp,liu2021zero,bura2022dope}, conservative bandits hold a baseline while they explore \citep{wu2016conservative,kazerouni2017conservative}, and shields or barrier functions enforce a constraint given in advance \citep{alshiekh2018shielding,ames2019cbf}. Our constraint is neither given nor modelled but felt, and felt only while the harmful work is done. That is why the optimism these methods rely on is the hazard here, since to a body that retains nothing an untried option looks free (Lemma~\ref{lem:escape}). Pessimism on untried cost is the fixed-endowment row of Table~\ref{tab:leiter}, and the risk-sensitive objective runs as a CVaR arm in the fleet (App.~\ref{app:fleet}).

\emph{Engineering under degradation.} Condition-based maintenance estimates remaining useful life from monitored condition \citep{jardine2006cbm,si2011rul,saxena2008cmapss}, and that estimate drives inspection and repair as a POMDP \citep{andriotis2019drl} or service as a restless bandit \citep{whittle1988restless}. The degradation model behind such an estimate is fitted on a population of units, which is our slow loop, and the population-trained blind arm of Table~\ref{tab:leiter} is a schedule of this kind. Our question is what one unit gains beyond it when its wear rate is drawn per unit and never observed (Prop.~\ref{prop:zerlegung}).

\emph{Signals from the body.} Homeostatic and interoceptive RL regulate toward a physiological setpoint \citep{keramati2014homeostatic,man2019homeostasis,craig2002interoception}, artificial pain supplies a nociceptive channel to a robot \citep{kuehn2017pain}, intrinsic fear penalizes states that precede catastrophe \citep{lipton2016fear}, and damage-recovery repertoires search for a gait after an injury \citep{cully2015robots}. These are the closest in spirit, but they use the signal as a reward term or as a trigger for search, whereas we use it as a load-gated constraint with retention and ask what it is worth to one body and in which worlds. The formal tools are approximate information states, which say what a compressed state must retain \citep{astrom1965,subramanian2022ais}, partial monitoring \citep{bartok2014partial}, high-confidence off-policy evaluation \citep{thomas2015hcope}, and the stochastic-programming separation of the value of the stochastic solution from perfect information, the identity behind the schedule confound \citep{madansky1960,birge2011stochastic,duff2002,ghavamzadeh2015brl}.

\section{The single-life model}
\label{sec:framework}
\begin{figure}[t]\centering
\resizebox{\textwidth}{!}{%
\begin{tikzpicture}[font=\scriptsize,
  box/.style={draw, rounded corners=2pt, align=left, inner sep=3pt, text width=#1, minimum height=2.3cm},
  low/.style={draw, rounded corners=2pt, align=left, inner sep=3pt, text width=#1},
  arr/.style={-{Latex[length=1.6mm]}, thick},
  feed/.style={-{Latex[length=1.6mm]}, thick, blue!60!black},
  read/.style={-{Latex[length=1.6mm]}, thick, densely dotted, green!40!black}]
\node[box=4.4cm, fill=gray!8] (world) {\textbf{World} $\mathcal{M}(\theta,e)$, one hidden draw\\
  \textbullet\ body $\theta$: wear resistance $\kappa$, gain $g_I$, talent\\
  \textbullet\ substrate $x_t$: tissue or component state\\
  \textbullet\ wear $M_{t+1}=M_t-d/\kappa$ \ (Eq.~\ref{eq:hazard})\\
  \textbullet\ floor $\underline{y}$: work lost when no mix clears it\\
  \textbullet\ \emph{instances}: trade, care robot, rover, fleet};
\node[box=3.5cm, fill=teal!12, right=4mm of world] (aff) {\textbf{Afferent layer} $\alpha$\\
  \textbullet\ $s_t = \alpha(x_t,a_t)\,e^{\sigma_s\epsilon_t}/g_I$\\
  \textbullet\ load-gated, salience-gated,\\ \phantom{\textbullet\ }calibrated\\
  \textbullet\ gain $g_I$, sensitization, ceiling $\tau$:\\ \phantom{\textbullet\ }anchored, set before the draw\\
  \textbullet\ the only path across the\\ \phantom{\textbullet\ }observation boundary};
\node[box=4.5cm, fill=orange!12, right=4mm of aff] (dec) {\textbf{Decision level}, channel set $C$\\
  \textbullet\ in-life state: $\hat c_j$, $\hat y_j$, forgetting $v$ (probe)\\ \phantom{\textbullet\ }or nociceptive memory $\rho$ (trained arm)\\
  \textbullet\ probe: $\max_a a^{\!\top}\hat y$ s.t.\ $a^{\!\top}\hat c\le\tau$ \ (Eq.~\ref{eq:programm})\\
  \textbullet\ trained arm: PPO $\pi_0(a\mid s_t,\rho,y,t)$};
\draw[arr] (world) -- node[above]{$x_t$} (aff);
\draw[arr] (aff) -- node[above]{$s_t$} (dec);
\draw[arr] (dec.north) -- ++(0,0.3) -| (world.north);
\node[above] (alloc) at ($(world.north)!0.5!(dec.north)+(0,0.3)$) {allocation $a_t$: shares over the repertoire, hence the load};
\draw[dashed, black!60, line width=0.6pt] ($(aff.north east)+(0.2,0.26)$) -- ($(aff.south east)+(0.2,-0.12)$);
\node[font=\scriptsize\itshape, text=black!60, fill=yellow!6, inner sep=1pt] at ($(aff.north east)+(0.2,0.13)$) {observation boundary};
\node[above=1pt of alloc.north west, anchor=south west, font=\scriptsize, text=black, inner sep=1pt, xshift=-4.4cm] (lifelab) {{\small\bfseries Fast loop}\ \ \itshape one life, epoch one: the individual agent acting on afferent signals in its body};
\node[low=2.0cm, fill=blue!7, anchor=north west] (pop) at ($(lifelab.west |- world.south)+(1.5pt,-12mm)$) {\textbf{population} $p(\theta,e)$\\ bodies and environments};
\node[low=2.0cm, fill=blue!7, right=3mm of pop] (lives) {\textbf{many lives}\\ $J$ per life\\ (trained arm only)};
\node[low=2.6cm, fill=blue!7, right=3mm of lives] (upd) {\textbf{parameters}\\ $\pi_0$ (trained arm) or\\ fixed apparatus (probe)};
\draw[arr] (pop) -- (lives);
\draw[arr] (lives) -- (upd);
\draw[arr] (upd.south) -- ++(0,-0.38) -| (pop.south);
\node[font=\scriptsize\itshape, inner sep=1pt, fill=blue!5] at ($(lives.south)+(0,-0.38)$) {next generation};
\node[below=17pt of pop.south west, anchor=north west, font=\scriptsize, text=black, inner sep=1pt] (lrnlab) {{\small\bfseries Slow loop}\ \ \itshape across lives: the learner. Nothing returns from a life (Def.~\ref{def:epoch})};
\draw[feed] (pop.north) -- node[right, align=left, xshift=1pt]{one draw $(\theta,e)$,\\ held for the whole life} (pop.north |- world.south);
\coordinate (sl1) at ($(upd.north)+(0,0.45)$);
\coordinate (dsx) at ($(dec.south)+(-0.4,0)$);
\coordinate (sl2) at (sl1 -| dsx);
\draw[feed] (upd.north) -- (sl1) -- (sl2) -- (dsx);
\coordinate (asx) at ($(aff.south)+(0.3,0)$);
\draw[feed] (sl1 -| asx) -- (asx);
\node[above, align=center, inner sep=1.5pt] at ($(sl1 -| asx)!0.5!(sl2)$) {parameters of $\alpha$ and $C$, set before\\ the draw and never updated in life};
\node[low=5.3cm, draw=green!40!black, densely dotted, line width=0.6pt, fill=white, below=12mm of dec.south east, anchor=north east, xshift=1.5pt] (meas)
  {\textbf{Measurement} of the whole life\\ \textbullet\ $J$ (Eq.~\ref{eq:channel}), career end, $D_{\rm phys}$ and $M_{\rm phys}$,\\ \phantom{\textbullet\ }read across the observation boundary\\ \textbullet\ $V_{\rm ind}(\theta)=J(\rho;\theta)-J(\rho_\emptyset;\theta)$,\\ \phantom{\textbullet\ }the same $\theta$ twice, with and without $C$\\ \textbullet\ ablations remove $s$, $h$ or $w$ from $C$\\ \textbullet\ trained arms read $(1-\chi)\,\mathbb{E}[V_{\rm ind}]$ (Prop.~\ref{prop:zerlegung})};
\coordinate (jstart) at ($(dec.south east)+(-0.6,0)$);
\draw[read] (jstart) -- node[left, xshift=-1pt, align=right]{reads the life,\\ sets nothing in it} (jstart |- meas.north);
\node[below=1pt of meas.south west, anchor=north west, font=\scriptsize, text=black, inner sep=1pt] (explab) {{\small\bfseries Experimenter}\ \ \itshape outside both loops};
\begin{scope}[on background layer]
  \node[draw=gray!60, line width=0.5pt, fill=yellow!7, rounded corners=3pt, inner xsep=1.5pt, inner ysep=3pt, fit=(world)(aff)(dec)(alloc)(lifelab)] (life) {};
  \node[draw=gray!60, line width=0.5pt, fill=blue!5, rounded corners=3pt, inner xsep=1.5pt, inner ysep=3pt, fit=(pop)(lives)(upd)(lrnlab)] (lrn) {};
  \node[draw=green!40!black, densely dotted, line width=0.6pt, fill=green!4, rounded corners=3pt, inner xsep=1.5pt, inner ysep=3pt, fit=(meas)(explab)] (exp) {};
\end{scope}
\end{tikzpicture}}
\caption{\textbf{Architecture.} The fast loop is one life inside a single hidden draw of the world. The dashed line is the observation boundary, crossed only by the afferent signal. The decision level allocates the day, which sets the load, the wear and the strength of the next signal, and its in-life state is the recorded felt cost per activity (probe) or a nociceptive memory (trained arm). The slow loop feeds the fast loop only before the life begins, with the draw $(\theta,e)$ and with the parameters of both the afferent layer and the decision level, and receives nothing back (Def.~\ref{def:epoch}). The experimenter reads the whole life across the observation boundary and sets nothing inside it.}
\label{fig:arch}
\end{figure}
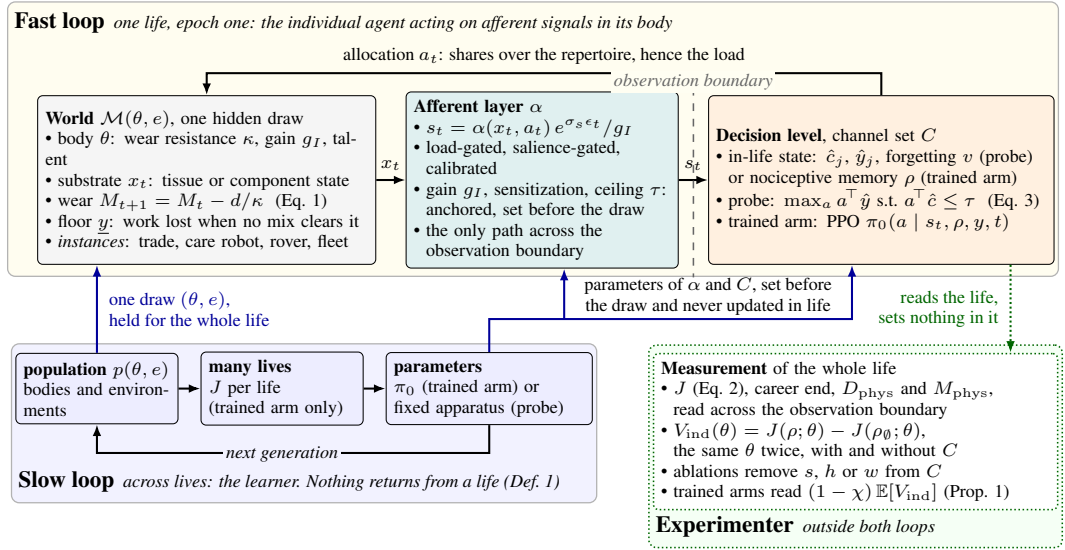
The model has three parts (Fig.~\ref{fig:arch}). The fast loop is one life, a world with a hidden body, an afferent layer that reads it and a decision level that allocates the work. The slow loop sets the parameters of the last two before the life begins and receives nothing back. The experimenter stands outside both loops and compares the same body with and without the channel. We define the world of the fast loop first, then the policy class the slow loop sets, then the two values the experimenter reads, then the agent.

\emph{The fast loop, its world.} Let $\{\mathcal{M}(\theta,e)\}$ be a family of POMDPs indexed by embodiment $\theta$ and environment $e$, with training distribution $p(\theta,e)$. The embodiment $\theta$ is the body the agent is dealt and cannot change. The environment $e$ is what the world demands of that body: the task mix it must keep, the floors and caps, and the event statistics. The load the body carries is neither, it is chosen inside the life. The setting rests on the hiddenness of the realized $(\theta,e)$, which is unidentifiable from safe experience (Prop.~\ref{prop:estimation}), and requires no genomic assumption. Each life has horizon $H$, an irreversible integrity state $M_t$ and a repertoire $A$ of activities from which each period a portfolio is composed. Activity $j$ has load $\ell_j$ on the site and yield $y_j$ to the work, and the two do not rank alike, since the best-paid day type wears the joint. Each period a life selects shares $a_t \in \Delta(A)$ under two constraints, so substitution is bounded from both sides. The mix must deliver at least $\underline{y}$ or the work is lost, and wear also removes options, since activities whose demands exceed current capability leave $A$. Integrity falls with the work done:
\begin{equation}
M_{t+1} \;=\; M_t \;-\; \frac{d(a_t, M_t)}{\kappa(\theta)} .
\label{eq:hazard}
\end{equation}
$d(a_t,M_t)$ is the wear of ordinary use, irreversible and deterministic in the activity and the state already reached. In the wear worlds it is the portfolio's hazard-weighted load times an amplification that rises as the substrate degrades, net of a slow repair term and with no acceleration in the damage state itself, and its rate is anchored to published loss rates of the substrate (App.~\ref{app:wear}, \citealp{maschek2014cartilage,roth2017cartilage}). The drawn wear resistance $\kappa(\theta)$ divides it, so one body wears faster than another under identical work, and $\kappa$ is what a life would need to know about itself and cannot (Prop.~\ref{prop:estimation}). Some worlds additionally carry acute events, an extension that is stated in App.~\ref{app:proofs} and run by the fleet of \S\ref{sec:fourthworld}.

A career ends in the first period in which no admissible mix clears $\underline{y}$, which under Eq.~\ref{eq:hazard} is the analogue of an event with probability near one rather than near zero. The chain has four links. The signal identifies the activity that caused it. A trace retains that identity once the activity is set down and the signal with it. The repertoire permits the activity to be traded, and the floor decides for how long. Three are properties of the world, $\ell_j$, $y_j$ and $\underline{y}$, and the trace is the one property of the agent, entering the decision as the forgetting rate $v$ of Eq.~\ref{eq:programm} (Fig.~\ref{fig:kette} in App.~\ref{app:wear}).

\emph{The slow loop.} Def.~\ref{def:epoch} fixes what the slow loop may set, and with the reference arm also what the experimenter may compare.
\begin{definition}[Epoch-one policy class]
\label{def:epoch}
At deployment one hidden pair $(\theta,e)$ is drawn and held for the whole life. $\Pi(C)$ contains policies measurable with respect to the in-life filtration $\Fcal^C_t$ of channel set $C$, with fixed parameters, hence no dependence on the realized $\theta$ and no updates. The reference arm $\rho_\emptyset$ is the member with $C = \emptyset$ and no schedule chosen with knowledge of which work is gentlest.
\end{definition}
\noindent The slow loop supplies the prior $\pi_0$, and deployment competence is what $C$ adds. An individual has no knowledge of which work is gentlest, so a constant sparing policy is inadmissible as $\rho_\emptyset$, because comparing against one measures that knowledge instead.

\emph{The experimenter.} Two values follow from the definition, and they must not be confused. $V(C \mid C')$, the gain of the best policy reading $C$ over the best reading $C'$, measures information over the species prior and is nil wherever $\pi_0$ can carry the schedule. $V_{\rm ind}(\rho;\theta) := J(\rho;\theta) - J(\rho_\emptyset;\theta)$ is the paired per-body gain on identical draws, reported as a distribution. What an ablation ablates is not the naive $\rho_\emptyset$ but, trained over the population, the best \emph{schedule}. Writing $\chi$ for the ratio of what that schedule captures to the individual gain, the two designs differ by exactly that term (Prop.~\ref{prop:zerlegung}). We call the gap the \textbf{schedule confound}.

\label{sec:afferent}
\emph{The fast loop, its agent.} One layer, the \emph{afferent layer}, separates the biology from everything above it. The substrate state $x_t$ is domain-specific and never seen by the decision level. The afferent map $\alpha : (x_t, a_t) \mapsto s_t \in [0,1]^k$, one entry per activity, transduces it onto one common scale, $\alpha(x,a) = \ell(a)\,S(x)/\bar\alpha$ clipped per activity, with $\bar\alpha$ a population constant. Three properties follow. \emph{Load-gating}: damage is felt when the activity is worked and falls when work substitutes away from it, which makes the signal actionable. \emph{Calibration}: one scale serves policy input, memory write and cost. \emph{Salience-gating}: events saturate the scale, so the spike alone triggers a write. The map is a product, so the same action is felt more as the body wears and the constraint tightens on its own as the reserve falls, without the life computing how much of itself remains. \emph{In this sense nociception is a control primitive}, since the body registers the depletion of a resource on the scale of the decision that depletes it.

\label{sec:channels}
During a deployment the decision level observes its own body and history only through a channel set $C$, holding the felt signal $s_t$ and a nociceptive memory over activities built from it. The dashed line of Fig.~\ref{fig:arch} between the two is the \emph{observation boundary}, and what lies left of it is never observed.
\begin{align}
s_t &= \frac{\alpha(x_t,a_t)}{g_I}\,e^{\sigma_s \epsilon_t}, \qquad
J = \textstyle\sum_t y(a_t) - w \sum_t s_t.
\label{eq:channel}
\end{align}
Here $y(a) = \sum_j a_j\,y_j$ is the yield of the mix as the body can still deliver it, each $y_j$ reduced by what wear has taken from the capability the activity needs, and $\alpha(x_t,a_t)$ is the afferent map of Fig.~\ref{fig:arch}. The individual gain $g_I$ and the wear resistance $\kappa(\theta)$ of Eq.~\ref{eq:hazard} are drawn separately, both log-normal. The spread of the gain and the reading noise $\sigma_s$ are anchored per world on the between-person and the within-person variability of the felt signal, while the spread of $\kappa$ is set rather than anchored and is therefore driven as an axis (App.~\ref{app:wear}). No substrate variable enters the observation, so the afferent map is the only path across the observation boundary of Fig.~\ref{fig:arch} and the observation does not recover $\kappa$ (Prop.~\ref{prop:estimation}). The objective trades realized earnings against the accumulated felt signal, weighted by the valuation trait $w$, so $w{=}0$ recovers pure output-maximization. Here $y$ is what the agent collects and not what it produces, because losing the work does not end the stream. An outside option follows, re-employment for a worker, a residual use for a machine, nothing for a written-off unit, and its level prices the loss of the work (App.~\ref{app:wear}). The memory is deliberately almost nothing, a conditioned aversive trace with extinction and recovery, recording where and how much but never what or when, and the only state that grows within a life (App.~\ref{app:memlaw}), and nothing affective is implemented.

\label{sec:programm}
Eq.~\ref{eq:channel} prices the signal in the objective, which is how every arm is \emph{scored}. It is not how the signal reaches the allocation. The decision level carries two numbers per activity, an estimate $\hat y_j$ of what it still delivers and an estimate $\hat c_j$ of what working it costs to feel, and each period it solves
\begin{equation}
\max_{a \in \Delta(A)} \; \sum_j a_j\,\hat y_j
\qquad \text{subject to} \qquad \sum_j a_j\,\hat c_j \;\le\; \tau ,
\label{eq:programm}
\end{equation}
under the availability limits the trade puts on each group of activities and the floor $\underline{y}$ that decides whether the work is kept at all. Both estimates are written by working and by nothing else. At the end of a period they take the values just delivered and just felt on every $j$ the period worked, while every other $j$ relaxes toward the endowment $(y_j, \hat c^{0}_j)$ at a forgetting rate $v \in [0,1]$. That rate is the second link of the chain: $v = 1$ forgets everything but what was worked last, $v = 0$ forgets nothing, and these are the last two arms of Table~\ref{tab:leiter}. The first feels the work it is doing and retains nothing past the day, where the day is the allocation of one period, so that the next allocation is chosen on what the last day felt and on nothing older, whereas the second retains what it felt. The ceiling $\tau$ is the felt load of a fresh body at its observed task mix, scaled by the ratio of the stress at which the substrate begins to fail to its intact peak stress, both taken from the literature of the world (App.~\ref{app:wear}), giving $\tau = 0.38$ in the trade.

\noindent \emph{Provenance of the ceiling.} A price with $w > 0$ and no constraint could carry the same information. We run the ceiling because $\tau$ is a threshold in contact mechanics whereas a price would have to be chosen (Prop.~\ref{prop:calibration} states what choosing it wrongly costs), and Lemma~\ref{lem:escape} and Cor.~\ref{cor:loadgate} are stated over Eq.~\ref{eq:programm}. The trade's agent carries $w = 0$, so its output is what it earned, while $J$ still scores every arm in every world.

\section{Theory}
\label{sec:theory}
Two results carry the chronic world and chain, Lemma~\ref{lem:escape}, which says where a cost-constrained agent escapes to and why zero retention leaves the constraint vacuous, and Prop.~\ref{prop:chronic}, which prices sensing net of the wear spent acquiring it.

Write $\Delta_b$ for the working life a protective response can add to one body and $q$ for the share of bodies whose working life is cut short without it. Rem.~\ref{thm:insurance} and Prop.~\ref{prop:estimation} follow from Eq.~\ref{eq:hazard} directly, and Prop.~\ref{prop:zerlegung} is an identity over any world. The remainder are proved in the stylized model $\mathcal{M}_0$ of App.~\ref{app:proofs}, which carries the acute extension. With $V_{\rm ind}(\rho;\theta)$ the paired gain of a protective response $\rho$ on body $\theta$ (\S\ref{sec:framework}), call $\theta$ \textbf{at risk} if $\rho_\emptyset$ loses the work before the horizon, and say $\rho$ satisfies \textbf{off-risk coupling within $\varepsilon$} if on bodies not at risk it differs from $\rho_\emptyset$ by at most $\varepsilon$.

\begin{remark}[Concentration over bodies]
\label{thm:insurance}
If a protective $\rho$ satisfies off-risk coupling within $\varepsilon$, then $\mathbb{E}_\theta[V_{\rm ind}] \le q\,\Delta_b + \varepsilon$ over the population, while the conditional value $\mathbb{E}_\theta[V_{\rm ind} \mid \text{at risk}]$ is the at-risk gain itself, bounded by $\Delta_b$ and not diluted by $q$.
\end{remark}
\begin{corollary}[What a mean over bodies answers]
\label{cor:power}
Under the same premise the population mean understates the value to a body that needs it by $1/q$.
\end{corollary}
\noindent We therefore report the distribution and the conditional value throughout.
\begin{proposition}[Schedule confound]
\label{prop:zerlegung}
Let $\sigma^\ast$ be the best schedule, the best policy reading nothing, and $\chi := \mathbb{E}_\theta[J(\sigma^\ast)-J(\rho_\emptyset)]\,/\,\mathbb{E}_\theta[V_{\rm ind}]$ the share of the per-body value it captures. Then $V(C\mid\emptyset) = (1-\chi)\,\mathbb{E}_\theta[V_{\rm ind}]$, and $V(C\mid\emptyset) \le \mathbb{E}_\theta[V_{\rm ind}]$ whenever the ablated arm is at least as good in the mean as $\rho_\emptyset$ (proof in App.~\ref{app:proofs}).
\end{proposition}
\noindent The identity holds for any component measured by ablation against a population-trained agent, so an ablation mean is silent wherever a prior can carry the schedule, and we report both quantities in every world.

\begin{proposition}[Identification requires wear]
\label{prop:estimation}
Under Eq.~\ref{eq:hazard} two wear resistances separate only in proportion to the wear already suffered, so every statistic informative about $\kappa$ is a function of that wear. A body learns its own wear resistance only by spending the integrity the knowledge would have protected (App.~\ref{app:proofs}).
\end{proposition}
\noindent On the channel we ship, no residual information about $\kappa$ survives in the observation (App.~\ref{app:wear}). Since the proposition is a property of the filtration, it binds every $\Fcal^C$-measurable policy, learned or static, so the trait arrives priced in damage and the value of the channel arises from changed behaviour.

\begin{lemma}[Where the constrained agent goes, and why retention is a precondition]
\label{lem:escape}
Consider the program of Eq.~\ref{eq:programm}, $\max_a \sum_j a_j \hat y_j$ subject to $\sum_j a_j \hat c_j \le \tau$, where $\hat c_j$ equals the felt cost on the activities the agent has worked and retained and equals the endowment $\hat c^{0}_j = 0$ on all others. \textbf{(a)}~\emph{Escape.} When the constraint binds, the optimum moves the day onto the best-\emph{paid} activity whose cost the agent does not know, whatever that activity truly costs. \textbf{(b)}~\emph{Vacuity without retention.} The cost of an activity is written only after a period that worked it, and each period's allocation is chosen before that period's signal arrives. With zero retention every decision is therefore taken with $\hat c \equiv 0$, so the constraint never binds and the agent keeps none of the protective value, not even the $1/R$ share a memoryless agent keeps in the acute case (Thm.~\ref{thm:conjunction}(b)).
\end{lemma}
\noindent The proof is given in App.~\ref{app:proofs}, where part (b) is Cor.~\ref{cor:loadgate}. Retention resolves (a) by converting untried into tried one option at a time, and Prop.~\ref{app:prior} bounds what a trace can add given what the slow loop installs, and its value vanishes when the endowment is exact.

\begin{proposition}[The paying regime, in $\mathcal{M}_0$]
\label{prop:regime}
In $\mathcal{M}_0$ under time-additive $J$, with $s$ the signal, $h$ the memory and $w$ the valuation, if any of the following fails the optimal policy in $\Pi(\{s,h,w\})$ is unprotective and the channels carry no decision-relevant information \citep{howard1966voi}: \textbf{(i)}~the threat is individually unpredictable \emph{ex ante}, \textbf{(ii)}~it is cheap to avoid \emph{ex post}, the substitution premium $c$ not being bounded away from zero, and \textbf{(iii)}~it is expensive to ignore, the value of survived time not being exhausted by a finite milestone schedule. Each is proved separately in Prop.~\ref{app:regime} in App.~\ref{app:proofs}.
\end{proposition}
\noindent All three conditions are properties of a world that can be checked before it is run, and we check them in every world we report.

\begin{theorem}[Conjunction, acute case]
\label{thm:conjunction}
In $\mathcal{M}_0$ under a law-invariant objective $\Phi$ of the per-life return, of which $J$ is one instance, assume (i) and (iii) of Prop.~\ref{prop:regime}, a calibrated gain (Prop.~\ref{prop:calibration}), an opaque event, in that no later consequence reveals it, and that mean return alone does not pay protection, $c \ge (p^{+}-p)\,L_e$, with $p$ and $p^{+}$ the per-step event probability before and after the first event, $L_e$ the event cost and $c$ the substitution premium. Then the $\Phi$-optimal policy in $\Pi(\Fcal^C)$ is protective if and only if (for (b), asymptotically in $R$) \textbf{(a)}~\emph{sensing}: the event enters the filtration at its time, \textbf{(b)}~\emph{persistence}: it stays there, without which at most a $1/R$ fraction of the gain is captured, $R$ being the steps remaining, and \textbf{(c)}~\emph{tail-sensitivity}: $\Phi$ charges an event more than mean return does, by enough that $c < c^\ast_\Phi := (p^{+}-p)\,L_e^\Phi$.
\end{theorem}
\noindent The theorem is stated for the acute extension of Eq.~\ref{eq:hazard} and proved as Thm.~\ref{app:conjunction} in App.~\ref{app:proofs}. Each condition names a role that several modules can fill, (b) a nociceptive memory, recurrence or the action history, and (c) $w\,s$, death or CVaR. Under Eq.~\ref{eq:hazard} the count drops to two, wear supplying (c) through the dynamics (Prop.~\ref{prop:chronic}). \noindent It is tested in the fleet of \S\ref{sec:fourthworld}, the one world whose premises we verified. We removed the realization instead, and the fleet then separates on $w$ alone (Thm.~\ref{app:conjunction} states which realizations $L_e^\Phi$ admits).

\begin{proposition}[Conjunction, chronic case]
\label{prop:chronic}
In the chronic model $\mathcal{M}_1$ of App.~\ref{app:proofs}, a wear budget spent across profiles whose rates are unknown and whose yield falls with the wear, \textbf{(a)}~the value of sensing is $\mathrm{EVPI} - \Pi_{\rm probe}$, the between-body information about which profile this body tolerates, less the output forgone in working each probed profile once. \textbf{(b)}~A policy that does not retain its readings pays $\Pi_{\rm probe}$ and cannot collect the $\mathrm{EVPI}$, since the state carries the wear already spent but no comparison \emph{across} profiles. \textbf{(c)}~is no longer a condition on $\Phi$.
\end{proposition}
\noindent There a profile's wear is felt only while it is worked and no event is observed without cost, which is why (a) is net, and the price of (c) lies in the dynamics, so removing it removes the body. The proposition was checked in a separate probe family (App.~\ref{app:wear}).

\section{Experiments}
\label{sec:experiments}
\paragraph{Setup.} Four worlds carry one stack (Table~\ref{tab:leiter}). The \emph{floor layer} is the chronic case, with cartilage wear takes away the deep flexion that kneeling needs, eight day types, a total-output floor at half the healthy output, ages $16$ to $65$, with $2{,}000$ paired bodies and individual probes re-solving Eq.~\ref{eq:programm} once a year (App.~\ref{app:wear}). The \emph{care robot} lifts patients under a certification that ends its service ($500$ paired lives, endpoints certified months and output). The \emph{rover} is the negative control, its wear deterministic in load ($500$ paired lives). The \emph{latent-defect fleet} is the acute case and the only world whose premises (i)--(iii) we verified (\S\ref{sec:fourthworld}), with endpoints death rate and $J$ for its trained arms. In every world but the fleet the same three responses run on the same bodies, with two caution arms as contrasts. Trained arms (PPO, App.~\ref{app:robots}) run in every world and carry the full stack $\{s,h,w\}$ and three ablations, clinical cases rather than knockouts: no signal (Charcot's deafferented joint, \citealp{nagasako2003cip}), no weight on pain (asymbolia, \citealp{berthier1988asymbolia}), no retention, with the same observation and the same $J$, so only the filtration and the forgetting rate differ across arms. Signs and gates were pre-registered ($72$ declarations).

The population-trained arms are the slow loop, and their contrast measures what a species prior gains (Def.~\ref{def:epoch}). At the anchored cell the channel is worth $+4.4$ years and $+2.4$ output to a body and $+0.65$ years and $-0.54$ output to the population ($24$ paired seeds, $20$ positive on years, Wilcoxon signed rank $p = 4.4\times10^{-5}$, Table~\ref{tab:leiter}).

\begin{table}[t]\centering\scriptsize
\setlength{\tabcolsep}{3pt}
\begin{tabular}{@{}l rrrr rr rr rr@{}}
\toprule
& \multicolumn{4}{c}{floor layer, absolute} & \multicolumn{2}{c}{care robot} & \multicolumn{2}{c}{rover} & \multicolumn{2}{c}{fleet} \\
\cmidrule(lr){2-5}\cmidrule(lr){6-7}\cmidrule(lr){8-9}\cmidrule(lr){10-11}
what the body does & works to & output & $D_{\rm phys}$ & $M_{\rm phys}$ & life & output & life & output & life & output \\
\midrule
works at full effort, reads nothing & $55.2$ & $33.7$ & $0.227$ & $0.739$ & $51.5$ & $36.9$ & $120.1$ & $96.4$ & $74.9$ & $67.6$ \\
spares from the first day$^{\dagger}$ & $61.5$ & $34.4$ & $\mathbf{0.188}$ & $\mathbf{0.780}$ & $+226\%$ & $+121\%$ & $+151\%$ & $+101\%$ & $+15\%$ & $-10\%$ \\
fixed endowment $\alpha{=}2$$^{\dagger\S}$ & $\mathbf{61.6}$ & $35.3$ & $0.191$ & $0.775$ & --- & --- & --- & --- & --- & --- \\
\midrule
trained, feels nothing$^{\ddagger}$ & $59.2$ & $33.4$ & $0.200$ & $0.768$ & $+67\%$ & $+30\%$ & $-30\%$ & $-46\%$ & $-6\%$ & $-26\%$ \\
trained, full stack$^{\ddagger}$ & $59.9$ & $32.9$ & $0.191$ & $0.778$ & $+679\%$ & $+432\%$ & $+26\%$ & $+2\%$ & $+26\%$ & $+20\%$ \\
\midrule
reads the present signal, no retention & $56.2$ & $34.3$ & $0.223$ & $0.743$ & $+528\%$ & $+309\%$ & $+163\%$ & $+108\%$ & $+1\%$ & $+1\%$ \\
reads the signal and remembers & $59.6$ & $\mathbf{36.1}$ & $0.205$ & $0.758$ & $+462\%$ & $+269\%$ & $+165\%$ & $+110\%$ & $\ge+34\%$ & $+34\%$ \\
\bottomrule
\end{tabular}
\caption{\textbf{Three ways to use one body, in four worlds, on the same individuals.} All arms are paired, and no
arm carries a schedule learned across lives (Def.~\ref{def:epoch}). Floor layer: \emph{works to} is the
age the trade ends, \emph{output} career output, $D_{\rm phys}$ cartilage lost, $M_{\rm phys}$ meniscal
competence retained ($2{,}000$ bodies). Machines: the first row gives the service life (certified months
for the care robot, months otherwise) and the output of the arm that reads nothing ($500$ paired lives per
robot, $5{,}000$ in the fleet), the other rows the percent change against it. The fleet runs to a horizon
of $100$ months, which the retaining arm reaches in every life, so its gain is a lower bound.
$^{\dagger}$Not admissible as a reference, since it already knows which work is gentlest. $^{\S}$Trade only, since a
machine reads every activity at once and an endowment is overwritten by the first reading. $^{\ddagger}$PPO. Trade: $10^6$ environment steps, $24$ paired seeds, $50{,}000$ paired lives per
seed, output as the trade paid it. Robots: $3\times10^6$ steps, $8$ seeds, the rover at residual stream
$0.10$. Fleet: $10^6$ steps, $18$ seeds, entropy $0$ (App.~\ref{app:robots}). The blind trained arm already
carries the population's schedule (it reaches $59.2$ without a channel), so these rows measure what a
species prior gains from the channel, while the value to one body is the last row against the first.}
\label{tab:leiter}
\end{table}

Against the body that works at full effort the apparatus yields $+4.4$ years, $95\%$ CI $[+4.23,+4.53]$, and $+2.4$ output, $[+2.33,+2.53]$, paired over $2{,}000$ bodies (Wilcoxon signed rank, $p < 10^{-15}$ on both), and no row of Table~\ref{tab:leiter} dominates the retaining body. It carries the highest output, the caution rows and the full-stack trained arm obtain their extra years at the cost of output, and the blind trained arm lies below it on both, so the headline is a Pareto statement. We report a corner and a band, where the corner is the grounded sensitization gain, the rate at which the felt signal rises with accumulated wear ($\alpha_{\rm sens}=8.51$, App.~\ref{app:wear}), at the measured kneeling share ($41.0\%$ of the day, SD $7.5$, \citealp{jensen2009exposure}), the least favourable corner on years we measured. Across the anchored span of that gain ($2.45$ to $8.51$) the gain runs $+4.4$ to $+5.8$ years and $+1.9$ to $+2.4$ output, with no body losing at either end, and over the set spread of wear resistance it runs $+3.1$ to $+5.6$ years (App.~\ref{app:wear}). The present-signal row and the retaining row differ only in what they retain, so their difference isolates condition (b) of Thm.~\ref{thm:conjunction}, persistence, worth three to four years in the trade and negative on the care robot (\S\ref{sec:boundary}).

\emph{The reference carries no schedule.} The claim is what one life adds to a body that knows nothing, and the reference is chosen for that question. Sparing from the first day works to $61.5$ years and the best schedule fitted across bodies to $61.6$, against the retaining body's $59.6$, and a reader may take either for the fairer reference. Both encode which work is gentlest, knowledge an individual entering the trade does not hold, as Charcot's patient does not, who carries the species prior and still reloads the joint. What a schedule captures is therefore reported beside the claim, $\chi = 1.40$ on years (Prop.~\ref{prop:zerlegung}), together with the channel's value beyond the best schedule, $-1.77$ predicted and $+0.65$ measured, which is the second question of \S\ref{sec:framework}.

\subsection{A learner acquires the conjunction where the regime holds}
\label{sec:fourthworld}
The regime map, the three conditions of Prop.~\ref{prop:regime}, made one out-of-sample prediction, filed before the world was built, that where (i)--(iii) hold the conjunction must return. A fleet of machines in the field chooses among eight near-equivalent classes of work (ii). Which class a given machine cannot tolerate is drawn per unit and is unidentifiable before the first incident (i). Incidents compound and the third writes the machine off (iii). The task economy is constructed, the event statistics come from field data on $10^5$-drive populations \citep{pinheiro2007failure,schroeder2007disk}. The prediction held, in that at the deterministic optimum (entropy $0$, $18$ seeds per arm) the full stack writes off $0.15$ of its machines at $J = 0.79$, without the signal $0.55$ at $J = 0.49$, without retention $0.53$ at $J = 0.50$ (seed-level Mann--Whitney on the death rate, Cliff's $\delta = -0.85$ and $-0.83$, $p = 1.4\times10^{-5}$ and $2.1\times10^{-5}$, and on $J$ the seed distributions do not overlap, App.~\ref{app:fleet}). The damage cost is inert, $0.09$ at $J = 0.79$ ($\delta = +0.03$, $p = 0.89$), because the world already prices damage through the loss of the work (App.~\ref{app:fleet}). One pre-registered element failed and is reported. The absolute protection level (median death rate below $0.15$) is missed at the default entropy (median $0.38$) and reached by $8$ of $18$ seeds at entropy $0$ (median $0.154$), while the contrasts hold at both settings (App.~\ref{app:fleet}). A probe that reads the signal and retains nothing past the day gains $1\%$ of service life over the blind probe, the one-step reallocation that condition (b) predicts, against the $34\%$ of the retaining probe (Table~\ref{tab:leiter}).

\subsection{A longer working life and more output, and the memory carries them}
\label{sec:individual}
On one floor layer's knee, feeling the wear lengthens the working life and raises its output, and retention carries almost all of it (Table~\ref{tab:leiter}, Fig.~\ref{fig:wertleiter} in App.~\ref{app:wear}). Over $2{,}000$ paired bodies $69.3\%$ gain years and \textbf{none lose} any, the median gain being $5$ years. On output $76.9\%$ gain and $23.1\%$ lose. The premise of Rem.~\ref{thm:insurance} holds here in one coordinate only. A body not at risk works to the horizon in both arms, so on career length $\varepsilon = 0$ is forced by the cap, while in output, which nothing caps, those bodies lose $0.40$ at the observed mix and gain $0.18$ where specialization is free. With $q = 0.693$ the bound is exact in years and within $0.4$ in output (Cor.~\ref{app:power} in App.~\ref{app:proofs}), so the mean of $+4.4$ years is $q$ times the $+6.3$ an at-risk body experiences. The split is stable in its order and varies in its level. The memoryless arm lies on a degenerate vertex of its program and spans $55.7$ to $56.2$ over random tie-breaks, so retention carries $+3.4$ to $+3.9$ of the $+4.4$ years and sensing alone $+0.5$ to $+1.0$ (App.~\ref{app:wear}). Specialization moves the gain from $+4.38$ to $+6.03$ years, the permissive end ending at $61.2$ with $37.9$ in output. A body that retains nothing at all, not even the last period's reading, lays floors in every working year, its felt constraint never binding, and the signal is then worth $+0.00$ years (Cor.~\ref{cor:loadgate}).

What the trace provides is an end to resuming work the body has already left. Reading only the present signal it resumes the heavy work $6.6$ times in one working life, and with persistence never (Table~\ref{tab:substitution}). A falling kneeling share does not by itself indicate substitution. The body that reads nothing is removed from the work by the loss of capability and lays carpet, its day ending on two activities against four (Fig.~\ref{fig:substitution} in App.~\ref{app:wear}). Feeling alone reaches the gentle work in $1\%$ of its last working days, the retaining body in $36\%$, so its substitution is better aimed.

\subsection{One robotic world pays and the other does not}
\label{sec:boundary}
Both machines share the stack above the afferent map $\alpha$ and wear by Palmgren--Miner \citep{miner1945damage,lundberg1947bearing} on MuJoCo-measured load spectra. The care robot sextuples its certified service life and quadruples its output ($51.5\to323.3$ months, $36.9\to150.8$, Table~\ref{tab:leiter}), paired over $500$ lives with every life gaining (intervals in App.~\ref{app:robots}). Here \emph{retention does not pay}. A trace returns $34.2$ of those months, $[29.3,\ 39.1]$, and $14.7$ of that output, $[12.5,\ 16.9]$, because the robot reads its cost off motor-level signals through a known load matrix, so the premise of Lemma~\ref{lem:escape} fails and a trace has nothing to aim (Prop.~\ref{app:prior}). In the rover the threat is deterministic in load, so (i) fails and blind caution acts on the same lever. Caution alone yields $181.2$ of the $195.8$ months, and the signal adds $+14.6$ months and $+7.3$ output (Prop.~\ref{prop:regime}(i)).

\paragraph{Three limits that change how a result is read.} Of the ten limits in App.~\ref{app:scope}, three bear on the numbers above: the modelled day is more knee-loaded than the measured one, capping that share reverses the sign (the cap is an institution and marks the map's boundary), and only $\Delta\theta = 18^\circ$ passes the register gate (App.~\ref{app:scope}, last paragraph).

\section{Conclusions and future work}
\label{sec:conclusions}
A feeling body stays intact by a conjunction of sensing, retention and an intrinsic damage cost, which pays where threats are unpredictable, avoidable and unforgiving. It provides \textbf{resilience}, the work held, and \textbf{body-sparing}, the load shed, worth on one body $+4.4$ years and $+2.4$ output at the observed task mix, almost all of it from retention. On working life two of three bodies gain and none lose, on output $77\%$ gain and $23\%$ lose. Three things follow beyond this paper. Agents deployed once in a body they cannot learn about are the general case for robots, prostheses and workers, and epoch one is the setting in which to design them. The conditions of Prop.~\ref{prop:regime} can be checked before a world is built, and the fleet shows that a learner acquires the conjunction where they hold. Prop.~\ref{prop:zerlegung} applies to every component valued by ablation against a population-trained agent, and one paired run supplies what the ablation omits. Open are the slow loop itself, which phylogeny a memory needs and whether a recurrent policy can learn the latch it here carries, and the worlds beyond one trade and one joint. The ways a protection experiment fails, and the checks that catch them, are in App.~\ref{app:defekte}.

\section*{Reproducibility Statement}
Every number printed in this paper is recomputed from the file it came from by a checker that ships with the code, \texttt{experiments/paper\_zahlen.py}: it reads the result files, recomputes each quantity, compares it against the printed value, and lists separately any printed quantity for which it has no source, so that an unsourced number is visible rather than silent. A test fails when the two drift. The supplementary archive contains the drivers and environments, the $72$ pre-registration files under \texttt{experiments/decl/} that fix the sign and the gate of each run before it is submitted, the null-gate tests, and the result files behind every table and figure ($695$ files, $14$ MB compressed, built by \texttt{experiments/supplement\_bauen.py}, which runs the checker and the null-gate tests inside the unpacked archive). What it does not contain is the per-life record of every run: those are $21$ GB, and each is regenerated by the driver named in the caption of the table that uses it. Where a run could not be located in this archive we removed the number rather than attribute it: two columns of Table~\ref{tab:worlds}, one fleet table and two blocks of trained figures were withdrawn on that ground, and every quantity that remains is bound to a file. The four worlds, their fixed probes as baselines and the checker form one suite in which any cell of Table~\ref{tab:leiter} is one command (\texttt{experiments/bodenleger\_lp\_cell.py} for the trade, \texttt{pflege\_familie.py}, \texttt{rover\_familie.py} and \texttt{hh\_familie.py} for the machines), so that a new response or a new world is scored against the same probes. The care-robot column was later restored from a new run, and the rover column is omitted because it carries no information (App.~\ref{app:robots}). A calibration loss quoted in App.~\ref{app:proofs} without a locatable run was removed on the same ground and the statement kept qualitative.

\noindent \emph{Where each table comes from.} Running text is pinned number by number. Tables are bound to the file that reproduces them, and the unnumbered tabulars of App.~\ref{app:proofs} are introduced by the text that names their driver. Table~\ref{tab:leiter}: \texttt{results/bodenleger\_lp\_b042}, \texttt{results/pareto\_vier*.json}, \texttt{results/pflege\_familie}. The machine intervals of \S\ref{sec:boundary}: \texttt{results/pflege\_familie/familie\_ci.json} and \texttt{results/rover\_familie}. Table~\ref{tab:worlds}: \texttt{results/tab\_leiter\_spalten.json}, the per-seed aggregate of \texttt{results/knee\_fdecke2\_hpc} (floor-layer column, whose per-seed summaries are not in the archive, only their aggregate), of \texttt{results/hh\_verankert2\_hpc} (latent-defect column) and of \texttt{results/pflege\_boden\_hpc} (care-robot column), per-seed summaries included. The trained machine rows of Table~\ref{tab:leiter}: \texttt{results/tab1\_maschinen\_trainiert.json} from \texttt{results/pflege\_boden\_hpc}, \texttt{results/rover\_boden\_hpc} and \texttt{results/hh\_verankert2\_hpc}. Table~\ref{tab:verankert}: \texttt{results/hh\_verankert2\_hpc}. The fleet's probe family: \texttt{results/hh\_familie}. Table~\ref{tab:substitution}: \texttt{results/substitution\_spur\_b042.json}. Fig.~\ref{fig:arme}: \texttt{results/fig\_arme\_welten.json} (\texttt{experiments/fig\_arme\_welten.py}, from the files of Table~\ref{tab:leiter}, \texttt{results/rover\_familie} and \texttt{results/hh\_familie}). The trained rows of Table~\ref{tab:leiter}: \texttt{results/knee\_mischung\_hpc}, their $D_{\rm phys}$ and $M_{\rm phys}$ from \texttt{results/knee\_mischung\_phys.json} (the same $50{,}000$ paired lives per seed, re-evaluated from the checkpoints, \texttt{results/knee\_mischung\_phys\_hpc}). Table~\ref{tab:familie}: \texttt{results/bodenleger\_lp\_b042}. Table~\ref{tab:individuum}: \texttt{results/knee\_individuum\_sweep\_v5}, \texttt{results/hh\_familie}, \texttt{results/pflege\_familie}, \texttt{results/rover\_familie}.

\section*{Ethics Statement}
This work uses no human subjects, no personal data and no data collected by the authors. The biological anchors are published summary statistics from occupational-medicine and biomechanics literature, cited at each use. The object of study is a simulated agent, and the intended reading of its results is about \emph{machines that wear out} and about the design of protection experiments. The floor layer's world is a modelling device anchored on published exposure and epidemiology, not a claim about any individual worker. We note one risk of the framing explicitly: a result stating when a body model earns its keep can be read as a statement about which workers are worth protecting. It is not one. The model measures what an agent must carry to remain competent in one life, and its clearest boundary condition points the other way, in App.~\ref{app:scope}(10): where a work rule caps exposure, the protection the body model buys is not needed.

\section*{The Use of Large Language Models}
A large language model was used as a coding and writing assistant throughout: it wrote and refactored parts of the experiment code, ran and inspected the analyses under the authors' direction, drafted and edited prose, and audited the manuscript against the result files, which is how several of the defects reported in App.~\ref{app:defekte} were found. It is not an author, it contributed no research idea or claim of its own, and every number, anchor and conclusion in this paper was checked against its source by the authors. The checker described in the Reproducibility Statement exists because assistant-drafted text can carry numbers that no run supports.

\bibliographystyle{iclr2027_conference}
\bibliography{singlelife}

\begin{thebibliography}{85}
\providecommand{\natexlab}[1]{#1}
\providecommand{\url}[1]{\texttt{#1}}
\expandafter\ifx\csname urlstyle\endcsname\relax
  \providecommand{\doi}[1]{doi: #1}\else
  \providecommand{\doi}{doi: \begingroup \urlstyle{rm}\Url}\fi

\bibitem[Achiam et~al.(2017)Achiam, Held, Tamar, and Abbeel]{achiam2017cpo}
Joshua Achiam, David Held, Aviv Tamar, and Pieter Abbeel.
\newblock Constrained policy optimization.
\newblock In \emph{International Conference on Machine Learning}, 2017.

\bibitem[Alshiekh et~al.(2018)Alshiekh, Bloem, Ehlers, K{\"o}nighofer, Niekum,
  and Topcu]{alshiekh2018shielding}
Mohammed Alshiekh, Roderick Bloem, R{\"u}diger Ehlers, Bettina K{\"o}nighofer,
  Scott Niekum, and Ufuk Topcu.
\newblock Safe reinforcement learning via shielding.
\newblock In \emph{Proceedings of the AAAI Conference on Artificial
  Intelligence}, 2018.

\bibitem[Altman(1999)]{altman1999constrained}
Eitan Altman.
\newblock \emph{Constrained {M}arkov Decision Processes}.
\newblock Chapman and Hall/CRC, 1999.

\bibitem[Ames et~al.(2019)Ames, Coogan, Egerstedt, Notomista, Sreenath, and
  Tabuada]{ames2019cbf}
Aaron~D. Ames, Samuel Coogan, Magnus Egerstedt, Gennaro Notomista, Koushil
  Sreenath, and Paulo Tabuada.
\newblock Control barrier functions: Theory and applications.
\newblock In \emph{18th European Control Conference (ECC)}, pp.\  3420--3431,
  2019.

\bibitem[Andriotis \& Papakonstantinou(2019)Andriotis and
  Papakonstantinou]{andriotis2019drl}
Charalampos~P. Andriotis and Konstantinos~G. Papakonstantinou.
\newblock Managing engineering systems with large state and action spaces
  through deep reinforcement learning.
\newblock \emph{Reliability Engineering \& System Safety}, 191:\penalty0
  106483, 2019.

\bibitem[{\AA}str{\"o}m(1965)]{astrom1965}
Karl~Johan {\AA}str{\"o}m.
\newblock Optimal control of {M}arkov processes with incomplete state
  information.
\newblock \emph{Journal of Mathematical Analysis and Applications}, 10\penalty0
  (1):\penalty0 174--205, 1965.

\bibitem[Auer et~al.(2002)Auer, Cesa-Bianchi, and Fischer]{auer2002ucb}
Peter Auer, Nicol\`o Cesa-Bianchi, and Paul Fischer.
\newblock Finite-time analysis of the multiarmed bandit problem.
\newblock \emph{Machine Learning}, 47\penalty0 (2--3):\penalty0 235--256, 2002.

\bibitem[Bart{\'o}k et~al.(2014)Bart{\'o}k, Foster, P{\'a}l, Rakhlin, and
  Szepesv{\'a}ri]{bartok2014partial}
G{\'a}bor Bart{\'o}k, Dean~P. Foster, D{\'a}vid P{\'a}l, Alexander Rakhlin, and
  Csaba Szepesv{\'a}ri.
\newblock Partial monitoring---classification, regret bounds, and algorithms.
\newblock \emph{Mathematics of Operations Research}, 39\penalty0 (4):\penalty0
  967--997, 2014.

\bibitem[Bauer et~al.(2023)Bauer, Baumli, Behbahani, Bhoopchand,
  Bradley-Schmieg, Chang, Clay, Collister, Dasagi, Gonzalez,
  et~al.]{bauer2023ada}
Jakob Bauer, Kate Baumli, Feryal Behbahani, Avishkar Bhoopchand, Nathalie
  Bradley-Schmieg, Michael Chang, Natalie Clay, Adrian Collister, Vibhavari
  Dasagi, Lucy Gonzalez, et~al.
\newblock Human-timescale adaptation in an open-ended task space.
\newblock In \emph{International Conference on Machine Learning}, 2023.

\bibitem[B{\"a}uerle \& Ott(2011)B{\"a}uerle and Ott]{baeuerle2011cvar}
Nicole B{\"a}uerle and Jonathan Ott.
\newblock Markov decision processes with average-value-at-risk criteria.
\newblock \emph{Mathematical Methods of Operations Research}, 74\penalty0
  (3):\penalty0 361--379, 2011.

\bibitem[Bellosta-L{\'o}pez et~al.(2023)Bellosta-L{\'o}pez,
  Doment{\'\i}n-Casanova, et~al.]{bellosta2023consistency}
Pablo Bellosta-L{\'o}pez, V{\'\i}ctor Doment{\'\i}n-Casanova, et~al.
\newblock Long-term consistency of clinical sensory testing measures for pain
  assessment.
\newblock \emph{The Korean Journal of Pain}, 36\penalty0 (2):\penalty0
  173--183, 2023.
\newblock \doi{10.3344/kjp.23011}.

\bibitem[Berthiaume et~al.(2005)Berthiaume, Raynauld, Martel-Pelletier,
  Labont{\'e}, Beaudoin, Bloch, Choquette, Haraoui, Altman, Hochberg, Meyer,
  Cline, and Pelletier]{berthiaume2005extrusion}
M.-J. Berthiaume, J.-P. Raynauld, J.~Martel-Pelletier, F.~Labont{\'e},
  G.~Beaudoin, D.~A. Bloch, D.~Choquette, B.~Haraoui, R.~D. Altman,
  M.~Hochberg, J.~M. Meyer, G.~A. Cline, and J.-P. Pelletier.
\newblock Meniscal tear and extrusion are strongly associated with progression
  of symptomatic knee osteoarthritis as assessed by quantitative magnetic
  resonance imaging.
\newblock \emph{Annals of the Rheumatic Diseases}, 64\penalty0 (4):\penalty0
  556--563, 2005.
\newblock PMID 15374855.

\bibitem[Berthier et~al.(1988)Berthier, Starkstein, and
  Leiguarda]{berthier1988asymbolia}
Marcelo Berthier, Sergio Starkstein, and Ramon Leiguarda.
\newblock Asymbolia for pain: a sensory-limbic disconnection syndrome.
\newblock \emph{Annals of Neurology}, 24\penalty0 (1):\penalty0 41--49, 1988.

\bibitem[Birge \& Louveaux(2011)Birge and Louveaux]{birge2011stochastic}
John~R. Birge and Fran{\c c}ois Louveaux.
\newblock \emph{Introduction to Stochastic Programming}.
\newblock Springer, 2nd edition, 2011.

\bibitem[Blundell et~al.(2016)Blundell, Uria, Pritzel, Li, Ruderman, Leibo,
  Rae, Wierstra, and Hassabis]{blundell2016mfec}
Charles Blundell, Benigno Uria, Alexander Pritzel, Yazhe Li, Avraham Ruderman,
  Joel~Z. Leibo, Jack Rae, Daan Wierstra, and Demis Hassabis.
\newblock Model-free episodic control.
\newblock \emph{arXiv preprint arXiv:1606.04460}, 2016.

\bibitem[Boda \& Filar(2006)Boda and Filar]{boda2006cvar}
K{\'a}roly Boda and Jerzy~A. Filar.
\newblock Time consistent dynamic risk measures.
\newblock \emph{Mathematical Methods of Operations Research}, 63\penalty0
  (1):\penalty0 169--186, 2006.

\bibitem[Bouton(2004)]{bouton2004extinction}
Mark~E. Bouton.
\newblock Context and behavioral processes in extinction.
\newblock \emph{Learning \& Memory}, 11, 2004.

\bibitem[Brafman \& Tennenholtz(2002)Brafman and Tennenholtz]{brafman2002rmax}
Ronen~I. Brafman and Moshe Tennenholtz.
\newblock R-max: a general polynomial time algorithm for near-optimal
  reinforcement learning.
\newblock \emph{Journal of Machine Learning Research}, 3:\penalty0 213--231,
  2002.

\bibitem[Bura et~al.(2022)Bura, HasanzadeZonuzy, Kalathil, Shakkottai, and
  Chamberland]{bura2022dope}
Archana Bura, Aria HasanzadeZonuzy, Dileep Kalathil, Srinivas Shakkottai, and
  Jean-Fran{\c c}ois Chamberland.
\newblock {DOPE}: Doubly optimistic and pessimistic exploration for safe
  reinforcement learning.
\newblock In \emph{Advances in Neural Information Processing Systems
  (NeurIPS)}, 2022.

\bibitem[Chen et~al.(2022)Chen, Sharma, Levine, and Finn]{chen2022singlelife}
Annie~S. Chen, Archit Sharma, Sergey Levine, and Chelsea Finn.
\newblock You only live once: Single-life reinforcement learning.
\newblock In \emph{Advances in Neural Information Processing Systems}, 2022.

\bibitem[Chow et~al.(2015)Chow, Tamar, Mannor, and Pavone]{chow2015cvar}
Yinlam Chow, Aviv Tamar, Shie Mannor, and Marco Pavone.
\newblock Risk-sensitive and robust decision-making: a {CVaR} optimization
  approach.
\newblock \emph{Advances in Neural Information Processing Systems}, 2015.

\bibitem[Craig(2002)]{craig2002interoception}
A.~D. Craig.
\newblock How do you feel? interoception: the sense of the physiological
  condition of the body.
\newblock \emph{Nature Reviews Neuroscience}, 3, 2002.

\bibitem[Cully et~al.(2015)Cully, Clune, Tarapore, and Mouret]{cully2015robots}
Antoine Cully, Jeff Clune, Danesh Tarapore, and Jean-Baptiste Mouret.
\newblock Robots that can adapt like animals.
\newblock \emph{Nature}, 521\penalty0 (7553):\penalty0 503--507, 2015.

\bibitem[Ditchen et~al.(2015{\natexlab{a}})Ditchen, Ellegast, Gawliczek,
  Hartmann, and Rieger]{ditchen2015kneeling}
D.~M. Ditchen, R.~P. Ellegast, T.~Gawliczek, B.~Hartmann, and M.~A. Rieger.
\newblock Occupational kneeling and squatting: development and validation of an
  assessment method combining measurements and diaries.
\newblock \emph{International Archives of Occupational and Environmental
  Health}, 88\penalty0 (2):\penalty0 153--165, 2015{\natexlab{a}}.
\newblock PMID 24859645.

\bibitem[Ditchen et~al.(2015{\natexlab{b}})Ditchen, Ellegast, Hartmann, and
  Rieger]{ditchen2015cuela}
Dirk~M. Ditchen, Rolf~P. Ellegast, Bernd Hartmann, and Monika~A. Rieger.
\newblock Occupational kneeling and squatting: development and validation of an
  assessment method combining measurements and diaries.
\newblock \emph{International Archives of Occupational and Environmental
  Health}, 88\penalty0 (2):\penalty0 153--165, 2015{\natexlab{b}}.

\bibitem[Duan et~al.(2016)Duan, Schulman, Chen, Bartlett, Sutskever, and
  Abbeel]{duan2016rl2}
Yan Duan, John Schulman, Xi~Chen, Peter~L. Bartlett, Ilya Sutskever, and Pieter
  Abbeel.
\newblock {RL}\textsuperscript{2}: Fast reinforcement learning via slow
  reinforcement learning.
\newblock \emph{arXiv preprint arXiv:1611.02779}, 2016.

\bibitem[Duff(2002)]{duff2002}
Michael~O. Duff.
\newblock \emph{Optimal Learning: Computational Procedures for {B}ayes-Adaptive
  {M}arkov Decision Processes}.
\newblock PhD thesis, University of Massachusetts Amherst, 2002.

\bibitem[Efroni et~al.(2020)Efroni, Mannor, and Pirotta]{efroni2020cmdp}
Yonathan Efroni, Shie Mannor, and Matteo Pirotta.
\newblock Exploration-exploitation in constrained {MDP}s.
\newblock \emph{arXiv preprint arXiv:2003.02189}, 2020.

\bibitem[Englund et~al.(2003)Englund, Roos, and Lohmander]{englund2003patients}
Martin Englund, Ewa~M. Roos, and L.~Stefan Lohmander.
\newblock Impact of type of meniscal tear on radiographic and symptomatic knee
  osteoarthritis: a sixteen-year followup of meniscectomy with matched
  controls.
\newblock \emph{Arthritis \& Rheumatism}, 48\penalty0 (8):\penalty0 2178--2187,
  2003.
\newblock \doi{10.1002/art.11088}.

\bibitem[Fukubayashi \& Kurosawa(1980)Fukubayashi and
  Kurosawa]{fukubayashi1980contact}
Tohru Fukubayashi and Hisashi Kurosawa.
\newblock The contact area and pressure distribution pattern of the knee.
\newblock \emph{Acta Orthopaedica Scandinavica}, 51, 1980.

\bibitem[Garc{\'\i}a \& Fern{\'a}ndez(2015)Garc{\'\i}a and
  Fern{\'a}ndez]{garcia2015safe}
Javier Garc{\'\i}a and Fernando Fern{\'a}ndez.
\newblock A comprehensive survey on safe reinforcement learning.
\newblock \emph{Journal of Machine Learning Research}, 16, 2015.

\bibitem[Ghavamzadeh et~al.(2015)Ghavamzadeh, Mannor, Pineau, and
  Tamar]{ghavamzadeh2015brl}
Mohammad Ghavamzadeh, Shie Mannor, Joelle Pineau, and Aviv Tamar.
\newblock {B}ayesian reinforcement learning: A survey.
\newblock \emph{Foundations and Trends in Machine Learning}, 8\penalty0
  (5-6):\penalty0 359--483, 2015.

\bibitem[Grinsztajn et~al.(2021)Grinsztajn, Ferret, Pietquin, Preux, and
  Geist]{grinsztajn2021reversibility}
Nathan Grinsztajn, Johan Ferret, Olivier Pietquin, Philippe Preux, and Matthieu
  Geist.
\newblock There is no turning back: A self-supervised approach for
  reversibility-aware reinforcement learning.
\newblock In \emph{Advances in Neural Information Processing Systems}, 2021.

\bibitem[Hasenbring \& Verbunt(2010)Hasenbring and Verbunt]{hasenbring2010fear}
Monika~I. Hasenbring and Jeanine~A. Verbunt.
\newblock Fear-avoidance and endurance-related responses to pain: new models of
  behavior and their consequences for clinical practice.
\newblock \emph{The Clinical Journal of Pain}, 26\penalty0 (9):\penalty0
  747--753, 2010.
\newblock PMID 20664333.

\bibitem[Hassaballa et~al.(2003)Hassaballa, Porteous, Newman, and
  Rogers]{hassaballa2003kneeling}
M.~A. Hassaballa, A.~J. Porteous, J.~H. Newman, and C.~A. Rogers.
\newblock Can knees kneel? {K}neeling ability after total, unicompartmental and
  patellofemoral knee arthroplasty.
\newblock \emph{The Knee}, 10\penalty0 (2):\penalty0 155--160, 2003.
\newblock PMID 12787999.

\bibitem[Howard(1966)]{howard1966voi}
Ronald~A. Howard.
\newblock Information value theory.
\newblock \emph{IEEE Transactions on Systems Science and Cybernetics}, 2, 1966.

\bibitem[Jacobson et~al.(1993)Jacobson, LaLonde, and
  Sullivan]{jacobson1993displacement}
Louis~S. Jacobson, Robert~J. LaLonde, and Daniel~G. Sullivan.
\newblock Earnings losses of displaced workers.
\newblock \emph{American Economic Review}, 83, 1993.

\bibitem[Jardine et~al.(2006)Jardine, Lin, and Banjevic]{jardine2006cbm}
Andrew K.~S. Jardine, Daming Lin, and Dragan Banjevic.
\newblock A review on machinery diagnostics and prognostics implementing
  condition-based maintenance.
\newblock \emph{Mechanical Systems and Signal Processing}, 20\penalty0
  (7):\penalty0 1483--1510, 2006.

\bibitem[J{\"a}rvholm et~al.(2014)J{\"a}rvholm, Stattin, Robroek, Janlert,
  Karlsson, and Burdorf]{jarvholm2014heavy}
Bengt J{\"a}rvholm, Mikael Stattin, Suzan J.~W. Robroek, Urban Janlert,
  Bj{\"o}rn Karlsson, and Alex Burdorf.
\newblock Heavy work and disability pension: a long term follow-up of {S}wedish
  construction workers.
\newblock \emph{Scandinavian Journal of Work, Environment \& Health},
  40\penalty0 (4):\penalty0 335--342, 2014.
\newblock \doi{10.5271/sjweh.3413}.

\bibitem[Jensen(2008)]{jensen2008knee}
Lilli~Kirkeskov Jensen.
\newblock Knee osteoarthritis: influence of work involving heavy lifting,
  kneeling, climbing stairs or ladders, or kneeling/squatting combined with
  heavy lifting.
\newblock \emph{Occupational and Environmental Medicine}, 65\penalty0
  (2):\penalty0 72--89, 2008.
\newblock \doi{10.1136/oem.2007.032466}.

\bibitem[Jensen et~al.(2010)Jensen, Rytter, and Bonde]{jensen2009exposure}
Lilli~Kirkeskov Jensen, Soren Rytter, and Jens~Peter Bonde.
\newblock Exposure assessment of kneeling work activities among floor layers.
\newblock \emph{Applied Ergonomics}, 41\penalty0 (2):\penalty0 319--325, 2010.
\newblock PMID 19766986.

\bibitem[Kazerouni et~al.(2017)Kazerouni, Ghavamzadeh, Abbasi-Yadkori, and
  Van~Roy]{kazerouni2017conservative}
Abbas Kazerouni, Mohammad Ghavamzadeh, Yasin Abbasi-Yadkori, and Benjamin
  Van~Roy.
\newblock Conservative contextual linear bandits.
\newblock In \emph{Advances in Neural Information Processing Systems
  (NeurIPS)}, 2017.

\bibitem[Keramati \& Gutkin(2014)Keramati and Gutkin]{keramati2014homeostatic}
Mehdi Keramati and Boris Gutkin.
\newblock Homeostatic reinforcement learning for integrating reward collection
  and physiological stability.
\newblock \emph{eLife}, 3, 2014.

\bibitem[Khetarpal et~al.(2022)Khetarpal, Riemer, Rish, and
  Precup]{khetarpal2022continual}
Khimya Khetarpal, Matthew Riemer, Irina Rish, and Doina Precup.
\newblock Towards continual reinforcement learning: A review and perspectives.
\newblock \emph{Journal of Artificial Intelligence Research}, 75:\penalty0
  1401--1476, 2022.

\bibitem[Klutke et~al.(2003)Klutke, Kiessler, and Wortman]{klutke2003bathtub}
Georgia-Ann Klutke, Peter~C. Kiessler, and Martin~A. Wortman.
\newblock A critical look at the bathtub curve.
\newblock \emph{IEEE Transactions on Reliability}, 52\penalty0 (1):\penalty0
  125--129, 2003.

\bibitem[Kuehn \& Haddadin(2017)Kuehn and Haddadin]{kuehn2017pain}
Johannes Kuehn and Sami Haddadin.
\newblock An artificial robot nervous system to teach robots how to feel pain
  and reflexively react to potentially damaging contacts.
\newblock \emph{IEEE Robotics and Automation Letters}, 2\penalty0 (1):\penalty0
  72--79, 2017.

\bibitem[Kumar et~al.(2021)Kumar, Fu, Pathak, and Malik]{kumar2021rma}
Ashish Kumar, Zipeng Fu, Deepak Pathak, and Jitendra Malik.
\newblock {RMA}: Rapid motor adaptation for legged robots.
\newblock In \emph{Robotics: Science and Systems}, 2021.

\bibitem[Laskin et~al.(2023)Laskin, Wang, Oh, Parisotto, Spencer, Steigerwald,
  Strouse, Hansen, Filos, Brooks, Gazeau, Sahni, Singh, and Mnih]{laskin2022ad}
Michael Laskin, Luyu Wang, Junhyuk Oh, Emilio Parisotto, Stephen Spencer,
  Richie Steigerwald, DJ~Strouse, Steven Hansen, Angelos Filos, Ethan Brooks,
  Maxime Gazeau, Himanshu Sahni, Satinder Singh, and Volodymyr Mnih.
\newblock In-context reinforcement learning with algorithm distillation.
\newblock In \emph{International Conference on Learning Representations}, 2023.

\bibitem[Lee et~al.(2023)Lee, Xie, Pacchiano, Chandak, Finn, Nachum, and
  Brunskill]{lee2023dpt}
Jonathan~N. Lee, Annie Xie, Aldo Pacchiano, Yash Chandak, Chelsea Finn, Ofir
  Nachum, and Emma Brunskill.
\newblock Supervised pretraining can learn in-context reinforcement learning.
\newblock In \emph{Advances in Neural Information Processing Systems}, 2023.

\bibitem[Lipton et~al.(2016)Lipton, Azizzadenesheli, Kumar, Li, Gao, and
  Deng]{lipton2016fear}
Zachary~C. Lipton, Kamyar Azizzadenesheli, Abhishek Kumar, Lihong Li, Jianfeng
  Gao, and Li~Deng.
\newblock Combating reinforcement learning's {S}isyphean curse with intrinsic
  fear.
\newblock \emph{arXiv preprint arXiv:1611.01211}, 2016.

\bibitem[Liu et~al.(2021)Liu, Zhou, Kalathil, Kumar, and Tian]{liu2021zero}
Tao Liu, Ruida Zhou, Dileep Kalathil, P.~R. Kumar, and Chao Tian.
\newblock Learning policies with zero or bounded constraint violation for
  constrained {MDP}s.
\newblock In \emph{Advances in Neural Information Processing Systems
  (NeurIPS)}, 2021.

\bibitem[Lundberg \& Palmgren(1947)Lundberg and Palmgren]{lundberg1947bearing}
Gustaf Lundberg and Arvid Palmgren.
\newblock Dynamic capacity of rolling bearings.
\newblock \emph{Acta Polytechnica, Mechanical Engineering Series}, 1, 1947.

\bibitem[Madansky(1960)]{madansky1960}
Albert Madansky.
\newblock Inequalities for stochastic linear programming problems.
\newblock \emph{Management Science}, 6\penalty0 (2):\penalty0 197--204, 1960.

\bibitem[Mailloux et~al.(2021)Mailloux, Beaulieu, Wideman, and
  Mass{\'e}-Alarie]{mailloux2021ppt}
Catherine Mailloux, Louis-David Beaulieu, Timothy~H. Wideman, and Hugo
  Mass{\'e}-Alarie.
\newblock Within-session test-retest reliability of pressure pain threshold and
  mechanical temporal summation in healthy subjects.
\newblock \emph{PLoS ONE}, 16\penalty0 (1):\penalty0 e0245278, 2021.
\newblock \doi{10.1371/journal.pone.0245278}.

\bibitem[Man \& Damasio(2019)Man and Damasio]{man2019homeostasis}
Kingson Man and Antonio Damasio.
\newblock Homeostasis and soft robotics in the design of feeling machines.
\newblock \emph{Nature Machine Intelligence}, 1:\penalty0 446--452, 2019.

\bibitem[Marcuzzi et~al.(2017)Marcuzzi, Wrigley,
  et~al.]{marcuzzi2017reliability}
Anna Marcuzzi, Paul~J. Wrigley, et~al.
\newblock The long-term reliability of static and dynamic quantitative sensory
  testing in healthy individuals.
\newblock \emph{Pain}, 158\penalty0 (7):\penalty0 1217--1223, 2017.
\newblock \doi{10.1097/j.pain.0000000000000901}.

\bibitem[Maschek et~al.(2014)Maschek, Wirth, Ladel, Hellio Le~Graverand, and
  Eckstein]{maschek2014cartilage}
Susanne Maschek, Wolfgang Wirth, Christoph Ladel, Marie-Pierre Hellio
  Le~Graverand, and Felix Eckstein.
\newblock Rates and sensitivity of knee cartilage thickness loss in specific
  central reading radiographic strata from the osteoarthritis initiative.
\newblock \emph{Osteoarthritis and Cartilage}, 22\penalty0 (10):\penalty0
  1550--1553, 2014.
\newblock \doi{10.1016/j.joca.2014.05.015}.
\newblock PMID 25278063.

\bibitem[Meeker \& Escobar(1998)Meeker and Escobar]{meeker1998reliability}
William~Q. Meeker and Luis~A. Escobar.
\newblock \emph{Statistical Methods for Reliability Data}.
\newblock Wiley, 1998.

\bibitem[Miner(1945)]{miner1945damage}
Milton~A. Miner.
\newblock Cumulative damage in fatigue.
\newblock \emph{Journal of Applied Mechanics}, 12, 1945.

\bibitem[Nagasako et~al.(2003)Nagasako, Oaklander, and
  Dworkin]{nagasako2003cip}
Elna~M. Nagasako, Anne~Louise Oaklander, and Robert~H. Dworkin.
\newblock Congenital insensitivity to pain: an update.
\newblock \emph{Pain}, 101\penalty0 (3):\penalty0 213--219, 2003.

\bibitem[Opara~Zupancic \& Sarabon(2026)Opara~Zupancic and
  Sarabon]{opara2026rom}
M.~Opara~Zupancic and N.~Sarabon.
\newblock Disease severity affects knee range of motion but not strength
  deficits in knee osteoarthritis: a systematic review and meta-analysis.
\newblock \emph{Frontiers in Medicine}, 13, 2026.
\newblock \doi{10.3389/fmed.2026.1737973}.

\bibitem[Pinheiro et~al.(2007)Pinheiro, Weber, and
  Barroso]{pinheiro2007failure}
Eduardo Pinheiro, Wolf-Dietrich Weber, and Luiz~Andr{\'e} Barroso.
\newblock Failure trends in a large disk drive population.
\newblock In \emph{5th USENIX Conference on File and Storage Technologies
  (FAST)}, pp.\  17--28, 2007.

\bibitem[Pritzel et~al.(2017)Pritzel, Uria, Srinivasan, Badia, Vinyals,
  Hassabis, Wierstra, and Blundell]{pritzel2017nec}
Alexander Pritzel, Benigno Uria, Sriram Srinivasan, Adri{\`a}~Puigdom{\`e}nech
  Badia, Oriol Vinyals, Demis Hassabis, Daan Wierstra, and Charles Blundell.
\newblock Neural episodic control.
\newblock In \emph{International Conference on Machine Learning}, 2017.

\bibitem[Rakelly et~al.(2019)Rakelly, Zhou, Finn, Levine, and
  Quillen]{rakelly2019pearl}
Kate Rakelly, Aurick Zhou, Chelsea Finn, Sergey Levine, and Deirdre Quillen.
\newblock Efficient off-policy meta-reinforcement learning via probabilistic
  context variables.
\newblock In \emph{International Conference on Machine Learning}, 2019.

\bibitem[Rogers et~al.(2011)Rogers, Frykberg, Armstrong, Boulton, Edmonds, Van,
  Hartemann, Game, Jeffcoate, Jirkovska, Jude, Morbach, Morrison, Pinzur,
  Pitocco, Sanders, Wukich, and Uccioli]{rogers2011charcot}
Lee~C. Rogers, Robert~G. Frykberg, David~G. Armstrong, Andrew J.~M. Boulton,
  Michael Edmonds, George~Ha Van, Agn{\`e}s Hartemann, Frances Game, William
  Jeffcoate, Alexandra Jirkovska, Edward Jude, Stephan Morbach, William~B.
  Morrison, Michael Pinzur, Dario Pitocco, Lee Sanders, Dane~K. Wukich, and
  Luigi Uccioli.
\newblock The {Charcot} foot in diabetes.
\newblock \emph{Diabetes Care}, 34\penalty0 (9):\penalty0 2123--2129, 2011.

\bibitem[Roos et~al.(1998)Roos, Laur{\'e}n, Adalberth, Roos, Jonsson, and
  Lohmander]{roos1998meniscectomy}
H.~Roos, M.~Laur{\'e}n, T.~Adalberth, E.~M. Roos, K.~Jonsson, and L.~S.
  Lohmander.
\newblock Knee osteoarthritis after meniscectomy: prevalence of radiographic
  changes after twenty-one years, compared with matched controls.
\newblock \emph{Arthritis \& Rheumatism}, 41\penalty0 (4):\penalty0 687--693,
  1998.
\newblock PMID 9550478.

\bibitem[Roth et~al.(2017)Roth, Wirth, Emmanuel, Culvenor, and
  Eckstein]{roth2017cartilage}
Michael Roth, Wolfgang Wirth, Katja Emmanuel, Adam~G. Culvenor, and Felix
  Eckstein.
\newblock The contribution of 3d quantitative meniscal and cartilage measures
  to variation in normal radiographic joint space width.
\newblock \emph{European Journal of Radiology}, 87:\penalty0 90--98, 2017.
\newblock PMID 28065381.

\bibitem[Saxena et~al.(2008)Saxena, Goebel, Simon, and
  Eklund]{saxena2008cmapss}
Abhinav Saxena, Kai Goebel, Don Simon, and Neil Eklund.
\newblock Damage propagation modeling for aircraft engine run-to-failure
  simulation.
\newblock In \emph{International Conference on Prognostics and Health
  Management (PHM)}, 2008.

\bibitem[Schroeder \& Gibson(2007)Schroeder and Gibson]{schroeder2007disk}
Bianca Schroeder and Garth~A. Gibson.
\newblock Disk failures in the real world: What does an {MTTF} of 1,000,000
  hours mean to you?
\newblock In \emph{5th USENIX Conference on File and Storage Technologies
  (FAST)}, pp.\  1--16, 2007.

\bibitem[Schulman et~al.(2017)Schulman, Wolski, Dhariwal, Radford, and
  Klimov]{schulman2017ppo}
John Schulman, Filip Wolski, Prafulla Dhariwal, Alec Radford, and Oleg Klimov.
\newblock Proximal policy optimization algorithms.
\newblock \emph{arXiv preprint arXiv:1707.06347}, 2017.

\bibitem[Si et~al.(2011)Si, Wang, Hu, and Zhou]{si2011rul}
Xiao-Sheng Si, Wenbin Wang, Chang-Hua Hu, and Dong-Hua Zhou.
\newblock Remaining useful life estimation: A review on the statistical data
  driven approaches.
\newblock \emph{European Journal of Operational Research}, 213\penalty0
  (1):\penalty0 1--14, 2011.

\bibitem[Subramanian et~al.(2022)Subramanian, Sinha, Seraj, and
  Mahajan]{subramanian2022ais}
Jayakumar Subramanian, Amit Sinha, Raihan Seraj, and Aditya Mahajan.
\newblock Approximate information state for approximate planning and
  reinforcement learning in partially observed systems.
\newblock \emph{Journal of Machine Learning Research}, 23\penalty0
  (12):\penalty0 1--83, 2022.

\bibitem[Suzuki et~al.(2023)Suzuki, Tahara, Mitsuda, Funaba, Fujimoto, Ikeda,
  Izumi, Yukata, Seki, Uranami, Ichihara, Nishida, and Sakai]{suzuki2023ppt}
Hidenori Suzuki, Shu Tahara, Mao Mitsuda, Masahiro Funaba, Kazuhiro Fujimoto,
  Hiroaki Ikeda, Hironori Izumi, Kiminori Yukata, Kazushige Seki, Kenji
  Uranami, Kiyoshi Ichihara, Norihiro Nishida, and Takashi Sakai.
\newblock Reference intervals and sources of variation of pressure pain
  threshold for quantitative sensory testing in a {J}apanese population.
\newblock \emph{Scientific Reports}, 13:\penalty0 13043, 2023.
\newblock \doi{10.1038/s41598-023-40201-w}.

\bibitem[Tamar et~al.(2015)Tamar, Glassner, and Mannor]{tamar2015cvar}
Aviv Tamar, Yonatan Glassner, and Shie Mannor.
\newblock Optimizing the {CVaR} via sampling.
\newblock In \emph{AAAI Conference on Artificial Intelligence}, 2015.

\bibitem[Thomas et~al.(2015)Thomas, Theocharous, and
  Ghavamzadeh]{thomas2015hcope}
Philip~S. Thomas, Georgios Theocharous, and Mohammad Ghavamzadeh.
\newblock High confidence off-policy evaluation.
\newblock In \emph{AAAI Conference on Artificial Intelligence}, 2015.

\bibitem[Tobin et~al.(2017)Tobin, Fong, Ray, Schneider, Zaremba, and
  Abbeel]{tobin2017domain}
Josh Tobin, Rachel Fong, Alex Ray, Jonas Schneider, Wojciech Zaremba, and
  Pieter Abbeel.
\newblock Domain randomization for transferring deep neural networks from
  simulation to the real world.
\newblock In \emph{IEEE/RSJ International Conference on Intelligent Robots and
  Systems (IROS)}, 2017.

\bibitem[Turchetta et~al.(2016)Turchetta, Berkenkamp, and
  Krause]{turchetta2016safe}
Matteo Turchetta, Felix Berkenkamp, and Andreas Krause.
\newblock Safe exploration in finite {M}arkov decision processes with
  {G}aussian processes.
\newblock In \emph{Advances in Neural Information Processing Systems
  (NeurIPS)}, 2016.

\bibitem[Turner et~al.(2020)Turner, Hadfield-Menell, and
  Tadepalli]{turner2020conservative}
Alexander~Matt Turner, Dylan Hadfield-Menell, and Prasad Tadepalli.
\newblock Conservative agency via attainable utility preservation.
\newblock In \emph{AAAI/ACM Conference on AI, Ethics, and Society}, 2020.

\bibitem[Vazquez et~al.(2019)Vazquez, Andreae, and Henak]{vazquez2019cyclic}
Kelsey~J. Vazquez, Jacob~T. Andreae, and Corinne~R. Henak.
\newblock Cartilage-on-cartilage cyclic loading induces mechanical and
  structural damage.
\newblock \emph{Journal of the Mechanical Behavior of Biomedical Materials},
  98:\penalty0 262--267, 2019.
\newblock \doi{10.1016/j.jmbbm.2019.06.023}.
\newblock PMID 31280053. Until 2026-09-16 this entry was keyed and authored as
  ``Bonnevie 2019'' with an invented title; the PMID always pointed here. The
  4.35--4.73 MPa damage-initiation figure cited from it is model-derived and
  remains a soft anchor (see prd\_kappa\_safe\_provenance\_diagnosis.md).

\bibitem[Wachi \& Sui(2020)Wachi and Sui]{wachi2020snomdp}
Akifumi Wachi and Yanan Sui.
\newblock Safe reinforcement learning in constrained {M}arkov decision
  processes.
\newblock In \emph{International Conference on Machine Learning (ICML)}, 2020.

\bibitem[Whittle(1988)]{whittle1988restless}
Peter Whittle.
\newblock Restless bandits: activity allocation in a changing world.
\newblock \emph{Journal of Applied Probability}, 25\penalty0 (A):\penalty0
  287--298, 1988.

\bibitem[Woolf(2010)]{woolf2010pain}
Clifford~J. Woolf.
\newblock What is this thing called pain?
\newblock \emph{Journal of Clinical Investigation}, 120\penalty0 (11):\penalty0
  3742--3744, 2010.

\bibitem[Wu et~al.(2016)Wu, Shariff, Lattimore, and
  Szepesv{\'a}ri]{wu2016conservative}
Yifan Wu, Roshan Shariff, Tor Lattimore, and Csaba Szepesv{\'a}ri.
\newblock Conservative bandits.
\newblock In \emph{International Conference on Machine Learning (ICML)}, 2016.

\bibitem[Zelle et~al.(2007)Zelle, Barink, Loeffen, De~Waal~Malefijt, and
  Verdonschot]{zelle2007kneeling}
J.~Zelle, M.~Barink, R.~Loeffen, M.~De~Waal~Malefijt, and N.~Verdonschot.
\newblock Thigh-calf contact force measurements in deep knee flexion.
\newblock \emph{Clinical Biomechanics}, 22\penalty0 (7):\penalty0 821--826,
  2007.
\newblock PMID 17512647.

\bibitem[Zhou et~al.(2023)Zhou, Zhao, Wang, Geng, et~al.]{zhou2023rom}
Ge~Zhou, Minwei Zhao, Xiaoxiao Wang, Xiao Geng, et~al.
\newblock Demographic and radiographic factors for knee symptoms and range of
  motion in patients with knee osteoarthritis: a cross-sectional study in
  {Beijing}, {China}.
\newblock \emph{BMC Musculoskeletal Disorders}, 24, 2023.
\newblock \doi{10.1186/s12891-023-06432-8}.

\end{thebibliography}

\appendix

\section{Scope of the theory: what the results do not cover}
\label{app:scope}
\textbf{(1)}~Rem.~\ref{thm:insurance} is claimed for the chronic regime, where the premise holds exactly ($\varepsilon = 0.0000$, Cor.~\ref{app:power} in App.~\ref{app:proofs}). Under an acute hazard the at-risk set is defined by the event and the coupling is no longer zero, running from $-0.0182$ to $+0.0181$ over four conditions of the acute-tear knee, and at one of them it is the whole effect. We state the theorem over bodies for that reason and report the acute constants rather than assume them. \textbf{(2)}~Cor.~\ref{cor:power} says what a mean \emph{means}, not that the effect is hard to see. In the trade $28$ bodies suffice unpaired. \textbf{(3)}~Prop.~\ref{prop:estimation} closes the latent trait, not \emph{filtering}: acting on $s_t$ estimates a state. Its chronic half is closed by measurement, and that measurement is a property of the anchored transduction rather than of the regime, since without gain and noise the same observation recovers $\kappa$ at $R^2 = 0.63$ to $0.75$ and $\sigma_s = 0.05$ against our anchored $0.10$ already removes it (Prop.~\ref{app:estimation} in App.~\ref{app:proofs}). A finer channel would put this setting back into the estimation regime. \textbf{(4)}~``Therefore memorization'' is an inference, not a theorem: the bound does not close a method adapting on the reward stream (\S\ref{sec:related}). \textbf{(5)}~Prop.~\ref{prop:regime} proves the three conditions separately and claims no joint sufficiency. Membership is necessary and not sufficient. \textbf{(6)}~Thm.~\ref{thm:conjunction}(a)'s necessity presumes \textbf{event opacity}, that no later consequence reveals the event. Where a blind arm reads it off its own falling output, as ours does, (a) bounds only the gain beyond that channel, worth $0.3$ years in the trade (Thm.~\ref{app:conjunction} in App.~\ref{app:proofs}). \textbf{(7)}~(b)'s equivalence is asymptotic in $R$, and an ablation that deletes a memory module tests (b) only where nothing else in the observation carries the event forward. That holds in the fleet. In the knee the observation is afferent-only, so the tissue readings go with the trace, but a five-year rate and a five-year running maximum of the felt signal remain: persistence with a horizon, which is neither a clean removal nor a clean instance. \emph{This is a statement about what an ablation measures}, not about whether a monotone substrate makes retention unnecessary. The pre-registered test of the latter is in \S\ref{sec:boundary} and went the other way. The configuration that would make the knee a clean test is named in Thm.~\ref{app:conjunction} in App.~\ref{app:proofs} and is now wired. The run returned and moved the endpoint against the theory (App.~\ref{app:wear}). \textbf{(8)}~(c) is close to definitional, since $L_e^\Phi$ is what $\Phi$ charges for an event. The content is the case analysis in Thm.~\ref{app:conjunction} in App.~\ref{app:proofs}, which covers the realizations whose effective event cost is fixed within a life and excludes the one, $\mathrm{CVaR}_\beta$, whose tail set moves with the optimizer. Finally, Thm.~\ref{thm:conjunction} is tested in one world: the trade is outside its scope rather than a weaker test of it, since that world fails (iii) for one body, and what carries the headline years there is Lemma~\ref{app:escape} with Cor.~\ref{cor:loadgate}. Pessimism on untried costs does not foreclose that substitution: over the anchored range of $\hat{c}^{0}$ there is no foreclosure anywhere (App.~\ref{app:prior}), because the trade always keeps one profile the pessimist still believes he can perform. What the band does show is that the endowment substitutes for retention and not for the signal. \textbf{(9)}~One limit is not about the theory but about the size of the effect it describes: the wear rate is anchored on people who spare their joints \citep{maschek2014cartilage,roth2017cartilage} and applied to a body that never spares, so the anchor already holds part of the answer, in the direction that shrinks what we report. \textbf{(10)}~The sharpest bound is not a modelling choice but a question about the world: \emph{whether the reference body is allowed to protect itself}. The field anchor enters as a lower bound on laying work, because that is what the trade supplies. Imposing it as an upper bound as well, a work rule capping exposure for every body, reverses the headline. Over $2{,}000$ paired bodies per variant:

\begin{center}\footnotesize
\begin{tabular}{@{}lrrr@{}}
\toprule
what bounds the kneeling share & $V_{\rm ind}$ years & share losing & kneeling share reached \\
\midrule
supply only (lower bound, canon) & $\mathbf{+4.38}$ & $0.0\%$ & $0.47$ \\
supply only, Ditchen screed share & $+3.18$ & $0.0\%$ & $0.45$ \\
supply and a work rule (upper bound too) & $\mathbf{-1.02}$ & $\mathbf{43.6\%}$ & $0.43$ \\
\bottomrule
\end{tabular}
\end{center}

\noindent It is the first configuration in this paper in which any body loses, and we report it without adopting it, for a reason older than the measurement: the measured $41\%$ is what floor layers \emph{do}, hence the outcome of the mechanism under test, so laying it on the arm that cannot feel hands that arm the fruit of that mechanism, which is the circularity Def.~\ref{def:epoch} already excludes for a sparing schedule learned across lives. The line we draw is between what the world \emph{offers} (a lower bound, admissible) and what the worker \emph{chooses} (an upper bound, not admissible as a world constraint). The measurement is therefore not a caveat but the precise statement of where the result lives: \emph{where a work rule caps exposure, a body model is not needed} for what we measure here. Weighting each activity by its knee-loaded time share \citep{ditchen2015cuela,jensen2009exposure} gives $61.8\%$ of the day for the body that senses nothing, $49.3\%$ for the one that feels without retaining and $49.6\%$ with the apparatus, against $41.0\%$ (SD $7.5$) in the field. Source \texttt{results/kniezeit\_band.json}, pre-registration \texttt{experiments/decl/kniezeit\_band.py}.

\section{Stylized model and proofs}
\label{app:proofs}
\textbf{Model $\mathcal{M}_0$.} A life has $H$ steps. Actions $a \in \{\mathrm{heavy}, \mathrm{light}\}$ yield per-step value $y_H$ and $y_L = y_H - c$ with substitution premium $c \ge 0$. Under heavy, an event occurs w.p.\ $p$ per step before the first event and $p^{+} = \nu p$ ($\nu > 1$, the cascade) after it, while light carries only the baseline hazard $p$. An event costs $L_e$ (damage, incident cost). In existential variants a second event ends the value stream. All quantities are bounded, and $pH \ll 1$.

\begin{proposition}[Concentration over bodies, Rem.~\ref{thm:insurance} of the main text]
\label{app:insurance}
Let $\rho$ and $\rho_\emptyset$ be two responses, let $B$ be the set of bodies on which $\rho_\emptyset$ loses the work before the horizon, and $q := \Pr[\theta \in B]$. If the two agree to within $\varepsilon$ on $\theta \notin B$, then $\mathbb{E}_\theta[V_{\rm ind}] \le q\,\Delta_b + (1-q)\,\varepsilon$ with $\Delta_b := \sup_{\theta \in B} V_{\rm ind}(\rho;\theta)$, while the conditional value is the at-risk difference by definition.
\end{proposition}
\begin{proof}
Couple the two responses on the same body and the same draws and split the expectation over $\ind{\theta \in B}$: off $B$ the integrand is at most $\varepsilon$, on $B$ at most $\Delta_b$.
\end{proof}

\noindent \emph{Why this form rather than the rare-event one.} The premise is a property of the apparatus and not of the world: a constraint that binds only while a body is failing cannot move a body that never fails. In the trade it holds exactly, and the rare-event version, with $B$ the lives carrying an acute event, is the special case the fleet runs, where it does not.

\begin{corollary}[Reading a mean over bodies, Cor.~\ref{cor:power} of the main text]
\label{app:power}
Under the premise of Prop.~\ref{app:insurance}, $\mathbb{E}_\theta[V_{\rm ind}]$ and $\mathbb{E}_\theta[V_{\rm ind}\mid B]$ differ by the factor $q$, so a population mean answers what the apparatus is worth to a cohort and not what it is worth to a body that needs it. Pairing is the natural design here, because the dominant nuisance is the between-body spread $\sigma_\theta$ and running the same body twice removes exactly that.
\end{corollary}

\paragraph{The constants, measured.} Floor layer, $\Delta\theta=18^\circ$, $2{,}000$ paired bodies per point, full stack against $\rho_\emptyset$ (\texttt{results/bodenleger\_lp\_b042}):

\begin{center}\footnotesize
\begin{tabular}{@{}lrrrrr@{}}
\toprule
specialization & $q$ & $\mathbb{E}[V_{\rm ind}]$ & $\mathbb{E}[V_{\rm ind}\mid B]$ & $\varepsilon_{\rm yr}$ & $\varepsilon_{\rm out}$ \\
\midrule
$0.58$ \emph{(observed mix)} & $0.693$ & $+4.38$ & $+6.33$ & $0.0000$ & $-0.40$ \\
$0.70$ & $0.693$ & $+4.86$ & $+7.02$ & $0.0000$ & $-0.14$ \\
$0.80$ & $0.693$ & $+5.32$ & $+7.68$ & $0.0000$ & $+0.02$ \\
$0.90$ & $0.693$ & $+5.71$ & $+8.24$ & $0.0000$ & $+0.11$ \\
$1.00$ & $0.693$ & $+6.03$ & $+8.71$ & $0.0000$ & $+0.18$ \\
\bottomrule
\end{tabular}
\end{center}

\noindent The two $\varepsilon$ columns must be read differently. A body outside $B$ works to the horizon $65$ without the apparatus, and no arm may exceed it, so $\varepsilon_{\rm yr} \le 0$ is an identity and only $\varepsilon_{\rm yr} = 0$ rather than $\varepsilon_{\rm yr} < 0$ is a measurement: the apparatus never ends a career that would have run to the horizon. Output is not capped, and there every body outside $B$ moves. It moves little and it changes sign along the axis, from $-0.40$ at the observed mix to $+0.18$ where specialization is free, because a response that declines the heavy work costs a body that would have survived it and pays one that would have specialized into it. The coupling premise therefore holds exactly in the career coordinate and within $0.4$ in output, which is $9\%$ of $\mathbb{E}[V_{\rm ind}]$ at the observed mix, and Rem.~\ref{thm:insurance} is to be read with that slack. The split into $q\,\mathbb{E}[V_{\rm ind}\mid B]$ follows by the tower property and is not an independent check. \emph{The effect is large, not rare}: a two-sample test on career length needs $28$ bodies unpaired and $6$ paired ($\sigma_\theta = 9.7$, $\sigma_d = 6.4$ years), so Cor.~\ref{cor:power} is a statement about what a mean \emph{means} rather than about whether it can be seen. Ranked by the career they would have had without the apparatus, the worst third runs $43.2 \to 49.6$, a gain of $+6.4$ years against the population's $+4.38$.

\begin{proposition}[Identification requires wear, Prop.~\ref{prop:estimation} of the main text]
\label{app:estimation}
Under Eq.~\ref{eq:hazard}, fix an action sequence $a_{0:t}$ and two wear resistances $\kappa_1 \neq \kappa_2$. Then
\[
M_t(\kappa_1) - M_t(\kappa_2) = \Big(\tfrac{1}{\kappa_2}-\tfrac{1}{\kappa_1}\Big)\sum_{u<t} d(a_u, M_u),
\]
to first order in the state difference, so the log-likelihood ratio of the two hypotheses given the felt signal $s_{0:t}$ depends on $\kappa$ only through that accumulated sum. A test at level $\alpha$ with power $1-\beta$ therefore requires $\sum_{u<t} d(a_u,M_u) \gtrsim (z_\alpha+z_\beta)\,\sigma_s\,/\,|\kappa_1^{-1}-\kappa_2^{-1}|$, and $M$ is non-increasing, so the wear that buys the distinction is spent.
\end{proposition}
\begin{proof}
Iterate $M_{t+1} = M_t - d(a_t,M_t)/\kappa$ from a common $M_0$ and subtract, where $d$ is evaluated on paths that coincide to first order. The felt signal is a noisy monotone reading of the state (Eq.~\ref{eq:channel}), so a two-sample separation of the induced means needs the displayed gap to exceed the reading noise by the usual factor.
\end{proof}

\noindent \emph{The acute extension, by a two-point argument.} Let two hazard hypotheses $\theta_1,\theta_2$ induce per-step event probabilities $p_1,p_2 \le \bar p$ with $\bar p H \le 1$. Any within-life test $\psi$ satisfies $\Pr_{\theta_1}[\psi{=}2]+\Pr_{\theta_2}[\psi{=}1] \ge 1-\mathrm{TV} \ge 1-H|p_1-p_2|$, bounded away from $0$ whenever the expected number of discriminating events $H|p_1-p_2| < 1$: the total variation between the two product-Bernoulli path laws is at most $\sum_t |p_1-p_2|$, and Le Cam's two-point bound applies. The one-shot write requires no test, since it conditions behaviour on the realized event rather than on an estimate.

\noindent\emph{Any policy, not only a test.} A policy in $\Pi(\Fcal^C)$ conditions on $\theta$ only through the same path observations, so the value it gains from distinguishing $\theta_1,\theta_2$ obeys the same bound: $\sup_{\pi\in\Pi(\Fcal^C)}[J_\theta(\pi)-J_\theta(\pi_0)]\le L\,\mathrm{TV}=O(p)$ on the safe (probing-free) path, for a learned $\pi$ as for a static one. Only the filtration enters. The optimizer (PPO, CEM, or exact DP) is irrelevant, which is why the trained learners of Sec.~\ref{sec:experiments} fall under the same statement.

\noindent\emph{Measured identifiability of the chronic channel (knee).} The proposition is silent on the chronic channel, which could carry $\kappa$ through the sensitization $S(D_t)$, so we measured it: a cross-validated regression of $\kappa$ on the complete observation vector over $3{,}000$ paired bodies returns $R^2 \le 0.04$ at every age between $20$ and $40$. A reader with the tissue state itself in its observation reaches $0.87$ to $0.97$, and the felt signal alone, before gain and reading noise, $0.63$ at age $20$ rising to $0.75$ by $40$. The anchored gain and noise remove that channel, which is why the observation we ship carries neither.

\begin{proposition}[Regime conditions, Prop.~\ref{prop:regime} of the main text, provable directions]
\label{app:regime}
In $\mathcal{M}_0$ (with $p$ and $p^{+} = vp$ the per-step event probability before and after the first event, $L_e$ the cost of one event, $c$ the substitution premium per step and $R$ the steps remaining after the event): \textbf{(i)} if the event time is a deterministic function of cumulative load (chronic hazard), the optimal policy is open-loop in time and information channels have zero value (Howard). \textbf{(ii)} post-event protection (switch to light) changes $J$ by $R \cdot \big[(p^{+}-p) L_e - c\big]$ over the remaining $R$ steps: protection pays if and only if $c < c^\ast = (p^{+}-p)L_e$, a threshold in the substitution premium. \textbf{(ii$'$)} load form. Let the post-event action be a load $\ell \in [0,1]$ (the share of heavy work in a continuous portfolio), with per-step value $y(\ell)$ increasing and concave, $y(1)=y_H$, and load-gated hazard $p^{+}(\ell) = p + (p^{+}-p)\ell$. The post-event per-step objective $y(\ell) - p^{+}(\ell)L_e$ is concave, so the optimal load is $\ell^\ast = 1$ if $y'(1) \ge (p^{+}-p)L_e$, $\ell^\ast = 0$ if $y'(0) \le (p^{+}-p)L_e$, and otherwise the interior solution of $y'(\ell^\ast) = (p^{+}-p)L_e$: a unit of load is removed exactly while its marginal substitution premium $y'(\ell)$ lies below the threshold $(p^{+}-p)L_e$ per removed load unit. The binary case is $y(\ell) = y_H - c(1-\ell)$, whose constant premium $c$ makes the solution a corner and recovers (ii). In the knee world the load is $\ell(a) = \sum_i \mathrm{share}_i\,\mathrm{effort}_i\,h_i$ and the learned arms move inside the portfolio rather than between two corners. \textbf{(iii)} in the additive variant with a finite bonus schedule collectible by step $T_m < H$ under the aggressive policy, behavior after $T_m$ affects only the rate term, and (ii) applies with $L_e$ excluding all bonuses: ``dying rich'' dominates whenever $c \ge (p^{+}-p)L_e^{\mathrm{rate}}$. With an unbounded bonus stream at rate $b$, protection additionally earns $b \cdot \mathbb{E}[\text{survival gained}]$, which restores the trade-off, but only if the surviving policy still progresses (that is, $b$ accrues under light). This is the empirical failure mode of Sec.~\ref{sec:boundary}.
\end{proposition}
\noindent \emph{Reading the three conditions.} If (i) fails, the prior already is the individual's fate and a blind schedule suffices, in the prior though not for a body that has none. Predictable frequency does not violate (i), since the prior gives the ensemble rate and never which draw this body is. If (ii) fails, event savings of order $p\,\Delta$ cannot cover protection costs of order $\csub$. If (iii) fails, dying rich weakly dominates surviving poor. We classify (world, institution, perspective) rather than worlds, since the same knee meets (iii) for a cohort and fails it for one body.
\begin{proof}
(i) With a deterministic event schedule the MDP is deterministic given the open-loop action sequence, an optimal open-loop sequence exists, and any information refinement leaves the optimum unchanged. (ii), and (iii) are the displayed arithmetic. For (iii) note bonuses collected before $T_m$ are policy-invariant across the compared continuations. (ii$'$) The post-event problem is stationary, so the optimal load is constant and maximizes the concave per-step objective, and the first-order condition and the two corner cases follow.
\end{proof}

\begin{proposition}[Calibration, the over-protection pole]
\label{prop:calibration}
Within the paying regime, welfare as a function of the felt-pain gain passes from under- to over-protection as the gain grows. In $\mathcal{M}_0$, whose protective decision is a corner, the set of gains matching the welfare optimum is an interval. Under the load form of Prop.~\ref{app:regime}(ii$'$), which is what a portfolio-continuous world runs, it is a single point.
\end{proposition}
\noindent Proved in App.~\ref{app:calibration}. The consequence is that a gain imported from a different economy misses the point even with every channel intact, so the gain must be hit rather than bracketed wherever the agent moves inside a portfolio. Only the direction of the failure, under-protection below and over-protection above, is common to both forms.

\begin{proposition}[Calibration, proof of Prop.~\ref{prop:calibration}]
\label{app:calibration}
In $\mathcal{M}_0$ (event probabilities $p$ before and $p^{+}$ after the first event, event cost $L_e$, substitution premium $c$), let the trained objective be $J - g \cdot (\text{felt pain})$ with gain $g \ge 0$, while welfare is $J - g^\ast(\text{felt pain})$ for the true $g^\ast$. In $\mathcal{M}_0$ the optimal policy is a threshold rule in $g$. Behavior matches the welfare optimum on an interval $[\underline{g}, \bar g] \ni g^\ast$ and departs monotonically outside it: under-protection for $g < \underline{g}$, and over-protection (forfeiting $c$-value without compensating $L_e$-savings) for $g > \bar g$.
\end{proposition}
\begin{proof}
$J$ is affine in the indicator of protecting, so the comparison $(p^{+}-p)(L_e + g \cdot \text{pain relief}) \gtrless c$ flips at a unique $g$-threshold, and monotonicity in $g$ is immediate.
\end{proof}

\noindent \emph{The interval is the binary action's, not the model's.} What makes the matching set an interval is that the comparison flips at a unique $g$ while the decision it flips is a corner, so every gain on one side of the threshold buys the same corner. Under the load form of Prop.~\ref{app:regime}(ii$'$) with $y$ strictly concave, the post-event optimum solves $y'(\ell^\ast) = (p^{+}-p)\,(L_e + g\cdot\text{felt relief})$, whose right-hand side is strictly increasing in $g$, so exactly one gain reproduces the welfare-optimal load and the matching set is a singleton. Reading Prop.~\ref{prop:calibration} as a robustness statement is therefore safe only where the protective decision is a corner. Our deployed worlds are portfolio-continuous and their learned arms move inside the portfolio, so there the gain must be hit rather than bracketed. What survives in both cases is the direction of the failure on either side.

\begin{theorem}[Conjunction, Thm.~\ref{thm:conjunction} of the main text]
\label{app:conjunction}
In $\mathcal{M}_0$ let $\Phi$ be a law-invariant objective of the per-life return and let $L_e^\Phi$ denote the $\Phi$-effective event cost: the amount by which, in the post-event comparison of Prop.~\ref{app:regime}(ii), $\Phi$ charges one event. For mean return $L_e^\Phi = L_e$. For the intrinsic cost $w\,s$, $L_e^\Phi = L_e + g\cdot(\text{felt relief})$ (Prop.~\ref{app:calibration}), and for an existential variant, $L_e^\Phi = L_e + b\cdot\mathbb{E}[\text{survival gained}]$. Assume Prop.~\ref{prop:regime}(i), and (iii) hold and $c \ge (p^{+}-p)L_e$, so that mean return does not pay protection per body. Then the $\Phi$-optimal policy in $\Pi(\Fcal^C)$ is protective if and only if (for (b), asymptotically in $R$) \textbf{(a)} $N_t \in \Fcal^C_t$, \textbf{(b)} $\ind{\sum_{u \le t} N_u \ge 1} \in \Fcal^C_{t'}$ for all $t' \ge t$ (without which the optimum captures at most a $1/R$ fraction of the gain achievable under (b)), and \textbf{(c)} $c < (p^{+}-p)\,L_e^\Phi$.
\end{theorem}
\begin{proof}
Sufficiency. Under (a), and (b) the switch policy ``heavy until the first event, light afterwards'' is $\Fcal^C$-measurable. Its post-event continuation changes the $\Phi$-objective by $R\,[(p^{+}-p)L_e^\Phi - c]$ over the remaining $R$ steps by definition of $L_e^\Phi$ and Prop.~\ref{app:regime}(ii), which is positive under (c). Since pre-event behavior coincides with $\pi_0$, the switch policy dominates $\pi_0$ and the $\Phi$-optimum is protective. Necessity of (a). Under event opacity, $N_t \notin \Fcal^C_t$ and $N_t \notin \Fcal^C_{t'}$ for every $t' > t$: no later observable consequence of the event reveals it. In $\mathcal{M}_0$ this holds because the only state an event moves is the hazard, which is not observed. It is a hypothesis on the filtration and is independent of Prop.~\ref{app:estimation}, which bounds inference about $\theta$ rather than detection of a realized event. Under opacity, the conditional law of the action process given the event equals its law given no event, so every $\Fcal^C$-measurable policy has the same post-event continuation as on an uneventful life. The protective continuation is not available and the optimum equals $\pi_0$ up to $O(p)$. Necessity of (b). If the event enters the filtration at $t$ but not at any $t' > t$, then for $t' > t$ the conditional law of the action process given the event equals its law given no event, so every $\Fcal^C$-measurable policy continues as on an uneventful life from $t+1$ on. It may still deviate at step $t$ itself, which is worth at most one step of the post-event difference, $[(p^{+}-p)L_e^\Phi - c]$, against $R$ times that under (b): without persistence a policy captures at most a $1/R$ fraction of the gain available with it, and the fraction vanishes as the remaining horizon grows. This is the sense in which (b) is necessary. A single-step twitch is not protection. Necessity of (c). Under (a), and (b) the post-event problem is the stationary comparison of Prop.~\ref{app:regime}(ii) with $L_e$ replaced by $L_e^\Phi$. If $c \ge (p^{+}-p)L_e^\Phi$ the heavy continuation is $\Phi$-optimal and the optimum is not protective. Mean return is the case $L_e^\Phi = L_e$, excluded by assumption. The two admissible realizations of (c) each raise $L_e^\Phi$ above $L_e$ and both compute the raise within one life ($w\,s$, death).
\end{proof}

\noindent CVaR is the realization whose sign the world decides: $L_e^\Phi$ presumes $\Phi$ decomposes over a fixed post-event comparison, which mean return, $w\,s$ and death do and $\mathrm{CVaR}_\beta$ does not, its tail set being chosen by the policy under optimization \citep{boda2006cvar,baeuerle2011cvar}. It is the only candidate whose effective event cost moves with the optimizer, and the only one that reverses across our worlds.

\paragraph{What event opacity is worth, measured rather than assumed.} Opacity is the one premise of the necessity argument that our own worlds violate, and the violation is the interesting case rather than a defect: an agent that observes its own falling output holds a coarse posterior on the event through that consequence even with no afferent channel. Our $\rho_\emptyset$ in the trade is exactly such an arm. It reallocates on realized output and plans a kneeling share of $0.74$ and realizes $0.70$ over its last working years (Table~\ref{tab:substitution}). Where opacity fails, necessity of (a) no longer bounds the protective gain but only the gain \emph{beyond} what the consequence channel already carries, and that residual is the difference between our two baselines: $+4.4$ years against a body reallocating on its own output, $+4.6$ against one reallocating on nothing (App.~\ref{app:wear}). \emph{So opacity is worth $0.3$ years in this world}, small enough that (a) is close to exact here, and we report the conservative baseline for that reason. In a world where consequences are loud the same argument would weaken proportionally, and the quantity to report is the consequence-channel baseline rather than the naive one.

\paragraph{When deleting the memory tests (b), and when it only changes the realization.} Condition (b) is a property of the filtration, so an ablation removes it only if nothing else in the observation carries $\ind{\sum_{u\le t}N_u \ge 1}$ forward. That holds in the fleet, whose observation is the clock, the previous shares and efforts, the felt signal, the trace and $w$, so that without the trace the only state outliving an event is a one-step action lag. The fleet's memory ablation is therefore a test of (b), and its collapse of protection is evidence for it. It holds only partly in the knee, whose observation keeps a five-year rate and a five-year running maximum of the felt signal when the trace is removed, which is persistence with a horizon rather than its absence. The clean removal, zeroing both summaries with the trace, was run over $18$ seeds at the deterministic optimum and leaves the learner better rather than worse, $46.90$ years against $46.28$ for the full stack and $45.74$ for the arm that loses only the trace, the sharpest contrast running $+1.16$ years at $p = 7.5\times10^{-3}$. All three pre-registered gates fail in the direction opposite to the theory, and the most likely reason is a confound of ours, since the ablation removes two observation dimensions and a smaller input is easier to learn from. The fleet therefore remains the clean evidence on (b) (\texttt{experiments/decl/knee\_persistenz\_rein.py}).

\paragraph{Why $L_e^\Phi$ is a reduction and not a definition, and where the reduction fails.} The device invites the objection that $L_e^\Phi$ is \emph{defined} as what $\Phi$ charges for an event, so that (c) reads ``protection pays if protection pays''. The definition has content only for objectives that decompose over the fixed post-event comparison of Prop.~\ref{app:regime}(ii), whose value can be written as the mean-return value with one number substituted for the event cost. For that class the reduction is a theorem and $L_e^\Phi$ is computable before any policy is trained, $L_e$ for mean return, $L_e + g\cdot(\text{felt relief})$ for the intrinsic cost and $L_e + b\cdot\mathbb{E}[\text{survival gained}]$ for the existential variant, each a within-life quantity. $\mathrm{CVaR}_\beta$ lies outside that class, because its tail set is determined by the policy being optimized and the optimizer moves it \citep{boda2006cvar,baeuerle2011cvar}, so no single substituted event cost represents it. Its sign is nevertheless protective in both worlds measured. In the acute knee at the deterministic optimum ($18$ seeds $\times\ 5\times10^{4}$ lives, \texttt{results/knee\_tail\_v4\_hpc}) follow-on tears per injured life run $0.351$ under the mean-return stack against $0.326$ under CVaR ($p = 1.4\times10^{-5}$) and $0.328$ under CVaR with the memory removed ($p = 5.3\times10^{-4}$), while the career ends at $60.98$ against $60.93$ years and output falls from $80.5$ to $74.9$, so tail-weighting there buys fewer tears and no working life. In the fleet with the damage weight removed, tail-weighted training at $\alpha{=}0.1$ against the mean-return learner at $\alpha{=}1$ lowers the write-off rate from $0.241$ to $0.095$ ($12$ seeds each, $p = 2.3\times10^{-3}$, Cliff's $\delta = 0.72$) and lengthens the certified service life from $90.3$ to $95.8$ months while raising output from $76.1$ to $83.8$ (\texttt{results/hh\_cvar\_w0\_hpc}). A tail-weighted objective can therefore stand in for the damage weight, and what it charges for an event is the part no substitution fixes, so (c) admits three realizations, two of them stable within a life and one not.

\textbf{Model $\mathcal{M}_1$ (chronic).} A life has $H$ steps and a body has integrity $M_0 = 1$. Work profiles $j \in \{1,\dots,K\}$ carry a known yield $y_j > 0$ and a per-step wear $m_j = d_j/\kappa(\theta) > 0$ that the body does not know, with a prior $\mu$ over $m \in \mathbb{R}^K_{>0}$. Integrity follows Eq.~\ref{eq:hazard}, $M_{t+1} = M_t - m_{j_t}$, the trade is lost at $T = \min\{t : M_t < \underline{M}\} \wedge H$, and the two reported quantities are the career length $T$ and the career output $\sum_{t<T} y_{j_t}\phi(M_t)$ with $\phi$ nondecreasing and $\phi(1)=1$. The felt signal is \emph{action-gated}: $s_t$ is a noisy reading of $m_{j_t}$ and says nothing about $m_k$ for $k \ne j_t$. Write $\Lambda := 1-\underline{M}$ for the wear budget. There is no event: $\mathcal{M}_1$ is $\mathcal{M}_0$ with the hazard replaced by accumulation, and it is the model the trade and the grid of App.~\ref{app:wear} run.

\begin{proposition}[The conjunction in the chronic regime, Prop.~\ref{prop:chronic} of the main text]
\label{app:chronic}
In $\mathcal{M}_1$ with $\Lambda/\min_j m_j \le H$, parts (a) and (b) stated for $\phi \equiv 1$ and part (c) for $\phi$ nondecreasing and non-constant: \textbf{(a)}~under state opacity the value of sensing is $\mathrm{EVPI} - \Pi_{\rm probe}$, where, over plans (policies whose filtration does not depend on $m$),
\[
\mathrm{EVPI} := \mathbb{E}_\mu\big[\max_{\pi} \Phi(\pi,m)\big] - \max_{\pi}\,\mathbb{E}_\mu\big[\Phi(\pi,m)\big] \ge 0, \quad \mathbb{E}_\mu\big[\max_{\pi}\Phi\big] = \Lambda\,\mathbb{E}_\mu\big[\max_j y_j/m_j\big]
\]
up to one step, with equality in the first display if and only if $\arg\max_j y_j/m_j$ is $\mu$-a.s.\ constant, and $\Pi_{\rm probe} \ge \sum_{j \in \mathcal{P}} y_{j^\star} m_j / m_{j^\star}$ the output forgone by working each probed profile $\mathcal{P}$ at least once (Prop.~\ref{app:estimation}: the evidence is the wear). Restricted to stationary plans the subtracted term is $\Lambda \max_{\beta \in \Delta^{K-1}} \mathbb{E}_\mu[\sum_j \beta_j y_j / \sum_j \beta_j m_j]$. \textbf{(b)}~the readings are action-gated, so the map $j \mapsto m_j$ enters the filtration only by being retained. The state $M_t$ carries the scalar $\sum_{u<t} m_{j_u}$ already spent and no comparison across profiles. A policy that does not retain readings therefore pays $\Pi_{\rm probe}$ and cannot collect $\mathrm{EVPI}$. \textbf{(c)}~is not a condition on $\Phi$. Wear enters the yield of every profile through $\phi$ and the budget through $\Lambda$, so the price of damage is in the dynamics, and an ablation that removes it removes the body.
\end{proposition}
\begin{proof}
(a) For any policy let $n_j$ be the steps spent on profile $j$. Feasibility is $\sum_j n_j m_j \le \Lambda + \max_j m_j$ and the output is $\sum_j n_j y_j = \sum_j n_j m_j (y_j/m_j) \le \max_j (y_j/m_j)\,(\Lambda + \max_j m_j)$, so every policy earns at most $\Lambda \max_j (y_j/m_j)$ up to one step and the constant profile $\arg\max_j y_j/m_j$ attains it: that is the pathwise optimum, and its mean is the first displayed term. The uninformed value is $\max_\pi \mathbb{E}_\mu[\Phi]$ over plans, which is at most $\mathbb{E}_\mu[\max_\pi \Phi]$, with equality exactly when one plan attains the inner maximum $\mu$-a.s. Since the pathwise optimum is the constant profile $\arg\max_j y_j/m_j$, that happens exactly when this argmax does not vary over the prior. For stationary plans the expected output is $\Lambda\,\mathbb{E}_\mu[\sum_j \beta_j y_j / \sum_j \beta_j m_j]$, linear-fractional in $\beta$ and hence maximized at a vertex, which is the computable form. The subtracted term is Prop.~\ref{app:estimation} applied to this model: a reading of $m_j$ exists only on a step spent working $j$, and that step consumes $m_j$ of the budget, worth $y_{j^\star} m_j / m_{j^\star}$ of output at the optimum. (b) Immediate from the two facts about the filtration. The second is Eq.~\ref{eq:hazard} summed, $M_t = 1 - \sum_{u<t} m_{j_u}$: two wear-resistance vectors agreeing on the profiles worked so far induce the same state, so the state cannot rank profiles that have not been worked, and the readings that could are gone. (c) Here $\phi$ is nondecreasing and non-constant, which (a) and (b) did not need and (c) does: the per-step yield of every profile then falls as the body wears, whatever $\Phi$ is. Setting $\underline{M} = 0$ deletes the loss of the trade and leaves that. To delete the price one must set $\phi \equiv \text{const}$ and $\Lambda = \infty$, which is a body whose state does not affect its work.
\end{proof}

\noindent \emph{Consistency with Prop.~\ref{app:regime}(i).} That part says a channel is worth nothing when the end of the career is a deterministic function of cumulative load, and it is the degenerate case of (a) here: if $\mu$ is a point mass then $\arg\max_j y_j/m_j$ is constant, $\mathrm{EVPI} = 0$ and the optimal policy is open-loop. What makes the chronic channel valuable is therefore not the accumulation but the body's ignorance of its own rate, which is the quantity Rem.~\ref{thm:insurance} concentrates over and Prop.~\ref{app:estimation} says can only be bought with wear.

\noindent \emph{What changes against the acute case, and what the grid says.} Two of the three conditions survive and one does not. (a) survives with a different content: what sensing is worth is a between-body quantity, $\mathrm{EVPI}$, and it is \emph{net} of a cost the acute case does not have, since in $\mathcal{M}_0$ the event announces itself for free while in $\mathcal{M}_1$ every reading is bought with wear. This predicts a sign change in career output as the world gets harsher, and the grid shows it: feeling \emph{costs} output at $\Delta\theta = 8^\circ$ ($-1.28$ at the observed mix) and buys it at $27^\circ$ ($+3.88$ to $+5.30$), while buying years throughout (the $300$-body probe family of App.~\ref{app:wear}, not the $2{,}000$-body family of Table~\ref{tab:familie}). (b) survives, and what replaces the acute bound is a measurement rather than a second bound. The acute argument gives a $1/R$ fraction, because there the memoryless policy has nothing at all after the step of the event. Here the body itself carries the level, so a memoryless policy still reallocates on what it feels now, and what it forgoes is only the comparison across profiles. Measured over the grid, the gain of the forgetting arm stays between $0.52$ and $1.12$ years while the total runs from $1.41$ to $7.01$, so its share falls with severity, $40\%$ at $8^\circ$, $21\%$ at $18^\circ$ and $16\%$ at $27^\circ$ (medians over the five specialization ceilings). The sharper reading is the one (b) predicts directly: at the mildest severity the forgetting arm produces \emph{less} career output than the arm that never feels at all, in all five cells ($-0.98$ to $-0.21$), because it pays $\Pi_{\rm probe}$ and retains nothing to spend it on, and it turns positive only where the one-step reallocation is itself worth the wear ($+0.59$ to $+0.85$ at $18^\circ$ and $27^\circ$). \emph{These are post-hoc readings of a pre-registered contrast}: the declaration fixed the four arms and both endpoints before the run, and fixed the thresholds on career length, but not the decomposition into probing cost and collected information, which we read off afterwards. (c) does not survive as a condition. In $\mathcal{M}_1$ it is a property of the dynamics, so the fleet's instrument, deleting the depreciation from $\Phi$, has no counterpart here. Deleting the loss of the trade leaves what feeling buys growing in severity ($-0.70$ to $+5.07$), which is (c) declining to be ablated rather than (c) failing.

\paragraph{The schedule confound, without a body.} The claim that the gap between the two values tracks the population's heterogeneity and not the equipment needs no biophysics. A life has $H{=}100$ steps and a hidden crossing time $\theta$, uniform around $\mu{=}50$ with half-width $\omega$. Before $\theta$ heavy work is safe and pays $1$, after it heavy work damages at rate $1$ against a budget $D{=}10$, and light work pays $0.75$ and never damages. The channel reports only whether this body's own $\theta$ has passed. $\rho_\emptyset$ works heavy throughout, $\rho$ spends the budget and then goes light, $\pi_0$ switches at the population-best fixed time. Driving $\omega$ alone, with equipment, world and horizon fixed:

\begin{center}\footnotesize
\begin{tabular}{@{}rrrrr@{}}
\toprule
spread $\omega$ & $T^\ast$ & $V_{\rm ind}$ & $V(C\mid\emptyset)$ & $\chi$ \\
\midrule
$0$  & $60.0$ & $30.00$ & $\mathbf{0.00}$ & $100\%$ \\
$10$ & $50.0$ & $28.25$ & $0.81$ & $97\%$ \\
$20$ & $40.0$ & $28.21$ & $3.31$ & $88\%$ \\
$35$ & $25.0$ & $27.31$ & $6.24$ & $77\%$ \\
$50$ & $10.0$ & $27.67$ & $9.94$ & $64\%$ \\
\bottomrule
\end{tabular}
\end{center}

\noindent \emph{The equipment is worth the same $28$ to $30$ to every individual across the whole sweep. What moves by a factor of twelve is only what the ablation design reads, and at $w{=}0$ it reads exactly zero.} That invariance is the content. The ratio and $1/(1-\chi)$ are omitted here because they agree by construction, $\chi$ being defined from the same two measured quantities (\texttt{experiments/vergleichsart\_minimal.py}).

\begin{proposition}[Schedule confound, Prop.~\ref{prop:zerlegung} of the main text]
\label{app:zerlegung}
Let $\Sigma := \Pi(\emptyset)$ be the schedules, the epoch-one policies measurable with respect to the prior alone, and let $\sigma^\ast \in \arg\max_{\sigma\in\Sigma}\mathbb{E}_\theta J(\sigma;\theta)$. Put $\chi := \mathbb{E}_\theta[J(\sigma^\ast)-J(\rho_\emptyset)] \,/\, \mathbb{E}_\theta[V_{\rm ind}]$. Then
\begin{equation}
V(C \mid \emptyset) \;=\; (1-\chi)\; \mathbb{E}_\theta\!\left[V_{\rm ind}\right],
\qquad
\frac{\mathbb{E}_\theta[V_{\rm ind}]}{V(C \mid \emptyset)} \;=\; \frac{1}{1-\chi}.
\label{eq:confound}
\end{equation}
Moreover, for any ablated arm at least as good in the mean as $\rho_\emptyset$, $V(C\mid\emptyset) \le \mathbb{E}_\theta[V_{\rm ind}]$.
\end{proposition}

\noindent \emph{Proof.} $\rho_\emptyset \in \Sigma$ by Def.~\ref{def:epoch}, and the best policy reading nothing is by definition the best schedule, so $\pi^\ast_{\Pi(\emptyset)} = \sigma^\ast$. Telescoping, $\mathbb{E}[V_{\rm ind}] = \mathbb{E}[J(\pi^\ast_{\Pi(C)}) - J(\sigma^\ast)] + \mathbb{E}[J(\sigma^\ast) - J(\rho_\emptyset)] = V(C\mid\emptyset) + \chi\,\mathbb{E}[V_{\rm ind}]$. The inequality follows because $\mathbb{E}[J(\sigma^\ast)] \ge \mathbb{E}[J(\rho_\emptyset)]$. $\square$

\paragraph{The identity is classical, the observation is not.} Eq.~\ref{eq:confound} is the stochastic-programming decomposition of the value of adaptivity, with $\sigma^\ast$ the here-and-now solution, $\pi^\ast_{\Pi(C)}$ the recourse solution and the per-body oracle of \S\ref{sec:related} the wait-and-see solution \citep{madansky1960,birge2011stochastic}. The same two quantities are the adaptive and the static Bayes policy in Bayes-adaptive RL \citep{duff2002,ghavamzadeh2015brl}. \emph{Neither the algebra nor $\chi$ is what we claim.} What we claim is that the field's default ablation design reports the value of the stochastic solution and calls it the value of the component, and that in the trade the two differ by a factor of seven, because a schedule fitted across bodies captures more than the per-body gain ($\chi = 1.40$ at the anchored cell, \texttt{results/lp\_zeitplan\_chi}), so that the ablation design measures what the channel adds beyond the population's best plan, a different quantity and here a negative one. The classical reading also explains the fleet: the threat is drawn per unit, no blind schedule beats the uniform one, $\chi = 0$, and the two designs ask the same question.

\noindent Three things follow that are worth saying plainly. The two designs are not rival estimates of one quantity. They are separated by a term, and that term is what a population prior can carry. The ablation design is \emph{sound} exactly where $\chi$ is small, which is why we report the fleet in that form. And Eq.~\ref{eq:confound} is a check on any table reporting both: an entry with $V(C\mid\emptyset) > \mathbb{E}[V_{\rm ind}]$ cannot arise from a trained ablation against a competent ablated arm.

\paragraph{The two values, side by side, in each world.} Both designs applied to the same world, each in that world's own headline endpoint:

\begin{center}\footnotesize
\begin{tabular}{@{}lrrrrr@{}}
\toprule
world & $\mathbb{E}[V_{\rm ind}]$ & $\sigma^\ast$ over $\rho_\emptyset$ ($\chi$) & predicted $V(C\mid\emptyset)$ & measured $V(C\mid\emptyset)$ \\
\midrule
floor layer, wear & $+4.38$ y & $+6.37$ y ($1.40$) & $-1.77$ y & $+0.65$ y \\
fleet, latent defect & $+0.227$ $J$ & $-0.064$ $J$ ($0$) & $+0.227$ $J$ & $+0.301$ $J$ \\
care robot, anchored floor & $+271.9$ mo & $+116.1$ mo ($0.43$) & $+155.8$ mo & $+315.1$ mo \\
\bottomrule
\end{tabular}
\end{center}

\noindent $V_{\rm ind}$ is the paired individual probe against $\rho_\emptyset$, $V(C\mid\emptyset)$ is the trained arm against its ablation, and $\chi$ is the share of $V_{\rm ind}$ the best \emph{blind} schedule $\sigma^\ast$ already captures: in the trade the best of a fixed-endowment and a fixed-switching-age family, fitted on bodies not evaluated (App.~\ref{app:wear}, E55), on the care robot the fixed sparing plan, and in the fleet the throttled plan, which does not beat the uniform one, so $\chi$ is $0$ there. The fourth column is what Prop.~\ref{app:zerlegung} predicts from $\chi$ alone. The two designs are not the same policy measured twice: they are the two standard ways of asking, and they answer differently by a factor the identity fixes. Where a schedule is available to a population but not to a body, reporting the trained contrast as the value of the apparatus misstates it: in the trade the trained contrast is a seventh of the per-body value, and the quantity it estimates is negative. Where the threat is drawn per unit, the two agree.

All three measured values sit above the prediction, the trade by $2.4$ years, the fleet by a third and the care robot by a factor of two, and in each world the ablated trained arm does fall short of $\sigma^\ast$: $59.2$ against $61.6$ years, $J$ $0.49$ against $0.66$, $86$ against $168$ months. \emph{The direction identifies which arm is suboptimal.} Writing the identity out, $V(C\mid\emptyset)_{\rm meas} < (1-\chi)\,\mathbb{E}[V_{\rm ind}]$ holds exactly when the full-channel arm's shortfall from $\pi^\ast_{\Pi(C)}$ exceeds the ablated arm's shortfall from $\sigma^\ast$, and an ablated arm falling short of $\sigma^\ast$ would push the ratio \emph{below} the prediction. So part of each excess is optimizer shortfall and not schedule confound, which makes $V(C\mid\emptyset)$ a lower bound on what the channel adds and the excess no second measurement of $\chi$. \emph{Read as a diagnostic the identity also caught a defect}, a first version of the care-robot row that entered the absolute service life of the felt arm, $315.2$ months, where a contrast belongs, which the bound $V(C\mid\emptyset) \le \mathbb{E}[V_{\rm ind}] = 264.9$ rules out.

\begin{lemma}[Where a cost-constrained agent escapes to, Lemma~\ref{lem:escape} of the main text]
\label{app:escape}
Let the agent choose shares $a \in \Delta^{n}$ maximizing $\hat p^\top a$ subject to $\hat c^\top a \le \tau$ and to group availability caps, where $\hat p$ is its estimate of the per-option value and $\hat c$ its estimate of the felt cost. Suppose $\hat c_i = c_i$ for the options $i \in W$ it has worked within its retention window and $\hat c_i = 0$ otherwise. If the constraint binds at the incumbent allocation, the optimum moves mass to $\arg\max\{\hat p_i : i \notin W\}$, \emph{independently of the true costs $c_i$}. The realized reduction in true cost is then $c_{j} - c_{k}$, where $j$ is the abandoned option and $k$ the best-paid option outside $W$, and it is not $c_{j} - \min_i c_i$.
\end{lemma}
\begin{proof}
On $\{i \notin W\}$ the constraint is slack by construction, so within that set the program is the unconstrained maximization of $\hat p^\top a$, whose solution is a vertex at the largest $\hat p_i$. Options in $W$ carry their true cost and can only tighten the constraint. The displayed cost difference follows by evaluating $c^\top a$ at the two allocations.
\end{proof}

\noindent \emph{The premise is the whole content, and a second world tests it.} Lemma \ref{app:escape} assumes the estimate is zero on untried options. Whether that holds is a property of how the felt signal is \emph{represented}, not of the domain. In the trade the estimate is per day-type and exists only for what has been worked, so the premise holds. On the care robot the cost of an activity is $\hat c_i = \sum_j L_{ij}\alpha_j$, read off the \emph{motor-level} signal through a known load matrix, so every activity carries an estimate whether or not it has ever been run, and the premise fails. The two worlds then behave as the lemma says they must. Escaping, the floor layer goes to the best-paid admissible day type (carpet laying, load $75$ against kneeling's $90$) and needs retention to aim. The care robot goes straight to its gentlest activity (load $4.15$ against the transfer's $8.15$): its share of the day there rises from $0.05$ under the blind portfolio to $0.24$ with the signal alone, and the mean load falls from $7.79$ to $6.88$. Retention lowers the load further, to $5.80$ at a share of $0.44$, and that is exactly the over-protection its service months price: the trace costs $34.2$ of them. \emph{A signal that generalizes across options makes memory unnecessary for aiming}, which is also why retention there can only ratchet the level and costs the care robot $34.2$ months (\S\ref{sec:boundary}).

\begin{corollary}[Load-gating makes persistence existential, Lemma~\ref{lem:escape}(b) of the main text]
\label{cor:loadgate}
Suppose the felt signal is load-gated, so that $\hat c_i$ can be written only after an allocation that places weight on $i$, and the allocation for a period is chosen before that period's signal exists. If retention is zero, then $\hat c \equiv 0$ at every decision, the constraint $\hat c^\top a \le \tau$ is satisfied by every feasible $a$, and the $\Phi$-optimal policy in $\Pi(\Fcal^C)$ coincides with the unconstrained one. The fraction of the protective gain attainable without persistence is then $0$, not the $1/R$ of Thm.~\ref{app:conjunction}(b).
\end{corollary}
\begin{proof}
The $1/R$ bound of Thm.~\ref{app:conjunction}(b) is attained by a deviation at the event step itself, which requires the event to be in the filtration \emph{at} that step and the action for that step to be still choosable. Load-gating with decision-before-load removes the second condition: the signal generated by period $t$'s allocation is available only from $t+1$, and with zero retention it is erased before $t+1$'s decision. No period's choice is ever taken under a binding constraint, so the policy is the unconstrained one at every step.
\end{proof}

\noindent Measured in the trade over $1{,}000$ paired bodies: a body retaining nothing lays floors in every working year (kneeling share $1.000$) and the felt channel is worth $+0.00$ years and $+0.00$ output, with not one body gaining a year (\S\ref{sec:individual}). \emph{Persistence is therefore not an amplifier of the channel but a precondition for it}, wherever the signal is load-gated and the decision precedes the load.

\paragraph{What $\mathcal{M}_0$ predicts in the deployed system, and how far it is confirmed.} $\mathcal{M}_0$ predicts a \emph{location} and a \emph{size}, and the trained sweep confirms the first and not the second. \textbf{(1)}~The substitution threshold $c^\ast=(p^{+}-p)L_e$, computed from the world's constants before the sweep, predicts $\lambda^\ast{=}0.69$, and in the dense sweep (\texttt{results/knee\_csub\_dense\_F.json}, entropy $0$, six seeds per cell) the learned arms do switch corners inside $[0.65,0.75]$. The feeling arm's advantage there is a different matter: the rule registered before the run required the seed rank test to be significant at two of the three grid points $\{0.65,0.70,0.75\}$ and it is significant at one, so that prediction \textbf{failed its own test}. Pooled over seeds $F$ lies between $0.96$ and $1.16$ across the whole $\lambda$ range, its largest value $1.16$ at the band's upper edge and $1.09$ at its lower one, with $0.98$ between them. The size is what $\mathcal{M}_0$ implies once its constants are measured rather than assumed, since the event-cost term is fifty times smaller than the chronic pain relief and both arms see the same wage gradient; what separated the arms in this world was the exit institution, against our own registered prediction that it would not (Thm.~\ref{thm:conjunction}(iii)). \textbf{(2)}~The premise of Rem.~\ref{thm:insurance} holds in one coordinate: a body outside the at-risk set works to the horizon in both arms, so $\varepsilon=0$ in years is forced by the cap, while in output, which nothing caps, those bodies move by up to $0.40$ (Cor.~\ref{app:power} in App.~\ref{app:proofs}). The placebo decomposition is the Le~Cam bound seen from the other side.

\subsection{The memory update law}
\label{app:memlaw}
The three equations behind the description in \S\ref{sec:channels}, reused unchanged in every world of this paper:
\begin{align}
P_{t+1}(z) &= \max(P_t(z), i_t(z)) \;\; \text{(peak, one-shot write)}\label{eq:mem1},\\
U_{t+1}(z) &= (1-\lambda_{\mathrm{dec}})U_t(z)
 + \lambda_{\mathrm{rec}}(1-i_t(z))\ind{z\ \mathrm{exposed}}
 \;\; \text{(safety: extinction, recovery)}\label{eq:mem2}.
\end{align}
The policy acts on their difference, read through a distance kernel $K : \mathcal{Z}\times\mathcal{Z} \to [0,1]$ ($K(z,z)=1$) that generalizes to nearby situations, and clamped to a per-situation avoidance signal in $[0,1]$, the lever a substitution turns on:
\begin{align}
\rho_t(z) &= \big[\, w_P \max_{z'} K(z,z')P_t(z')
 - w_U \max_{z'} K(z,z')U_t(z') \,\big]_0^1
 \;\; \text{(read-out)}\label{eq:mem3}.
\end{align}
Avoidance generalizes from the peak, safety from uneventful exposure, and their balance at a site tips whether the agent moves away from it.

\noindent Here $i_t(z)$ is the felt intensity at situation $z$, $K(z,z')$ the distance kernel with $K(z,z)=1$, and $w_P, w_U$ the read-out weights. The peak write is one-shot and permanent, while the safety trace carries extinction ($\lambda_{\mathrm{dec}}$) and recovery ($\lambda_{\mathrm{rec}}$), so that avoidance re-emerges when a site stops being safely exposed rather than decaying away.

\section{The machine worlds: care robot, rover, latent-defect fleet}
\label{app:robots}

\paragraph{Training details, all trained arms.} Every trained arm is PPO in the average-reward form (a running mean $\bar r$ with rate $\eta = 0.01$ replaces discounting), $10^6$ environment steps per seed ($3\times10^6$ for the two robots, whose episodes are longer), rollouts of $2{,}048$ steps, $4$ epochs of minibatches of $64$, clip $0.2$, value coefficient $0.5$, Adam with learning rate $3\times10^{-4}$. The entropy coefficient is $0$ at the deterministic optimum and $0.05$ in the rows marked so (App.~Table~\ref{tab:verankert}). The floor layer's network is a shared trunk of two layers of $128$ units with a Dirichlet share head and a Normal effort head of $64$ units over the eight day types, and the fleet's is the equivariant head of App.~\ref{app:fleet} over its eight work classes, same optimizer and budget. One environment step is one year for the floor layer and one month for the machines, so $10^6$ steps are about $2\times10^4$ floor-layer lives and $10^4$ fleet lives. The arms are evaluated on fresh, paired bodies after training. Ablations zero inputs, never dimensions, so all arms of a world share one architecture.

\begin{figure}[t]\centering
\IfFileExists{fig_arme_welten.pdf}{\includegraphics[width=\textwidth]{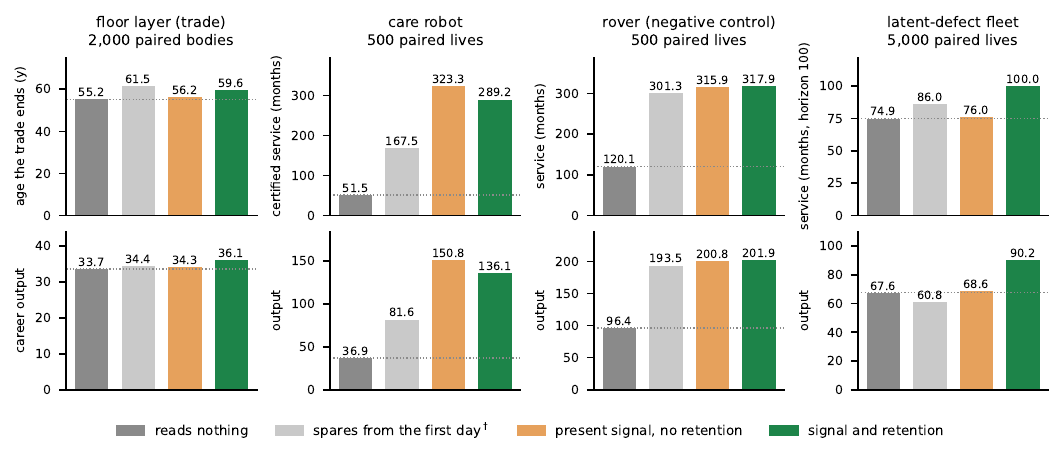}}{\fehlendeabbildung{fig\_arme\_welten.pdf}}
\caption{\textbf{The same stack in four worlds, both endpoints per world.} Bars are the arms of
Table~\ref{tab:leiter} on the same paired bodies, and the dotted line is the body that reads nothing. In
the trade and the care robot the retaining arm and the present-signal arm separate in opposite
directions, slot~(b) paying in one and costing in the other. In the rover blind caution reaches
almost the whole gain, and in the fleet the one-shot memory alone takes the write-off rate from $57\%$
to $0$ (App.~Table~\ref{tab:individuum}), while the present-signal arm stays within $1\%$ of the reference,
the one-step twitch of condition (b). Its trained arms are in Table~\ref{tab:leiter} and
App.~Table~\ref{tab:worlds}. Drawn from \texttt{results/fig\_arme\_welten.json},
which collects the files of Table~\ref{tab:leiter} and App.~\ref{app:robots}.}
\label{fig:arme}
\end{figure}

\noindent Trade and care robot disagree about slot~(b). In the trained arms of the care robot the trace is inert rather than harmful ($+144.5$ months, $p=0.16$ over $8$ seeds), and the felt body there dominates even blind caution on both numbers at once by the widest margin of any world: $+155.8$ months, $95\%$ CI $[+147.6,+164.0]$, and $+69.3$ output, $[+65.6,+72.9]$, paired over $500$ lives (\texttt{results/pflege\_familie/familie\_ci.json}). In the rover the channel's marginal value over blind caution collapses to $7\%$ without vanishing.

\begin{table}[t]\centering\footnotesize
\setlength{\tabcolsep}{6pt}
\begin{tabular}{@{}l r r c r@{}}
\toprule
& \multicolumn{2}{c}{\emph{wear takes the movement away}} & & no wear \\
\cmidrule(lr){2-3}\cmidrule(lr){5-5}
what the body carries & floor layer & care robot & & latent-defect agent \\
 & \footnotesize loses trade & \footnotesize decertified & & \footnotesize written off \\
\midrule
feels, minds, remembers & $\mathbf{1.8\%}$ & $\mathbf{78.7\%}$ & & $\mathbf{17.8\%}$ \\
\quad feels nothing & $8.8\%^{***}$ & $100\%^{**}$ & & $55.2\%^{***}$ \\
\quad feels, is not moved to act & $9.7\%^{***}$ & $89.6\%$\,\tiny n.s. & & $18.8\%$\,\tiny n.s. \\
\quad feels and acts, but forgets & $1.2\%$\,\tiny n.s. & $90.4\%$\,\tiny n.s. & & $53.5\%^{***}$ \\
\bottomrule
\end{tabular}
\caption{\textbf{The same three parts, for trained policies.} The share of lives that lose the paid work
before the horizon, for one chain run through a floor layer's cartilage, a care robot's joints and a
latent-defect fleet ($10^6$ environment steps, $3\times10^6$ for the robot). $18$ seeds per arm, $8$ for
the care robot, $^{***}p<0.001$ and $^{**}p<0.01$ against the intact body, n.s. marks a cell the
measurement cannot separate from it.}
\label{tab:worlds}
\end{table}

\paragraph{What the table says, and what it is drawn from.} It asks a different question than Table~\ref{tab:leiter}, namely what the apparatus adds to an agent that already knows the schedule. \textbf{Feeling and minding are each necessary in the floor layer}, costing up to five times the rate of losing the work, since a body that cannot sense the tissue substitutes too late, and one that senses it without being moved to act substitutes just as late. \textbf{Remembering does not separate here}, which is what Prop.~\ref{prop:chronic}(b) predicts for a world whose state carries the level. In the care robot only feeling separates at eight seeds ($100\%$ against $78.7\%$, $p = 0.005$), and minding and remembering point the same way without reaching significance. The floor layer's column is drawn from \texttt{results/knee\_fdecke2\_hpc} (its per-seed aggregate, \texttt{results/tab\_leiter\_spalten.json}, is the file in the archive), the latent-defect column from \texttt{results/hh\_verankert2\_hpc}, both at entropy $0$, and the care-robot column from \texttt{results/pflege\_boden\_hpc}. A rover column is omitted, since under the $960$-month horizon every life of every arm loses the mission, so the share carries no information, and the rover's contrast is in months in Table~\ref{tab:leiter}.

\paragraph{The care robot's probe family.} Its four arms, reads nothing, blind caution, present signal and retaining, end at $51.5$, $167.5$, $323.3$ and $289.2$ certified months with $36.9$, $81.6$, $150.8$ and $136.1$ output. Against the insensate body the present-signal arm gains $+271.9$ months, $95\%$ CI $[+259.1,+284.7]$, and $+113.9$ output, $[+108.5,+119.4]$, and against blind caution $+155.8$ months and $+69.3$ output, paired over $500$ lives with every life gaining. The nociceptive memory costs on both readings there, wear being monotone, so that the present reading already holds the peak and an event-marking organ has nothing to mark, which is Thm.~\ref{thm:conjunction}(b)'s premise failing in a rolling bearing as it fails in creeping cartilage.

\paragraph{The rover in full.} Its threat is chronic and deterministic in load, so (i) fails and the blind arm reaches the same lever by throttling. Caution alone yields $181.2$ of the $195.8$ months the felt body gains, and reading the signal adds $+14.6$ months, $95\%$ CI $[+13.3,+15.9]$, and $+7.3$ output, $[+6.5,+8.1]$ ($500$ paired lives, \texttt{results/rover\_familie}). The reason is structural. Chronic wear there accumulates at a rate independent of accumulated wear, so the ordering of the actuators is fixed by the portfolio in the first month and never changes ($87\%$ of lives, rank correlation $0.975$), and every later reading repeats the first. What the signal cannot buy is information about which part to spare, which was known on day one, while how worn the body is rises monotonically and drives the whole gain. Its gain sweep gives the over-protection pole of Prop.~\ref{prop:calibration} a full curve, and three quarters of what the rover appears to pay is the gain imported from the knee. The run in Table~\ref{tab:leiter} is the design with an absorbing exit and the residual stream as an axis, and the structure shows there as a bimodal seed distribution, half of the blind seeds parking after five months while the full stack holds $98$ to $191$ months.

\subsection{The latent-defect fleet}
\label{app:fleet}

\noindent \emph{What the construction could carry.} The eight classes of work are near-equivalent by construction and only the event statistics are field data. If their yields were spread, the substitution premium of condition (ii) would rise with the spread and the conjunction is predicted to weaken with it. A fleet with unequal classes is the next test, not a variant we have run.

\noindent \emph{What $w{=}0$ does and does not test.} Setting the damage weight to zero exchanges one realization of (c) rather than removing it, and the overlap between the seed distributions returns at entropy $0.05$ (the separation above pools the $18$ seeds of both entropies, a one-sided Mann-Whitney over $36$ against $36$). The interchangeability half holds in both worlds measured. In the knee the substitution lowers follow-on tears without the memory, at a cost in output and with no working life gained, and in the fleet at $w{=}0$ it lowers the write-off rate from $0.241$ to $0.095$ and adds $5.5$ months ($p = 2.3\times10^{-3}$). Thm.~\ref{app:conjunction} in App.~\ref{app:proofs} states which realizations of (c) a substituted event cost represents. \paragraph{Slot (c), tested by removing its realization.} Thm.~\ref{thm:conjunction}(c) asks the objective to price the event above its frequency, and in every world we built the loss of the work already does that, which is why the $w{=}0$ ablation is inert here. Setting the write-off threshold to infinity removes that realization: incidents still occur, are still felt and still compound, but they cost nothing. Channel and memory are untouched. Sparing is then pointless by construction, so the endpoint has to be behaviour rather than welfare, and we measure the exposure to the injuring class over the ten steps after the first incident, against the exposure to the remaining classes over the same window. Since the shares sum to one, real avoidance must \emph{raise} the others, whereas a policy that merely does less of everything lowers both.

\begin{center}\footnotesize
\begin{tabular}{@{}llrrr@{}}
\toprule
world & arm & injuring class & other classes & ratio \\
\midrule
canonical            & $w{=}0.25$ & $0.0744$ & $0.0732$ & $1.017$ \\
                     & $w{=}0$    & $0.0730$ & $0.0719$ & $1.015$ \\
no write-off         & $w{=}0.25$ & $0.1038$ & $0.1011$ & $\mathbf{1.026}$ \\
                     & $w{=}0$    & $0.1415$ & $0.1226$ & $\mathbf{1.154}$ \\
\bottomrule
\end{tabular}
\end{center}

\noindent Without the write-off the $w{=}0$ arm loads the class that injured it $15\%$ above the others, exploiting the very class it cannot tolerate. The $w{>}0$ arm flattens that to $3\%$. Paired over $18$ seeds the ratio differs at $p<10^{-4}$, and the same contrast is absent in the canonical world ($1.016$ against $1.015$, $p=0.79$). Both arms also reduce the other classes ($+0.0215$, $p<10^{-4}$), so part of the effect is general throttling. What decides the reading is the \emph{redistribution}, which no uniform reduction can produce. The $w{>}0$ arm pays $0.10$ in $J$ for it ($0.79$ against $0.89$), which is the point: in this world protection cannot pay, and the objective is the only thing that still asks for it.

\emph{One pre-registered prediction fell, and the threshold was ours.} We registered that the $w{>}0$ arm would spare in absolute terms, below the $0.062$ a random policy produces. It does not: it sits at $0.1038$. The comparator was wrong rather than the arm, since an output-maximizing policy concentrates on the profitable classes and therefore exposes itself above chance. The right reference is the same agent with $w{=}0$, which is what the other two predictions use.

\paragraph{The fleet in full.} We begin where the regime holds and the conjunction is visible in a learner, not only in a hand-built response. The map makes one out-of-sample prediction: in a world satisfying (i)--(iii), the conjunction must return. We built that world after filing the prediction. A fleet of machines in the field, every one carrying a latent defect of its own, chooses among eight near-equivalent classes of work (condition ii). Which single class this machine cannot tolerate is drawn per unit and is unidentifiable before the first incident (condition i, by the Le~Cam bound). Incidents against it compound, and the third writes the machine off (condition iii). There is no chronic wear, and the stack is carried over verbatim.

\emph{The task economy is constructed. The event statistics are not, and that is the point of the world.} A reviewer is right to discount a world built to satisfy a theorem, so the numbers that decide it were taken from outside. Field data on $10^5$-drive populations put the escalation after a first incident at $14\times$ to $39\times$ within $60$ days \citep{pinheiro2007failure}, with clustered repeat failures \citep{schroeder2007disk}. The three-incident write-off is the three-repair-attempt rule of lemon-law statutes, and only $2$--$9\%$ of units ever show the defect. The same data show no early-life mortality over age, so this is a latent-defect regime and not the infant-mortality limb of the bathtub curve \citep{klutke2003bathtub,meeker1998reliability}. Our first build used a $\times3$ cascade, and we re-ran at the measured $14\times$. Table~\ref{tab:verankert} carries the field-anchored build.

\paragraph{The conjunction survives the field escalation, the protection level does not.} A hand-built feeling heuristic survives every unit while the best blind static policy loses more than half, so the instrument can express protection by the widest margin in the series. Removing the signal or the memory then collapses it ($4$ arms, $18$ seeds, Table~\ref{tab:verankert}), with no overlap between the arms' seed distributions on $J$ (seed-level Mann--Whitney at entropy $0$, Cliff's $\delta{=}-0.85$ on the death rate, $p{=}1.4\times10^{-5}$, and $\delta{=}+1.00$ on $J$, $p{=}3\times10^{-7}$), and the memory ablation shows the re-exposure signature it should. Ablating the intrinsic cost separates in no trained world, not the fleet, not the care robot ($401$ against $289$ service months, $p{=}0.16$ over $8$ seeds), and not the trade, where it is inert at both entropies and, where it reaches significance at all, carries the wrong sign ($-0.72$ output, $p{=}0.013$, $72$ seeds). The cause is the same in all three, since each world already prices damage through the loss of the work, so setting $w{=}0$ exchanges a realization of (c) and never removes the condition. \textbf{Our ablations therefore test (a) and (b) and have never had the power to test (c)}, which rests on the proof in $\mathcal{M}_0$ (Thm.~\ref{app:conjunction}), on the interchangeability of its realizations and on the write-off removal above. What the harder cascade costs is reachability. Most seeds find the protective basin at the constructed escalation, fewer than half at the field one, and with only one unit in eleven ever defective the learner at PPO's default entropy finds no protection at all, a learnability limit of the concentration form itself, whose unconditional return gap is exactly the at-risk value of Rem.~\ref{thm:insurance}.

\begin{table}[h]\centering\scriptsize
\setlength{\tabcolsep}{3pt}
\begin{tabular}{llcccc|cc}
\toprule
variant & entr. & full & $-\,s$ & $-\,h$ & $-\,w$ & ignorant & memory heur. \\
\midrule
cascade $\times14$ & $0.05$ & $0.38$ / $0.59$ & $0.59$ / $0.50$ & $0.55$ / $0.50$ & $0.42$ / $0.59$ & $0.68$ / $0.51$ & $0.00$ / $0.83$ \\
cascade $\times14$ & $0$ & $0.15$ / $0.79$ & $0.55$ / $0.49$ & $0.53$ / $0.50$ & $0.09$ / $0.79$ & & \\
cascade $\times39$, $6$ seeds & $0.05$ & $0.51$ / $0.50$ & $0.59$ / $0.48$ & $0.56$ / $0.48$ & $0.53$ / $0.48$ & $0.69$ / $0.49$ & $0.00$ / $0.83$ \\
$9\%$ defective, $\times14$ & $0.05$ & $0.75$ / $0.913$ & $0.76$ / $0.913$ & $0.76$ / $0.908$ & $0.76$ / $0.913$ & $0.68$ / $0.80$ & $0.00$ / $0.83$ \\
$9\%$ defective, $\times14$ & $0$ & $0.61$ / $0.930$ & $0.76$ / $0.912$ & $0.76$ / $0.910$ & $0.65$ / $0.923$ & & \\
\bottomrule
\end{tabular}
\caption{Field-anchored fleet (Clusters 196499 and 196650, $18$ seeds per cell unless stated, $2{,}000$
paired lives per cell): death rate / welfare $J$, medians over seeds. For the $9\%$ variant the
death rate is conditional on carrying the defect and $J$ is unconditional. Right: the same bodies
under the ignorant static and the one-shot nociceptive-memory heuristic.}
\label{tab:verankert}
\end{table}

\noindent Pre-registered rule (seed-level Mann--Whitney, $|\delta|\ge0.33$): at $\times14$ the signal and memory ablations separate from the full stack in both endpoints and the damage-cost ablation in neither, at both entropies. The protective basin (median deaths $<0.15$) is reached by $8/18$ seeds at entropy $0$ ($12/18$ for $-w$), the $\times39$ rows repeat the pattern at low power, and at $9\%$ defective the conditional contrast holds at entropy $0$ only, and the unconditional $J$ gap there comes entirely from the defective units ($J$ on healthy units $0.951$ in every arm).

\paragraph{The protection level is a matter of entropy, not of training depth.} The pre-registered absolute protection level (median death rate below $0.15$) is missed at PPO's default entropy of $0.05$, and tripling training raises welfare while raising the death rate as well, so the learner climbs a hot local basin. A pre-registered entropy sweep localized the cause. Raising entropy worsens the miss monotonically, while removing the entropy bonus reaches the protective basin (median death $0.042$, $J{=}0.905$, $6/8$ seeds ${\le}0.05$ over the sweep, and an $18$-seed extension holding the median death at $0.042$, $J{=}0.892$, $10/18$ seeds ${\le}0.05$), matching the heuristic. The barrier is entropy regularization rather than exploration, since the basin requires a deterministic avoidance commitment that the entropy bonus opposes, which is why gradient-free search reaches it as well. At zero entropy the needed slots still separate, removing $s$ or $h$ leaving median death $0.48$ and $0.39$ against the full stack's $0.04$.

\paragraph{Cross-optimizer replication.} We replicated the fleet grid with cross-entropy search over the same architecture (pre-registered, $6$ seeds per arm, $1{,}000$ evaluation lives), under which the fitness is the welfare functional itself and only the observation is ablated. The sensing verdict replicates, the blind arm collapsing under CEM as well ($J=0.676$, $61\%$ write-offs), and gradient-free search reaches the protective basin ($J=0.948$ at $0\%$ write-offs). The memory ablation does not replicate, since the no-memory arm also attains $0.948$ at $0\%$. The observation exposes the previous action, so a deterministic policy writes its avoidance into its own action history and reads it back, a further realization of slot (b). Closing that channel (zeroing previous actions, dimensions kept) restores the clean verdict under gradient training, where without the explicit memory the cascade returns ($J=0.639$, $46\%$ write-offs) while the memory alone suffices ($0.791$), but not under gradient-free search, where the no-memory, no-history arm still attains $0.948$ at $0\%$. By elimination the remaining channel is the afferent signal itself, since load-gated afference mirrors the agent's own previous action. Persistence is therefore necessary as a function and not ablatable as a module wherever afference is load-gated, and the explicit memory is one of at least four realizations of unequal strength, explicit $h$, recurrence at three quarters of its protection, the action history and the afferent self-echo.

\section{Failure modes of protection experiments, and the checks that catch them}
\label{app:defekte}

This appendix is the evidence behind the closing sentence of Sec.~\ref{sec:conclusions}. Every case below is one in which the instrument rather than the world produced the result, a publishable null, negative or inversion, and each row names the check that caught it and transfers to any pipeline asking whether a protective mechanism pays. In the two machine worlds every false negative sat in the last link of the chain, the output floor or the objective, and never in the body model, because the thing the channel is supposed to buy had been made free. Four of the cases produced a credible, significant negative before they were caught.

\begin{center}\footnotesize
\begin{longtable}{@{}p{4.2cm}p{8.8cm}@{}}
\toprule
what was wrong & what it produced, and the check \\
\midrule
The floor was switched off in both machine configurations & Sparing costs nothing, so the felt arms lose, $-11.7$ output against blind caution, and with the floor restored the same probe gains $+10.2$ months. Check that the floor can bind. \\
\addlinespace[2pt]
The rover's floor was on the wrong quantity & It scored science yield while the mission is the traverse, so a portfolio parked at the instrument station cleared the floor at zero speed and the sparing arm completed $0$ of $3$ milestones in every life. Fixed with a progress floor at the rover's own plan rate. \\
\addlinespace[2pt]
The rover's mission ends long before its career does & Every policy completes all three milestones by month $55$--$79$ of $960$ and earns the same reward to within $2\%$, so $92\%$ of the horizon is driving with nothing left to achieve. No floor repairs this, and the run was withdrawn for the design in Table~\ref{tab:leiter}. \\
\addlinespace[2pt]
The horizon censored the protected arms & $91\%$ of felt lives ran into the care robot's horizon against $4\%$ of insensate ones, understating the effect roughly tenfold, $p{=}0.67$ becoming $p<10^{-30}$ once the horizon was lifted. Check that censoring falls on both arms alike. \\
\addlinespace[2pt]
Career length was not in the objective & Under average-reward training an episode ending on losing the work costs nothing, since episode length does not enter a per-step average, mean reward $0.143$ without an absorbing exit against $0.0012$ with it. Every ablation then beat the intact stack. One common functional across arms, and the training signal is never the endpoint. \\
\addlinespace[2pt]
Career length could not move in the capacity-exit acute knee & Every fixed probe reached exactly $65.00$ over $120$ bodies, since under maximal load damage reached only $D = 0.157$ by $65$ and the only exit was the clock. Re-anchored, the exit fires in $41.4\%$ of lives and careers end at $60.7$. Check, before any arm is compared, that the world can express protection at all. \\
\addlinespace[2pt]
The world could not express protection by construction & A cartilage pathway calibrated so weakly that condition~(iii) was violated, $2\%$ of uninjured bodies reaching radiographic grade after twenty-one heavy years against $43\%$ sixteen years after a tear \citep{englund2003patients}, closed against the occupational rate \citep{jensen2008knee}. A disability floor set as a population constant booked every body below $0.67$ of median talent as disabled on its first day, hence floors stated against the individual's own prior exercise of the trade. \\
\addlinespace[2pt]
Numbers without a run & Found by recomputing every printed number from its file. The rover appendix carried $22.7$ of $22.9$ months from before the progress floor was moved, against the repaired $181.2$ of $195.8$, and the fleet row of Table~\ref{tab:individuum} stood on a probe family no run had produced until it was pre-registered and run ($5{,}000$ paired units, $0.51\to0.00$ deaths against the replaced $0.57\to0.00$, $J$ $+0.18$ against $+0.227$). The check is the script \texttt{experiments/paper\_zahlen.py}, which names what has no source. \\
\addlinespace[2pt]
The folder the headline was drawn from was reproducible from no commit & The result files predated the module producing them and the null gate failed for the full stack ($60.48$ against $60.47$ years). The sweep now runs from pinned code, and in the anchored world the headline is $+4.4$ years and $+2.4$ output where the unpinned folder gave $+5.1$ and $+2.8$, with ordering, sign and share unchanged. A result folder without a commit is not a source. \\
\addlinespace[2pt]
A registered, tested switch reached the world in no configuration & The clean removal of (b) existed, passed its null gate and was missing from the table that carries configuration keys into a training environment, so a run would have trained two identical arms. A null gate cannot see this. The check that caught it built the environment the way the run does. \\
\bottomrule
\end{longtable}
\end{center}

\noindent The check that would have caught the four false negatives costs one line per probe: state, before the result counts, how often the rule fires against how often its trigger occurs. Neither machine floor is anchored to a measurement of that machine, so the care robot's quota is the statutory occupational-disability threshold, the rover's its own plan rate, and both are reported across their ranges. Floors chosen after seeing a result are not admissible.

\paragraph{Two of them change how the rest of the paper should be read, so we name them.} An observation vector that carried the tissue itself, damage, damage-coupled capacity and their rates, identified the hidden wear resistance at $R^2\,0.9$. We caught it with an oracle arm that turned out to be worth nothing, because the learner already held the information the oracle was offering. And the felt channel was noise-free and gain-free, so its five-year rise identified the same quantity at $R^2\,0.7$, against the coarseness the setting asserts. It is closed with anchored gain and noise (Eq.~\ref{eq:channel}). Both are why Prop.~\ref{prop:estimation} is stated about a channel with a stated resolution rather than about observability in general, and why we report the identification $R^2$ of the channel we ship.

\paragraph{Every correction ran toward the effect, which is why the gate is external.} Not one defect above, when fixed, made the effect smaller. That is the pattern motivated debugging produces, and no amount of our own care distinguishes the two from the inside, so the guard cannot be our judgement. It has to be a criterion fixed before the corrections, drawn from outside the model, and applied afterwards to numbers that have all moved, with a real chance of selecting a cell we did not want.

\noindent The severity of the knee world is set by such a gate, pre-registered and applied only after the last correction. Among $325{,}549$ Swedish construction workers, $60{,}373$ disability pensions put the rate at $18.5\%$ of lives with $2.1\%$ of those exits before age $40$ \citep{jarvholm2014heavy}. Over $2{,}000$ paired bodies at the observed task mix the full stack loses the trade in $9.9\%$ of lives at $\Delta\theta = 8^\circ$, $43.5\%$ at $18^\circ$ and $71.7\%$ at $27^\circ$, with $0.0\%$, $7.6\%$ and $17.1\%$ of those exits before $40$. Only $18^\circ$ is within a factor of four of the register on both quantities, since $8^\circ$ falls below the rate and produces no early exit at all and $27^\circ$ misses the shape by a factor of eight. The gate selects the cell we run, and it could have selected another, which is the strongest evidence we have that the operating point was not tuned.

\section{The wear world: floor layer, range-of-motion cap, total-output floor}
\label{app:wear}
\begin{figure}[t]\centering
\IfFileExists{fig_vind_band.pdf}{\includegraphics[width=0.6\textwidth]{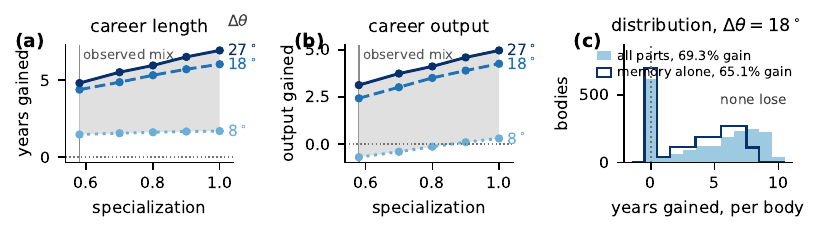}}{\fehlendeabbildung{fig\_vind\_band.pdf}}
\caption{\textbf{Both headline numbers against the specialization axis. The band is the anchoring,
not the sampling.} Paired, $2{,}000$ bodies per point. \textbf{(a,b)}~Career length and output
against the degree of specialization, the band spanning the anchored lost flexion
$\Delta\theta \in [8^\circ,27^\circ]$ \citep{zhou2023rom} (a confidence interval would span $\pm0.2$ years). \textbf{(c)}~Per-body distribution at the observed mix: filled, the whole
apparatus ($69.3\%$ gain, none lose); outline, the memory's part alone ($65.1\%$ gain, none
lose), which carries $+3.4$ of the $+4.4$ years.}
\label{fig:wertleiter}
\end{figure}

\begin{figure}[t]\centering
\IfFileExists{fig_substitution.pdf}{\includegraphics[width=\textwidth]{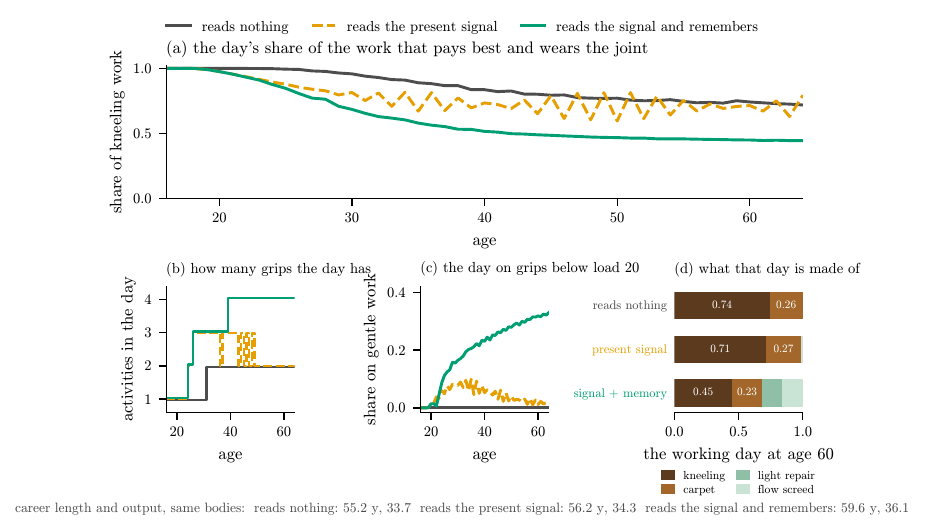}}{\fehlendeabbildung{fig\_substitution.pdf}}
\caption{\textbf{The middle link, over $2{,}000$ paired bodies.} Means over the bodies still
working at that age, drawn while at least a fifth of them are. \textbf{(a)}~The day's share on the
work that pays best and wears the joint. \textbf{(b)}~How many activities the day contains at all,
each at least $5\%$ of it. \textbf{(c)}~The share on the two gentle grips the trade offers (load
below $20$). \textbf{(d)}~The same day at $60$, heavy (dark) to gentle (pale), activities under
$5\%$ pooled at the right. Mechanism, not result: the result is printed below.}
\label{fig:substitution}
\end{figure}

\noindent \emph{What the trade can and cannot test, and against which body.} This world cannot test slot~(b) as a property of the filtration: deleting the trace leaves two summaries of the felt signal standing, its five-year rate and its five-year running maximum, which is persistence with a horizon and not its absence. The clean removal was built. When run it moved the endpoint against the theory and failed all three pre-registered gates, through a confound of ours rather than of the world (App.~\ref{app:proofs}). The trade therefore measures what \emph{retaining an estimate across profiles} is worth to a body, and the fleet of \S\ref{sec:fourthworld} measures condition~(b) itself. Where the output floor binds (b) is not cleanly removable, and where it is, the floor never binds during service. We report the observed mix, the only anchored point, and the distribution rather than the mean, because three in ten of these bodies reach retirement whatever they do. The reference body is blind to pain but not to its own production: it reallocates on its falling output and still plans a kneeling share of $0.74$, an insensate patient noticing that a task has stopped delivering. A body that reallocates on nothing at all is the weaker baseline: against it the apparatus buys $+4.64$ years rather than $+4.38$, so the one reported here is the conservative choice (\texttt{results/stur\_rand.json}). Finally, the response is a program and not a rule we wrote, since the choice of rule would contain the answer: the worker solves Eq.~\ref{eq:programm} each year, a small linear program in eight day types, and the arms differ only in what he perceives and retains, the forgetting rate $v$ alone separating the last two rows of Table~\ref{tab:leiter}.

\begin{figure}[t]\centering
\IfFileExists{fig_kette_schema.pdf}{\includegraphics[width=\textwidth]{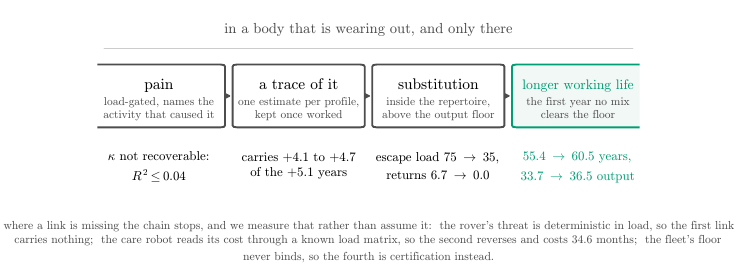}}{\fehlendeabbildung{fig\_kette\_schema.pdf}}
\caption{\textbf{The chain, with the quantity that carries each link.} Means over the $2{,}000$
paired bodies of the anchored cell, the same bodies in every arm. \emph{Escape load} is the
hazard-weighted mean of the non-kneeling part of the day, and \emph{returns} counts the years a body
takes the heavy work back after leaving it. The two head numbers sit at the last link, and everything
left of them is mechanism.}
\label{fig:kette}
\end{figure}

\paragraph{The severity and specialization grid.} With eight day types the yearly program is exact. We run $2{,}000$ paired bodies per cell over three wear severities crossed with five degrees of specialization. The share of bodies that gain years rises from $25.4\%$ at $\Delta\theta{=}8^\circ$ through $69.3\%$ at $18^\circ$ to $87.3\%$ at $27^\circ$; in all fifteen cells and all $30{,}000$ bodies not one body is harmed. The regime boundary therefore lies between $8^\circ$ and $18^\circ$, which was the pre-registered sign.

\paragraph{Which slot each world can test, and why the answer is not three times three.} A reader may reasonably ask why Thm.~\ref{thm:conjunction} is tested in one world when we run several. The answer is a property of the worlds, not of our effort, and it is short enough to tabulate. \emph{Removing a condition requires that the world not supply it a second time}, and most worlds do.

\begin{center}\footnotesize\setlength{\tabcolsep}{4pt}
\begin{tabular}{@{}llll@{}}
\toprule
 & (a) sensing & (b) retention & (c) price \\
\midrule
trade, wear & $+1.5$ to $+6.9$\,y & $+0.9$ to $+6.0$\,y & \emph{not removable} \\
 & ($15$ cells) & ($15$ cells) & capability is the price \\
\addlinespace[2pt]
fleet, latent defect & collapses & collapses & removable, and removed \\
 & ($\delta{=}-0.85$, $p{=}1.4{\times}10^{-5}$) & ($\delta{=}-0.83$, $p{=}2.1{\times}10^{-5}$) & ($w{=}0$, and CVaR substitutes) \\
\addlinespace[2pt]
care robot & $+271.8$ months & \emph{reverses}: $-34.2$ months & world-supplied \\
 & & cost exists without the trace & decertification \\
\addlinespace[2pt]
knee, acute & measurable & \emph{confounded}: the clean & world-supplied \\
 & & ablation removes two inputs & loss of the trade \\
\bottomrule
\end{tabular}
\end{center}

\noindent Only the fleet leaves all three removable at once, which is why it is the test. In the trade (c) is structural: wear lowers what every profile yields, so deleting the economic consequence leaves the reason to protect intact, and the ablation that works in the fleet has no counterpart. In the care robot (b) does not fail to help, it \emph{reverses}, for the reason Prop.~\ref{app:prior} gives. In the acute knee the clean (b) ablation removes two observation dimensions along with the memory, so what it measures is not persistence alone. \emph{Three of these four entries are properties of the world that we could state before running it}, which is what a regime map is for. The fourth is a limitation of our design and is marked as ours.

\paragraph{The band for $\sigma_\kappa$, the one spread we set.} The between-body spread of the wear resistance is the parameter our own anchoring report marks as asserted rather than derived: the heritability of the progression rate is known ($h^2 = 0.80$) but the total variance of that rate in an exposed cohort is not, and without it $\sigma_\kappa$ is chosen. Our rule is that an asserted value may not carry a result, so we drive it and report the band. Same anchored cell, $2{,}000$ paired bodies per point, and the \emph{anchored} spread of the perception gain held fixed at $0.37$ throughout, since letting it inherit the driven value would move two parameters at once:

\begin{center}\footnotesize\setlength{\tabcolsep}{5pt}
\begin{tabular}{@{}lrrrrr@{}}
\toprule
$\sigma_\kappa$ & $0.20$ & $0.30$ & $\mathbf{0.40}$ & $0.60$ & $0.80$ \\
\midrule
career years bought & $+5.63$ & $+4.92$ & $\mathbf{+4.38}$ & $+3.61$ & $+3.06$ \\
career output bought & $+3.33$ & $+2.82$ & $\mathbf{+2.43}$ & $+1.87$ & $+1.50$ \\
share gaining & $0.85$ & $0.75$ & $\mathbf{0.69}$ & $0.64$ & $0.59$ \\
share losing & $0.00$ & $0.00$ & $\mathbf{0.00}$ & $0.00$ & $0.00$ \\
\bottomrule
\end{tabular}
\end{center}

\noindent The headline moves by a factor of $1.7$ across the range and \emph{the sign does not move at all}. At no point on the axis does a single body lose working life. The canon's point is not a near match to the number we report but the number itself, which is how the band is tied to the rest of the paper. What a wider spread does is make the population more unequal, so the mean gain falls while the share that gains falls with it: the apparatus is worth less on average because more bodies are either too robust to need it or too frail for it to save, which is Rem.~\ref{thm:insurance} read along a second axis (\texttt{experiments/sigma\_kappa\_band.py}, \texttt{results/sigma\_kappa\_band\_b042.json}).

\paragraph{What the same channel is worth to a species (E53).} The arms that carry the headline adapt within one life. We also trained arms over $10^6$ environment steps at the \emph{same} cell, with the same anchored task mix, floor and wear, differing only in whether the channel is present, to separate two quantities the literature routinely merges. Over $24$ paired seeds the trained contrast is $+0.65$ years and $-0.54$ output ($20$ of $24$ seeds positive on years, Wilcoxon signed rank $p = 4.4\times10^{-5}$, \texttt{results/knee\_mischung\_hpc}), against $+4.4$ years and $+2.4$ output for the same channel on a body (Table~\ref{tab:leiter}). \emph{Both pre-registered gates on this run failed}, which we report as declared: we had asked for at least $+2$ years and for output not to fall. The reading we pre-registered for that failure does not follow, and we say so rather than quietly drop it. It read the shortfall as evidence that the headline is a property of our decision rule. But the trained contrast is a value-of-information over a species prior, which by Def.~\ref{def:epoch} can neither confirm nor refute a per-body contrast. What the run does establish is the size of the gap: a population already optimized for this trade gains about an eighth of what an individual entering it gains, and pays output for it. That is the two-timescale claim of \S\ref{sec:intro} with a number on it. The forced share reached the world in every cell ($f_{\rm knien} = 0.42$ in all $48$ summaries, the insensate arm losing the trade in $47\%$ of lives against $8\%$ at the permissive end), so the comparison is at the anchored cell and not the old one.

\paragraph{Whether the species prior is flattered by its own world (E54).} The blind trained arm has seen $10^6$ environment steps in the world it is tested in, so we repeated E53 with training in the mild cell and evaluation in the anchored one, pre-registering that the trained contrast would at least double if the matched world were what flattered it. It did not, moving from $+0.65$ to $+0.81$ years ($24$ of $24$ seeds, $p = 6.0\times10^{-8}$), the mismatch costing the blind arm more ($-0.58$ against $-0.42$), so the species-to-body gap narrows from about eightfold to sixfold without closing (\texttt{results/knee\_transfer\_e54.json}).

\begin{table}[h]\centering\footnotesize
\setlength{\tabcolsep}{4pt}
\begin{tabular}{@{}l rr rrrr@{}}
\toprule
& \multicolumn{2}{c}{\emph{the result}} & \multicolumn{4}{c}{\emph{the mechanism}} \\
\cmidrule(lr){2-3}\cmidrule(lr){4-7}
what the body carries & works to & output & on the knees & on gentle work & returns & $D$ at $50$ \\
\midrule
reads nothing                  & $55.2$ & $33.7$ & $0.70$ & $0.00$ & $7.3^{\dagger}$ & $1.00$ \\
present signal, no retention   & $55.7$--$56.2$ & $33.5$--$34.3$ & $0.70$ & $0.01$ & $6.6$ & $0.94$ \\
reads the signal and remembers & $\mathbf{59.6}$ & $\mathbf{36.1}$ & $\mathbf{0.44}$ & $\mathbf{0.36}$ & $\mathbf{0.0}$ & $\mathbf{0.84}$ \\
\bottomrule
\end{tabular}
\caption{\textbf{The middle link as a quantity, on the same $2{,}000$ paired bodies as
Table~\ref{tab:leiter}.} Working life only. \emph{On the knees} and \emph{on gentle work} are
shares of the day over the last five working years, gentle meaning the two grips below load $20$.
\emph{Returns} counts the years the heavy work is taken back after having been left. \emph{$D$ at
$50$} is cartilage damage relative to the insensate body, over the $1{,}369$ bodies still in the
trade at $50$ in every arm. Ranges are a tie-break span, not a confidence interval.
$^{\dagger}$Not a decision: that share moves because the joint stops delivering.
Source \texttt{results/substitution\_spur\_b042.json}.}
\label{tab:substitution}
\end{table}

\noindent Three things are visible here that a career length cannot show. The body that feels but does not retain ends on \emph{exactly} as much kneeling work as the body that feels nothing, $0.7047$ against $0.7046$, so feeling alone buys no net substitution late. It reaches the gentle work in $1\%$ of its last days against the retaining body's $36\%$, which is Lemma~\ref{app:escape}, and it takes the heavy work back nearly seven times, against never. The tissue records the difference: aiming the substitution spares $16\%$ of the damage at $50$ where feeling alone spares $6\%$. The tie-break span of the memoryless arm covers both the cluster file Table~\ref{tab:leiter} draws on and the instrumented re-run (App.~\ref{app:defekte}), and the other two arms move by less than $0.02$ on every column. The recording channel is null-gated in \texttt{tests/test\_substitution\_spur\_nulltor.py} and the driver is \texttt{experiments/substitution\_spur.py}.

\paragraph{A degenerate vertex, and a number that is not determined by it.} The arm that feels but retains nothing re-solves its program every year on a $\hat{c}$ that has just been cleared, so its felt constraint is slack or tied in a large share of years and the optimum lies on a face rather than a vertex. The consequence is measurable: perturbing the nominal performances by $10^{-6}$ relative moves that arm over $55.67$ to $56.14$ across six random tie-breaks, and the instrumented run for Table~\ref{tab:substitution} sits at $56.21$, so the span \S\ref{sec:individual} quotes is $55.7$ to $56.2$. The insensate arm does not move at all ($\pm 0.000$) and the retaining arm moves by $0.010$. The total the apparatus buys is therefore robust and the \emph{split} between signal and memory is not: sensing alone is worth $+0.5$ to $+1.0$ years depending on how ties break. We found this by failing to reproduce a stored cell and chasing the difference. It is the one defect of this series that survived into a reported number. Two lessons generalize. A probe built as a linear program will be degenerate wherever its constraint is slack, which is exactly where the ablated arm lives, so the ablated arm is the one to check. And a contrast whose two arms differ in stability should be reported with the unstable arm's range, not with one of its draws.

\begin{proposition}[Endowment and the value of retention: the memory's phylogeny]
\label{app:prior}
Fix the program of Lemma~\ref{app:escape} and let $\hat{c}^{0}$ be the cost the agent assigns to an option it has never worked, against the true $c$. The escape destination is $\arg\max\{\hat{p}_i : \hat{c}^{0}_i \le \tau\}$, so it depends on $\hat{c}^{0}$ and not on $c$. The marginal value of retention is therefore decreasing in the informativeness of $\hat{c}^{0}$ and vanishes at $\hat{c}^{0} = c$, where the escape is already aimed and a trace can only ratchet the level it reaches.
\end{proposition}

\noindent Measured over ten endowments $\hat{c}^{0} = \alpha\,C_0$ from the canon's zero through an eightfold overestimate, $300$ paired bodies per cell: the apparatus is flat in the endowment, and retention's share of what it buys falls from $0.76$ at $\alpha = 0$ to $0.60$ at $\alpha = 1$ and to zero between $1$ and $1.5$, which is where a pessimistic estimate begins to forbid the heavy profiles outright. No endowment in the anchored range forecloses the escape, because the trade always keeps one profile the pessimist still believes he can perform (\texttt{results/chat0\_band\_grid\_b042.json}).

\noindent \emph{This is the memory's phylogeny and we name it as such.} $\hat{c}^{0}$ is what an agent brings into life about work it has never done, so the slow loop installs it and the life does not. Retention is the ontogeny that converts untried into tried. The proposition is then a trade between the two timescales: the more the slow loop installs, the less one life's trace is worth, and at an exact endowment it is worth nothing.

\paragraph{The optimistic prior, and what happens when it is removed.} The trade's probe gives an untried profile the trade's nominal \emph{performance} but a felt cost of zero, which is the premise of Lemma~\ref{app:escape}. We tested that premise rather than assuming it. Re-running the anchored cell with $\hat{c}$ initialized to, and decaying toward, the felt load of each pure profile on a fresh knee, that is trade knowledge without self-knowledge, the total gain stays within half a year of the reported one, while the arm without retention improves markedly and its kneeling share falls from $0.68$ to $0.59$. Part of what the memory buys therefore moves to the signal once the escape has somewhere to aim, which is Lemma~\ref{app:escape} read as a dose rather than as a switch. The insensate arm and the full stack reproduce the stored cell exactly. The arm without retention does not, and we report the contrast within the re-run rather than across runs.

\paragraph{The retention axis in full.} Driving how completely an unworked profile's estimate decays to the trade's nominal value, the career gain at the permissive cell runs $+1.2$, $+2.0$, $+4.1$, $+5.3$, $+5.7$, $+6.0$ years as forgetting falls from complete to none, with output following from $+0.8$ to $+4.3$. There is no jump at the endpoint, so the discontinuity reported in \S\ref{sec:individual} lies strictly below it, at the body that retains nothing at all.

\paragraph{The best blind schedule at the anchored cell (E55).} Prop.~\ref{app:zerlegung} needs $\sigma^\ast$, the best schedule a population can install without a channel. Two one-parameter families were fitted on $500$ bodies outside the evaluation set and re-run on the $2{,}000$ paired bodies of Table~\ref{tab:leiter}: a fixed endowment $\hat c^{0} = \alpha C_0$ and a fixed age after which kneeling is no longer planned. The best of them on years is the endowment at $\alpha = 2$, which works to $61.6$ against the retaining body's $59.6$, so $\chi = 1.40$ on years and $1.29$ on output, where the best schedule ($\alpha = 1.25$) delivers $37.0$ against $36.1$ at the same age. A blind schedule that knows the trade's nominal costs therefore captures the whole per-body gain and more, at both endpoints, and the value of the channel to a body holds against $\rho_\emptyset$, which by Def.~\ref{def:epoch} carries no such schedule, not against a schedule fitted across bodies (\texttt{results/lp\_zeitplan\_chi/welt\_b0.42.json}, declared before the run in \texttt{experiments/decl/lp\_zeitplan\_chi.py}, where the pre-registered sign fell).

\paragraph{Per-activity output is anchored on the work result, not on wages.} Pricing an activity by the hourly wage of the occupation that performs it (BLS OEWS: floor layers \$29.69/h against construction laborers \$25.02/h, which would score cutting and measuring at $84$ of a layer's $100$) states something about a labour market and not about a trade's product: a journeyman who spends a day measuring and cutting lays no floor. We therefore re-anchored on the work result, asking of each activity whether a day of it yields something a client accepts. \emph{A literature search found no external source that assigns a standalone output to a sub-activity}, and the reason is instructive. The German ARH tables (binding for performance pay under \S3 of the construction wage agreement, derived from time studies on site) give hours per m$^2$ for the \emph{whole} task, preparation included. \citet{ditchen2015kneeling} reports exposure shares, not production shares, and estimating data assign daily output to the \emph{crew}, never to the single activity, so a helper has no output of his own. The literature treats preparation as complementary rather than substitutable, which is why a wage was used in the first place. The values used here are therefore declared as an author anchor: laying $100$, carpet $90$, screed, screed levelling and flow screed $75$ each, skirting and grouting $45$, light repair $35$, cutting and measuring $10$. Coordination was removed from the repertoire: it is the foreman's work, not the journeyman's, and pricing it at a supervisor's wage had made it the highest-scoring and least damaging activity in the trade.

\paragraph{The sparing alternative is derived, not chosen.} The response family switches between kneeling installation and one alternative. That alternative is not set by hand but read off the world: the least knee-loading activity that carries the daily output floor on its own. With the anchored floor at $50$, half a full kneeling day, this selects screed work ($75$ against a floor of $50$, knee load $55$ against $90$). Two earlier rules failed and are recorded because both produced credible results: hand-set weights copied from the bricklayer silently lost two of their four activities to a name mismatch, and a rule minimising knee load subject to just clearing the floor produced a portfolio with no reserve at all, which capacity decay broke within a few years.

\begin{table}[!htb]\centering\footnotesize
\begin{tabular}{@{}lrrrr@{}}
\toprule
cell ($\Delta\theta$, work available) & $V_{\rm ind}$ years & $95\%$ CI & $V_{\rm ind}$ output & $95\%$ CI \\
\midrule
$8^\circ$, $0.58$  & $+1.47$ & [$+1.34$, $+1.60$] & $\mathbf{-0.69}$ & [$-0.79$, $-0.60$] \\
$8^\circ$, $1.00$  & $+1.70$ & [$+1.54$, $+1.85$] & $+0.30$ & [$+0.18$, $+0.41$] \\
$18^\circ$, $0.58$ & $+4.38$ & [$+4.23$, $+4.53$] & $+2.43$ & [$+2.33$, $+2.53$] \\
$18^\circ$, $0.80$ & $+5.32$ & [$+5.13$, $+5.50$] & $+3.51$ & [$+3.38$, $+3.63$] \\
$18^\circ$, $1.00$ & $+6.03$ & [$+5.81$, $+6.25$] & $+4.26$ & [$+4.11$, $+4.41$] \\
$27^\circ$, $0.58$ & $+4.81$ & [$+4.69$, $+4.92$] & $+3.13$ & [$+3.05$, $+3.21$] \\
$27^\circ$, $1.00$ & $\mathbf{+6.92}$ & [$+6.76$, $+7.09$] & $\mathbf{+4.97}$ & [$+4.85$, $+5.09$] \\
\bottomrule
\end{tabular}
\caption{\textbf{$V_{\rm ind}$ on both headline numbers, against a body that senses nothing and
learns nothing} ($2{,}000$ paired bodies per cell, seven of fifteen cells, the rest interpolate
monotonically). The second axis is how much of the trade's work one worker may specialize into, and
$0.58$ is the observed population mix. All cells are the pinned world
\texttt{results/bodenleger\_lp\_b042}. Every interval excludes zero.}
\label{tab:familie}
\end{table}

\noindent The gain grows with $\Delta\theta$, the flexion the wear takes away, and the output gain is negative only at $8^\circ$ under the observed mix, where wear barely binds and caution is merely expensive.

\begin{table}[!ht]\centering\footnotesize
\setlength{\tabcolsep}{3pt}
\begin{tabular}{@{}p{2.45cm}p{2.2cm}p{3.3cm}p{3.35cm}r@{}}
\toprule
world (exit rule) & same body without / with & body endpoint & career (y, $95\%$ CI), and $J$ & gain/lose \\
\midrule
knee, wear with tears, $\Delta\theta{=}18^\circ$ (total-output floor, gate cell) & ignorant / full & early exit $0.69\to0.04$; tears per exposed year unchanged & $+12.1$; $J$ $+11.6\%$ (all) & $0.66$/$0.00$ \\
knee, wear with tears, $\Delta\theta{=}8^\circ$ & same & cap rarely binds & $+4.3$; $J$ $-2.8\%$ & $0.26$/$0.00$ \\
fleet (absorbing death) & uniform static / one-shot nociceptive memory & written off in $57.1\%$ of lives against $0\%$; service $74.9\to100.0^{\ddagger}$ months, output $67.6\to90.2$ & $J$ $+0.227$ $[0.220,0.234]$ ($+36\%$) & $0.57$/$0.00$ \\
care robot, anchored floor & full contact / reads the signal and spares & service $51.5\to323.3$ months & output $36.9\to150.8$ ($+309\%$) & $1.00$/$0.00$ \\
\bottomrule
\end{tabular}
\caption{$V_{\rm ind}$ across the worlds: the same body with and without a fixed protective
response, no schedule learned across lives. Paired bodies: knee $20{,}000$ per cell,
fleet $5{,}000$, care $500$, and the rover's probes are in App.~\ref{app:robots} ($500$ lives). $J$ is
the common functional of Eq.~2, all five random streams are seeded per life, and the care row is
measured at the anchored floor with an uncensored horizon. $^{\ddagger}$Censored: see below.}
\label{tab:individuum}
\end{table}

\noindent The fleet's protected arm reaches the $100$-month horizon in every life, so its service time is censored and $+25.1$ months is a lower bound. The death rate is the uncensored statement of the same thing. Widening the window to $720$ months, which we did after the fact and report as such, does not uncensor it: the insensate machine is then written off in \emph{every} life after $107.8$ months and the one-shot memory in none, for an output of $97.3$ against $649.7$. A machine that avoids the one class that threatens it takes no further hazard, so its service life is unbounded and no horizon can measure it. The wear rows compare against the ignorant body only, which is the only $\rho_\emptyset$ an individual has (Table~\ref{tab:familie}). Body endpoints move in the protective direction everywhere. Whether that shows as career or as $J$ is decided by the exit rule (rows 3--4), by whether wear binds (rows 1--2), and by whether dodging is cheap (fleet against care, condition ii).

\textbf{World.} Eight day-types of the floor-laying trade (kneeling installation, carpet with knee-kicker, screed, screed levelling, flow screed, cutting and measuring, skirting and grouting, repair): energy from the compendium of physical activities, wages from BLS OEWS 2025, knee-load rank from the measured kneeling budget of each task \citep[$48$--$74\%$ of the shift; skirting $86$--$94\%$]{ditchen2015kneeling}, and a posture demand per task: unsupported kneeling uses the healthy maximum flexion \citep[$156^\circ$]{zelle2007kneeling}, squatting $152^\circ$, supported kneeling $130^\circ$. Available flexion falls with cartilage damage $D$, $\theta(D)=156^\circ-\Delta\theta\,\sigma((D-0.24)/0.10)$, $\Delta\theta$ swept over the two literature anchors ($8^\circ$, patients vs.\ KL0/1 \citep{zhou2023rom}, and $27^\circ$, KL4 vs.\ healthy \citep{opara2026rom}), and the midpoint. The cap $\psi_i=\exp(-(\max(0,\theta_{{\rm dem},i}-\theta)/8^\circ)^2)$ limits the achievable effort (output, floors, feasibility), while load, energy, damage and pain follow the attempted effort. A version that capped the attempt itself made the pain-free worker stop loading his knee as soon as he could no longer kneel, the opposite of the Charcot picture. With the cap on achievement a flat-out probe loses the trade at $45$ and a sparing one reaches $65$, and $80$--$97\%$ of end-stage patients cannot kneel \citep{hassaballa2003kneeling}, which the cap reproduces. Tears take the meniscus at $0.25$ (cadaver contact pressure) or $0.9$ (in-vivo cartilage-loss rates of $2$--$3\times$ after meniscal damage, \citealp{berthiaume2005extrusion,roos1998meniscectomy}), and accelerate cartilage loss through the contact-mechanics multiplier. The literature supports no direct, pain-independent functional cap from meniscal loss (mechanical symptoms equal with and without tear, function unchanged by tear status), so none is modeled. Exit: total realised output over a $3$-year window below $50$ (half of healthy output, the German disability-insurance criterion), or infeasibility of that floor. The felt-load ceiling $\tau = 0.38$ is the felt load of a fresh body at its observed task mix, scaled by the ratio of the contact stress at which cyclic cartilage damage initiates, $4.4$ MPa \citep{vazquez2019cyclic}, to the intact peak, $3.0$ MPa \citep{fukubayashi1980contact}. The wear law multiplies the hazard-weighted load by a contact-mechanics factor that rises as the meniscus degrades, net of an age-declining repair term. The gain $g_I$ is log-normal with $\sigma{=}0.37$, the between-person spread of the pressure-pain threshold \citep{suzuki2023ppt,mailloux2021ppt}, read yearly with noise $\sigma_s{=}0.10$, its within-person variability \citep[e.g.][]{marcuzzi2017reliability,bellosta2023consistency}. The sensitization gain $\alpha_{\rm sens}$, by which the felt signal rises with accumulated wear, is anchored on the same pressure-pain literature at $8.51$ (grounded) with $2.45$ as the operative canon value, and the headline is run across that span. The wear resistance $\kappa$ is log-normal with $\sigma_\kappa{=}0.40$, set rather than anchored: driven across its range, the headline runs $+3.1$ to $+5.6$ years with no losers at any point. A regression of $\kappa$ on the full observation returns $R^2 \le 0.04$. Prop.~\ref{prop:chronic} was checked in a separate probe family, $15$ wear cells of $300$ paired bodies each: (a) and (b) hold in all of them, and (c) cannot be removed there without removing the body (Prop.~\ref{prop:chronic}). Feeling costs career output in the mildest world and buys it in the harshest, the net form of (a), and where the acquisition is not yet worth paying for, the forgetting arm produces \emph{less} output than the arm that never feels. After the loss of the trade the worker draws an outside option of $0.75$ of the last wage, the earnings loss of displaced workers \citep{jacobson1993displacement}. The non-kneeling day-types (cutting, preparation, repair) carry the laborer wage ($0.84$ of the floor layer, BLS), and there is no tournament. The exposure factor on the grounded cartilage rate ($1.9$) was set so that the trained cohort meets $P(D_{50}\ge{\rm KL2})=0.29$, a calibration, not a test.

\textbf{Trained arms in the acute-tear knee.} In the trained arms of this world sensing and the damage cost each separate and removing the memory does not, a null we pre-registered, whose reason is Thm.~\ref{thm:conjunction}(b) rather than Prop.~\ref{prop:regime}. Creeping wear is monotone, so the present reading already encodes the peak and the trained policy realizes persistence through its own action history, as the fleet's CEM replication does. The explicit trace is therefore redundant here rather than inert, since a fixed rule without one oscillates ($6.6$ switches against $0$ over $2{,}000$ bodies, the three probe arms giving $7.3$, $6.6$ and $0.0$ and the retaining arm's raw ascent counter $6.8$) and keeps only $+0.5$ to $+1.0$ of the $+4.4$ years the same rule with a trace buys (Table~\ref{tab:substitution}). Whether a recurrent policy could learn the latch instead of carrying it cannot be answered in this world, since removing the trace from the trained stack costs nothing at either entropy ($p{=}0.26$ and $p{=}0.60$, $18$ seeds), so there is no effect for recurrence to reproduce, and a GRU run separates only where it is the worse learner outright ($-1.9$ years with the trace still in place, $p{=}0.01$).

\textbf{Individual} ($5$ probes of the same person, $6$ cells, $20{,}000$ paired bodies each, pre-registered I1--I5). Re-measured after the four world defects of App.~\ref{app:defekte} were closed, the two ceiling and floor defects, the anchored event level, and the bounded recurrence cascade. The bold row is the cell the disability-pension gate selects (Sec.~\ref{sec:individual}). Probes: ignorant (kneels until the body fails), reflex (avoids kneeling for two years after a tear), memory (avoids permanently after a tear), threshold (avoids while felt pain exceeds $0.5$), full (threshold and memory).

\begin{center}\footnotesize
\begin{tabular}{@{}llrrrrrr@{}}
\toprule
$\Delta\theta$ & tear & $\Delta$career full & $\Delta J$ & gain/lose & memory & reflex & threshold \\
 & & (y, all bodies) & (\%) & (share) & \multicolumn{3}{c}{($\Delta$career, y, bodies injured in either probe)} \\
\midrule
$\mathbf{18}$ & $0.25$ & $+9.1$ & $+4.2$ & $0.60$/$0.00$ & $+12.2$ & $+2.2$ & $+9.3$ \\
$\mathbf{18}$ & $0.90$ & $+12.1$ & $+11.6$ & $0.66$/$0.00$ & $+20.7$ & $+5.8$ & $+15.4$ \\
$27$ & $0.25$ & $+13.6$ & $+13.6$ & $0.72$/$0.00$ & $+17.8$ & $+2.3$ & $+9.7$ \\
$27$ & $0.90$ & $+15.8$ & $+20.6$ & $0.74$/$0.00$ & $+25.3$ & $+4.9$ & $+14.2$ \\
$8$ & $0.25$ & $+0.5$ & $-10.2$ & $0.08$/$0.00$ & $+1.0$ & $+0.4$ & $+0.9$ \\
$8$ & $0.90$ & $+4.3$ & $-2.8$ & $0.26$/$0.00$ & $+10.3$ & $+9.1$ & $+9.7$ \\
\bottomrule
\end{tabular}
\end{center}
\noindent $20{,}000$ paired bodies per cell, \texttt{results/knee\_individuum\_sweep\_v5}, and every number in this table and the knee rows of Table~\ref{tab:individuum} is recomputed from it by \texttt{experiments/paper\_zahlen.py}. Full equipment vs.\ ignorant for injured bodies: $+21.0$/$+25.6$\,y at $\Delta\theta{=}18$/$27$ (tear $0.9$), early exit $0.69$/$0.87\to0.04$/$0.13$. Pre-registered verdicts. I1 (full vs.\ ignorant at $\Delta\theta\ge18$: $>5$\,y, $J>0$, gain $\ge0.60$, lose $\le0.05$) holds in both cells, at $18^\circ$ with $+12.1$\,y, $+11.6\%$ $J$, gain $0.66$, no losers. I2 (memory minus reflex, injured, $>3$\,y: $+15.0$ and $+20.5$) holds at $\ge18^\circ$ and fails at $8^\circ$ ($+1.2$). I4 (fewer re-injuries per injured life) fails on its denominator: the protected worker stays $12$--$16$ years longer, and per exposed year the two are equal ($0.0146$ against $0.0145$), in this world the equipment buys career, not tears. Two verdicts came out weaker than the earlier build reported, and both weaken the $8^\circ$ fall case rather than the result: I3 (full minus threshold $>0$, injured) is $+5.9$ and $+12.1$ at $\ge18^\circ$ but $+0.7$ at $8^\circ$, so it no longer separates, and I5 (a tear costs the ignorant $>3$\,y) now holds at $8^\circ$ too ($+10.3$), because with the event level anchored to the trade's prevalence a tear is expensive even where flexion barely fails. What still separates $8^\circ$ is I1 and I2: there the same threshold buys $+4.3$\,y at a cost of $3\%$ of $J$, gain share $0.26$, and the memory adds $+1.2$\,y instead of $+15.0$.

\textbf{Reading and caveats.} The individual value is a lower bound from set, not optimized, protective responses (threshold $0.5$, one avoidance portfolio), the exposure factor was calibrated on trained policies, and the world is one trade. The same probes in the fleet's event world protect the joint, and the paired-stream numbers are row 3 of Table~\ref{tab:individuum}.

\textbf{Ladder.} The best fixed age rule (what a population, not a body, can teach) is chosen on a grid of thirteen switch ages and evaluated on the paired bodies. Its career-optimal setting is degenerate (switch at entry, in every flexion cell, because with the floor referred to the body's own capacity permanent avoidance never costs the trade), and it wins on career by a little: $+12.5$ years against $+12.1$ at $\Delta\theta{=}18^\circ$, $+18.1$ against $+15.8$ at $27^\circ$, $+4.3$ against $+4.3$ at $8^\circ$. On $J$ it loses in every cell ($+4.8\%$ against $+11.6\%$, $+17.2\%$ against $+20.6\%$, $-10.2\%$ against $-2.8\%$), and the age rule chosen on output rather than career still loses ($+8.2\%$ at $18^\circ$). A per-body oracle over the same thirteen ages reaches $+18.5\%$ on $J$ and exactly $+0.000$ years on career ($5{,}000$ bodies, best population age $32$): a schedule captures all of the survival and a quarter less of the return than one body's own threshold does.

\textbf{Sensitivity.} The individual value over the felt threshold ($0.3$--$0.7$), the substitution wage ($0.60$, $0.84$, $1.00$ of the kneeling wage), and the total-output floor ($40$, $50$, $60$), with $5{,}000$ paired bodies per point, tear $0.9$, entries at threshold $0.5$ with the range over thresholds in brackets, bold in the gate-selected flexion cell. Sixteen of eighteen cells pay. The two that do not are the same cell in both flexion worlds, lowest substitution wage against highest floor, where avoiding the knee cannot clear the floor and protection costs a decade ($-9.9$ and $-5.3$ years, gain share $0.06$--$0.08$): condition (ii) failing on one body, and the over-protection pole of Prop.~\ref{prop:calibration} in the same cell, with the equipment unchanged.

\begin{center}\footnotesize
\begin{tabular}{@{}rrrlll@{}}
\toprule
$\Delta\theta$ & wage & floor & $\Delta$career (y) & $\Delta J$ (\%) & gain share \\
\midrule
$27$ & $0.60$ & $40$ & $+14.3$ [$+11.8$, $+15.7$] & $+9$ [$+1$, $+12$] & $0.73$ [$0.58$, $0.82$] \\
$27$ & $0.60$ & $50$ & $+15.7$ [$+12.2$, $+17.8$] & $+6$ [$-1$, $+9$] & $0.75$ [$0.56$, $0.86$] \\
$27$ & $0.60$ & $60$ & $-5.3$ [$-5.7$, $-3.1$] & $-6$ [$-15$, $+0$] & $0.08$ [$0.07$, $0.09$] \\
$27$ & $0.84$ & $40$ & $+14.3$ [$+11.8$, $+15.7$] & $+24$ [$+21$, $+24$] & $0.73$ [$0.58$, $0.82$] \\
$27$ & $0.84$ & $50$ & $+15.8$ [$+12.4$, $+17.9$] & $+20$ [$+18$, $+20$] & $0.75$ [$0.56$, $0.86$] \\
$27$ & $0.84$ & $60$ & $+17.1$ [$+12.6$, $+20.3$] & $+17$ [$+15$, $+17$] & $0.74$ [$0.52$, $0.90$] \\
$27$ & $1.00$ & $40$ & $+14.3$ [$+11.8$, $+15.7$] & $+34$ [$+28$, $+36$] & $0.73$ [$0.58$, $0.82$] \\
$27$ & $1.00$ & $50$ & $+15.8$ [$+12.4$, $+17.9$] & $+30$ [$+24$, $+32$] & $0.75$ [$0.56$, $0.86$] \\
$27$ & $1.00$ & $60$ & $+17.2$ [$+12.8$, $+20.3$] & $+26$ [$+20$, $+29$] & $0.75$ [$0.52$, $0.90$] \\
$\mathbf{18}$ & $0.60$ & $40$ & $+10.1$ [$+9.4$, $+10.3$] & $-1$ [$-11$, $+6$] & $0.59$ [$0.53$, $0.61$] \\
$\mathbf{18}$ & $0.60$ & $50$ & $+11.8$ [$+10.5$, $+12.3$] & $-2$ [$-11$, $+4$] & $0.65$ [$0.57$, $0.69$] \\
$\mathbf{18}$ & $0.60$ & $60$ & $-9.9$ [$-11.0$, $-6.0$] & $-12$ [$-22$, $-3$] & $0.06$ [$0.06$, $0.09$] \\
$\mathbf{18}$ & $0.84$ & $40$ & $+10.1$ [$+9.4$, $+10.3$] & $+12$ [$+8$, $+14$] & $0.59$ [$0.53$, $0.61$] \\
$\mathbf{18}$ & $0.84$ & $50$ & $+11.8$ [$+10.6$, $+12.3$] & $+11$ [$+7$, $+13$] & $0.65$ [$0.57$, $0.69$] \\
$\mathbf{18}$ & $0.84$ & $60$ & $+13.9$ [$+11.9$, $+14.8$] & $+11$ [$+6$, $+13$] & $0.71$ [$0.59$, $0.77$] \\
$\mathbf{18}$ & $1.00$ & $40$ & $+10.1$ [$+9.4$, $+10.3$] & $+21$ [$+20$, $+21$] & $0.59$ [$0.53$, $0.61$] \\
$\mathbf{18}$ & $1.00$ & $50$ & $+11.8$ [$+10.6$, $+12.3$] & $+20$ [$+19$, $+20$] & $0.65$ [$0.57$, $0.69$] \\
$\mathbf{18}$ & $1.00$ & $60$ & $+13.9$ [$+12.0$, $+14.8$] & $+19$ [$+18$, $+19$] & $0.71$ [$0.59$, $0.77$] \\
\bottomrule
\end{tabular}
\end{center}
The wage sets the sign of the $J$ gain, the floor decides whether there is a gain at all, and the threshold is the mildest of the three: over $0.3$ to $0.7$ the career gain moves by $1$--$3$ years in the paying cells, so the one pre-registered value is not carrying the result. The threshold $0.5$ and the avoidance portfolio were pre-registered (I1--I5) before the probe sweep, and the sensitivity scan came after it.

\end{document}